\documentclass{article}

\usepackage{mathptmx}
\usepackage{hyperref}

\usepackage{authblk}
\def\hb{\hbox to 10.7 cm{}}
\usepackage[all,cmtip]{xy}
\usepackage{xspace}
\usepackage{enumerate,paralist}
\usepackage{subcaption}
\usepackage{tikz}
\usetikzlibrary{arrows,positioning,shapes.misc,shapes.geometric}
\usetikzlibrary{arrows.meta}   
\tikzset{>=stealth',args/.style={circle,draw=black},minimum size=10mm}

\usepackage{lmodern}
\usepackage{amsmath,amsthm,amssymb}
\usepackage[normalem]{ulem}
\newtheorem{example}{Example}
\newtheorem{definition}{Definition}
\newtheorem{lemma}{Lemma}
\newtheorem{theorem}{Theorem}

\newtheorem{proposition}{Proposition}

\newtheorem{corollary}{Corollary}

\newtheorem{remark}{Remark}

\newcommand{\exactcompliant}{exact\xspace}

\newcommand {\tup}[1]      {{\langle #1 \rangle}}

\newcommand{\Xb}{{\bf X}}
\newcommand{\Yb}{{\bf Y}}

\definecolor{dgreen}{rgb}{0.,0.6,0.}

\newcommand{\jesse}[1]{#1}

 \newcommand{\jss}[1]{\textcolor{blue!20!black}{#1}}
  \newcommand{\jessen}[1]{{\color{blue}#1}}

\newcommand{\HR}{\mathit{HD}}
\newcommand{\IC}{\mathit{IC}}
\newcommand{\HRc}{\mathcal{HD}}
\DeclareMathOperator{\lfp}{lfp}
\newcommand{\upclosure}[1]{{#1\!\!\uparrow}\xspace}
\newcommand{\downclosure}[1]{{#1\!\!\downarrow}\xspace}
\newcommand{\model}{mod}

\DeclareRobustCommand\sampleline[1]{%
  \tikz\draw[#1, thick,line width=1] (0,0) (0,\the\dimexpr\fontdimen22\textfont2\relax)
  -- (2em,\the\dimexpr\fontdimen22\textfont2\relax);%
    
}
\title{Non-Deterministic Approximation Fixpoint Theory \\ and Its Application in Disjunctive Logic Programming}

\author[1]{Jesse Heyninck}
\author[2]{Ofer Arieli}
\author[3]{Bart Bogaerts}
\affil[1]{Open Universiteit, the Netherlands}
\affil[2]{School of Computer Science, Tel-Aviv Academic College, Israel}
\affil[3]{Vrije Universiteit Brussel, Belgium}

\DeclareMathAlphabet\mathbfcal{OMS}{cmsy}{b}{n}

\begin{document}
\maketitle

\begin{abstract}
Approximation fixpoint theory (AFT) is an abstract and general algebraic framework for studying the semantics of nonmonotonic logics. 
It provides a unifying study of the semantics of different formalisms for nonmonotonic reasoning, such as logic programming, default logic 
and autoepistemic logic. In this paper, we extend AFT to dealing with {\em non-deterministic constructs\/} that allow to handle
indefinite information, represented e.g.\ by disjunctive formulas. This is done by generalizing the main constructions and corresponding 
results of AFT to non-deterministic operators, whose ranges are sets of elements rather than single elements. The applicability 
and usefulness of this generalization is  illustrated in the context of disjunctive logic programming.
\end{abstract}

\section{Introduction}\label{sec:intro}
Semantics of various formalisms for knowledge representation can often be described by fixpoints of corresponding operators. For example, in many logics theories 
of a set of formulas can be seen as fixpoints of the underlying consequence operator~\cite{tarski1936concept}. Likewise, in logic programming, default logic or formal 
argumentation, all the major semantics can be formulated as different types of fixpoints of the same operator (see~\cite{denecker2000approximations}). Such operators are usually non-monotonic, 
and so one cannot always be sure whether their fixpoints exist, and how they can be constructed. 

In order to deal with this `illusive nature' of the fixpoints, Denecker, Marek and Truszczy{\'n}ski~\cite{denecker2000approximations} introduced a method for \emph{approximating\/} 
each value $z$ of the underlying operator by a pair of elements $(x,y)$. These elements intuitively represent lower and upper bounds on $z$, and so a corresponding 
\emph{approximation operator\/} for the original, non-monotonic operator, is constructed. If the approximating operator that is obtained is precision-monotonic, intuitively meaning 
that more precise inputs of the operator give rise to more precise outputs, then by Tarski and Knaster's Fixpoint Theorem the approximating operator has fixpoints that can be constructively
computed, and which in turn approximate the fixpoints of the approximated operator, if such fixpoints exist. 

The usefulness of the algebraic theory that underlies the computation process described above was demonstrated on several knowledge representation formalisms, such as propositional 
logic programming~\cite{denecker2012approximation}, default logic~\cite{denecker2003uniform}, autoepistemic logic~\cite{denecker2003uniform}, abstract argumentation and abstract 
dialectical frameworks~\cite{strass2013approximating}, hybrid MKNF~\cite{liu2022alternating}, the graph description language SCHACL~\cite{bogaerts2021fixpoint}, and active 
integrity constraints~\cite{ai/BogaertsC18}, each one of which was shown to be an instantiation of this abstract theory of approximation. More precisely, it was shown that various 
semantics of the formalisms mentioned above correspond to the various fixpoints defined in~\cite{denecker2000approximations}. This means that approximation fixpoint theory 
(AFT, for short) captures the uniform principles underlying all these non-monotonic formalisms in a purely algebraic way. 

Besides its unifying capabilities, AFT also allows for a straightforward definition of semantics for new formalisms. Indeed, one merely has to define an approximation of the
operator of interest, and AFT then automatically gives rise to a family of semantics with several desirable properties. This potential of AFT for a straightforward derivation 
of semantics was demonstrated in e.g.\ logic programming, where it was used to define semantics for extensions of logic programs~\cite{antic2013hex,charalambidis2018approximation,pelov2007well}, in argumentation theory, where it was used to define semantics for weighted abstract dialectical 
frameworks (ADFs,~\cite{bogaerts2019weighted}), in autoepistemic logic, where it was used for defining distributed variants that are suitable for studying access control 
policies~\cite{vanhertum2016distributedautoepistemic}, and in second-order logic extended with non-monotone inductive definitions~\cite{dasseville2016compositional}.

Another benefit of AFT is that, due to its generality, it has proven useful to develop central concepts, such as strong equivalence~\cite{truszczynski2006strong},
groundedness~\cite{bogaerts2015grounded}, safe inductions~\cite{bogaerts2018safeInductions}, and stratification~\cite{vennekens2006splitting} for approximation operators in a purely algebraic way, which then allow to derive
results on these concepts for all the specific formalisms representable in AFT as straightforward corollaries.

So far, AFT has mainly been applied to \emph{deterministic\/} operators, i.e., operators which map single inputs to single outputs. This means that, while AFT is able to characterize 
semantics for normal logic programs (i.e., programs consisting of rules with single atomic formulas as their head), \emph{disjunctive\/} logic programs~\cite{minker2002disjunctive} 
(i.e., programs consisting of rules with disjunctions of atoms in their head) cannot be represented by it. The same holds for the representation of e.g.\ default logic versus disjunctive default 
logic~\cite{Gelfond_et-al_PKRR_1991}, abstract argumentation versus set-based abstract argumentation~\cite{nielsen2006generalization} and the generalization of abstract dialectical 
frameworks to~\emph{conditional\/} abstract dialectical frameworks~\cite{DBLP:conf/aaai/HeyninckTKRS22}.

Extending AFT to handle disjunctive information is therefore a desirable goal, as the latter has a central role in systems for knowledge representation and reasoning, and since
disjunctive reasoning capabilities provide an additional way of expressing uncertainty and indeterminism to many formalisms for non-monotonic reasoning. However,  the introduction of disjunctive reasoning often increases the computational complexity of formalisms and thus extends their modeling capabilities~\cite{eiter1993complexity}.
Perhaps due to this additional expressiveness, the integration of non-deterministic reasoning with non-monotonic reasoning (NMR) has often proven non-trivial, as witnessed e.g.\ 
by the large body of literature on disjunctive logic programming~\cite{lobo1992foundations,minker2002disjunctive}. The implementation of non-deterministic reasoning in NMR 
yielded the formulation of some (open) problems that are related to the combination of non-monotonic and disjunctive reasoning~\cite{beirlaen2017reasoning,beirlaen2018critical,bonevac2018defaulting}, or was restricted to limited  semantics available for the core formalism, as is the case for 
e.g.\ default logic~\cite{Gelfond_et-al_PKRR_1991}.

The goal of this work is to provide an adequate framework for modeling disjunctive reasoning in NMR. We do so by extending AFT to  handle {\em non-deterministic operators\/}. 
This idea was first introduced by Pelov and Truszczy{\'n}ski in~\cite{pelov2004semantics}, where some first results on two-valued semantics for disjunctive logic programs were 
provided. In this paper, we further extend AFT  for non-deterministic operators, which, among others, allows a generalization of the results of~\cite{pelov2004semantics} to the 
three-valued case. In particular, we define several interesting classes of approximating fixpoints and show their existence, constructability and consistency where it is possible. 
An application of this theory is demonstrated in the context of disjunctive logic programming. Furthermore, we show that  our theory is a conservative  generalization of the work 
in~\cite{denecker2000approximations} of AFT for deterministic operators,  in the sense that all the concepts introduced in this paper coincide with the deterministic counterparts when the operator at hand happens to be deterministic. 

The outcome of this work is therefore a comprehensive study of semantics for non-monotonic formalisms incorporating non-determinism.
Its application is demonstrated in this paper in the context of disjunctive logic programming. Specifically, the paper contains the following contributions:

\begin{enumerate}
     \item We define variants of both the \emph{Kripke-Kleene} and the \emph{stable/ well-founded} semantics, the \emph{interpretation} or \emph{fixpoint semantics\/} and \emph{state semantics\/}. 
              Interpretation semantics consist of single pairs of elements, and thus approximate of a single element. State semantics, on the other hand, consist of pairs of sets of elements, intuitively 
              viewed as a convex set, which approximates a set of elements.
     \item We show that the Kripke-Kleene state and well-founded state  (obtained as the least fixed point of the stable state operator) exist and are unique (Theorem \ref{theorem:ndso:fixpoint} and 
              Theorem~\ref{proposition:properties:of:well-founded:state}). 
     \item We show that the Kripke-Kleene state approximates any fixpoint of an approximation operator (Theorem~\ref{proposition:properties:of:well-founded:state}), 
              whereas the well-founded state approximates any stable fixpoint of an approximation operator (Theorem~\ref{proposition:properties:of:well-founded:state}). 
      In more detail, any fixpoint of an approximation operator respectively stable fixpoint  of an approximation operator is an element of the convex set represented by the Kripke-Kleene state respectively the well-founded state.
           
     \item We show that when restricting attention to deterministic operators, the theory reduces to deterministic AFT~\cite{denecker2000approximations} (Remark~\ref{remark:deterministic} and 
              Propositions~\ref{prop:deterministic:aft},~\ref{prop:KK-singletons} and~\ref{prop:stble:coincide:deterministic}).
     \item We show that, just like in deterministic AFT, stable fixpoints are fixpoints that are minimal with respect to the truth order (Proposition \ref{prop:stable:is:minimal:fp}). 
     \item We demonstrate the usefulness of our abstract framework by showing how all the major semantics for disjunctive logic programming can be characterized as fixpoints 
              of an approximation operator for disjunctive logic programming. In more detail, the weakly supported models~\cite{brass1995characterizations} are characterized as 
              fixpoints of the operator ${\cal IC}_{\cal P}$ (Theorem~\ref{theo:correspondence:supported}), the stable models can be characterized as the stable fixpoints of 
              ${\cal IC}_{\cal P}$ (Theorem~\ref{prop:stable:fixpoints:represent:stable:models}) and the well-founded semantics by Alc{\^a}ntara, Dam{\'a}sio and 
              Pereira~\cite{alcantara2005well} is strongly related to the well-founded state of ${\cal IC}_{\cal P}$ (Theorem~\ref{prop:well:founded:state:represents:alcantara:almost}).
\end{enumerate}

\paragraph*{Relation with previous work on non-deterministic approximation fixpoint theory}
This work extends and improves the work by Heyninck and Arieli \cite{DBLP:conf/kr/HeyninckA21}. %
The theory has been simplified on several accounts, among others since: (1)~we no longer require minimality of the elements of the range of a non-deterministic (approximation) 
operator, which leads to a significant decrease in the number of lattice-constructions needed, (2)~the state operator is now more general and can 
be defined on the basis of a non-deterministic approximation operator. Furthermore, intuitive explanations and illustrative examples are added throughout the paper. 
As a by-product, the simplified framework allows us to identify and correct a faulty statement on the $\leq_i$-monotonicity 
of the approximation operator for disjunctive logic programs, made by Heyninck and Arieli \cite{DBLP:conf/kr/HeyninckA21}  (see Remark~\ref{remark:minimality:in:the:op}). 

Both this paper and the work of Heyninck and Arieli~\cite{DBLP:conf/kr/HeyninckA21} were partially inspired by the work of Pelov and Truszczy{\'n}ski~\cite{pelov2004semantics} where, 
to the best of our knowledge, the idea of a non-deterministic operator was first introduced in approximation fixpoint theory. Pelov and Truszczy{\'n}ski~\cite{pelov2004semantics} studied only two-valued semantics, whereas 
we study the full range of semantics for AFT. Furthermore, Pelov and Truszczy{\'n}ski~\cite{pelov2004semantics} required minimality of the codomain of non-deterministic operators, which we do not require here.

\paragraph*{Outline of this paper}
This rest of this paper is organized as follows: In Section~\ref{sec:back:prelim} we recall the necessary background on disjunctive logic programming (Section~\ref{sec:LP}) 
and approximation fixpoint theory (Section~\ref{sec:AFT}). In Section~\ref{sec:nd:aft} we introduce non-deterministic operators and their approximation, and show some preliminary results 
on approximations of non-deterministic operators. In Section~\ref{sec:theory:of:ndao} we %
study these non-deterministic approximation operators, showing their consistency and introducing and studying their fixpoint, Kripke-Kleene interpretation and the 
Kripke-Kleene state semantics. In Section~\ref{sec:stable:semantics} we introduce and study the stable interpretation and state semantics, as well as the well-founded 
state semantics. Related work is discussed in Section~\ref{sec:related:work}, followed by a conclusion in Section~\ref{sec:conc}.

\section{Background and Preliminaries}
\label{sec:back:prelim}

In this section, we recall the necessary basics of approximation fixpoint theory (AFT) for deterministic operators. We start with a brief survey on disjunctive logic 
programming (DLP, Section~\ref{sec:LP}), which will serve to illustrate concepts and results of the general theory of non-deterministic AFT (Section~\ref{sec:AFT}).

\begin{remark}
Before proceeding, a note on notation: As we often have to move from the level of single elements to sets of elements, the paper is, by its nature, notationally heavy. 
We tried to keep the notational burden as light as possible by staying consistent in our notation of different types of elements. For the readers  convenience, we provide 
already at this stage a summary of the notations of different types of sets (Table~\ref{tab:set-notations}), the preorders (Table~\ref{tab:orders}) and the operators (Table~\ref{tab:operators}) that are used in this paper. 

\def\arraystretch{1,3}\tabcolsep=10pt
\begin{table}[htb]
\begin{tabular}{lll}
Elements & Notations & Example \\ \hline \hline
Elements of ${\cal L}$ & $x,y,\ldots$ & $x,y$\\
Sets of elements of ${\cal L}$ & $X,Y,\ldots$ & $\{x_1,y_1,x_2,y_2\}$ \\
Pairs of sets of elements of ${\cal L}$ & ${\bf X},{\bf Y},\ldots$ & $\{x_1,x_2,\ldots\}\times \{y_1,y_2,\ldots\}$\\
Sets of sets of elements of ${\cal L}$ &  ${\cal X},{\cal Y},\ldots$ & $\{\{x_1,x_2\},\{x_1\}\}$
\end{tabular}
\caption{List of the notations of different types of sets used in this paper} 
\label{tab:set-notations}
\end{table}

\begin{table}[htb]
\begin{tabular}{lll}
Preorder & Type & Definition \\ \hline \hline
\multicolumn{3}{c}{Element Orders} \\ \hline
$\leq$ & ${\cal L}$ & primitive \\
$\leq_i, \leq_t$ & ${\cal L}\times {\cal L}$ & bilattice orders (Definition~\ref{def:bilattice}) \\
$\leq_i$ & ${\cal L}^2\times {\cal L}^2$ & $(x_1,y_1) \leq_i (x_2,y_2)$ iff $x_1 \leq x_2$ and $y_1 \geq y_2$ \\
$\leq_t$ & ${\cal L}^2\times {\cal L}^2$ & $(x_1,y_1) \leq_t (x_2,y_2)$ iff $x_1 \leq x_2$ and $y_1 \leq y_2$ \\ \hline
\multicolumn{3}{c}{Set-based Orders} \\ \hline
$\preceq^S_L$ & $\wp({\cal L})\times \wp({\cal L})$ & $X \preceq^S_L Y$ iff for every $y\in Y$ there is an 
$x\in X$ s.t.\ $x\leq y$\\
$\preceq^H_L$ & $\wp({\cal L})\times \wp({\cal L})$ &  $X \preceq^H_L Y$ iff for every $x\in X$ there is an 
$y\in Y$ s.t.\ $x\leq y$ \\
$\preceq^A_i$ & $\wp({\cal L})^2\times\wp({\cal L})^2$ & $(X_1,Y_1)\preceq_i^A (X_2,Y_2)$ iff $X_1\preceq^S_L X_2$ and $Y_2\preceq^H_L Y_1$\\
\end{tabular}
\caption{List of the preorders used in this paper.}
\label{tab:orders}
\end{table}

\begin{table}[htb]
\begin{tabular}{lclll}
Operator & Notation& Type & Definition \\ \hline \hline
Non-deterministic operator & $O$ & ${\cal L}\mapsto \wp({\cal L})$ & Definition \ref{def:non-deterministic-operator} \\
Non-deterministic approximation operator & ${\cal O}$ & ${\cal L}^2\mapsto \wp({\cal L})\times\wp({\cal L})$ & Definition \ref{def:ndao}\\
Non-deterministic state approx.\ operator  & ${\cal O}'$& $\wp({\cal L}^2)\mapsto \wp({\cal L})\times\wp({\cal L})$ & Definition \ref{def:ndsao}\\ \hline
Stable operator & $S({\cal O})$ & ${\cal L}^2\mapsto \wp({\cal L})\times\wp({\cal L})$ & Definition \ref{def:stable:op}\\

\end{tabular}
\caption{List of the operators used in this paper.}
\label{tab:operators}
\end{table}

\end{remark}

\subsection{Disjunctive Logic Programming}
\label{sec:LP}

In what follows we consider a propositional\footnote{For simplicity we restrict ourselves to the propositional case.} language ${\mathfrak L}$, 
whose atomic formulas are denote by $p,q,r$ (possibly indexed), and that contains the the propositional constants ${\sf T}$ (representing truth), ${\sf F}$ (falsity), 
${\sf U}$ (unknown), and ${\sf C}$ (contradictory information). The connectives in  ${\mathfrak L}$ include negation $\neg$, conjunction $\wedge$, disjunction $\vee$, 
and implication $\leftarrow$. Formulas are denoted by $\phi$,$\psi$ (again, possibly indexed). Logic programs in ${\mathfrak L}$ may be divided to different kinds as follows:
\begin{itemize}
     \item A (propositional) {\em disjunctive logic program\/} ${\cal P}$ in ${\mathfrak L}$ (a dlp, for short) is a finite set of rules of the form 
             $\bigvee_{i=1}^n p_i~\leftarrow~\psi$,  where $\bigvee_{i=1}^n p_i$ (the rule's head) is a non-empty disjunction of atoms, and $\psi$ (the rule's body) 
             is a (propositional) formula.
     \item A rule is called {\em normal\/}, if its body is a conjunction of literals (i.e., atomic formulas or negated atoms), and its head is atomic. A program 
             is {\em normal\/} if it consists only of normal rules; It is {\em positive\/} if there are no negations in the rules' bodies. 
     \item We call a rule \emph{disjunctively normal} if its body is a conjunction of literals (and its head is a non-empty disjunction of atoms). 
              A program is called {\em disjunctively normal\/}, if it consists of disjunctively normal rules.
\end{itemize}
The set of atoms occurring in a logic program $\mathcal{P}$ is denoted ${\cal A}_{\cal P}$. In what follows, we will often leave the reference to the language $\mathfrak{L}$ 
of ${\cal P}$ implicit.

\medskip
The primary algebraic structure for giving semantics to logic programs in our setting is the four-valued structure ${\cal FOUR}$, shown in Figure~\ref{fig:four}. 
This structure was introduced by Belnap \cite{Be77a,Be77b} and later considered by Fitting~\cite{Fi91} and others in the context of logic programming. It is the simplest instance of a {\em bilattice\/} (Definition~\ref{def:bilattice}).\footnote{We refer to~\cite{Fi06,Fi20} for further details on bilattices and their applications in logic programming.}

\begin{figure}[hbt]
\begin{center}
\resizebox{0.4\textwidth}{!}{%
\begin{tikzpicture}[node distance=1cm, auto, >=latex, scale=0.5]
		\node (a) {};
		\node (b) [above of=a, yshift=3cm] {$\leq_i$};
		\node (c) [right of=a, xshift=3cm] {$\leq_t$};
		\draw[->] (a.center) -- (b);
		\draw[->] (a.center) -- (c);
		\tikzstyle{dot}= [circle, fill, minimum size=4pt,inner sep=0pt, outer sep=0pt]
		\node (third) [dot,above right of=a, node distance=1.2cm, xshift=1cm, label=below:{${\sf U}$}] {};
		\node (second) [dot,above left of=third, label=left:{{${\sf F}$}}] {};
		\node (fourth) [dot,above right of=third, label=right:{${\sf T}$}] {};
		\node (first) [dot, above right of= second,label=above:{{${\sf C}$}}] {};	

		\path[color=black] (second) edge (first);
		\path[color=black] (third) edge (second);
		\path[color=black] (fourth) edge (third);
		\path[color=black] (first) edge (fourth);
	\end{tikzpicture}}
	\end{center}
\caption{A four-valued bilattice}	
\label{fig:four}	
\end{figure}
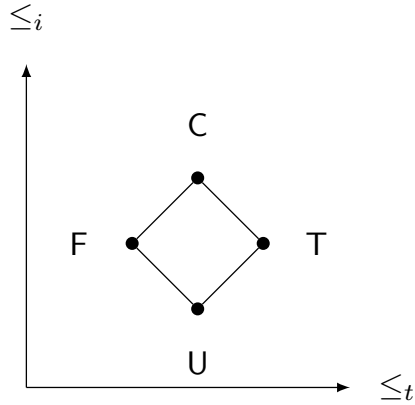

Each element of ${\cal FOUR}$ is associated with the propositional constant of ${\cal L}$ with the same notation. These elements are arranged in two lattice orders, 
$\leq_t$ and $\leq_i$, intuitively representing differences in the amount of {\em truth\/} and {\em  information\/} (respectively) that each element exhibits. 
According to this interpretation, ${\sf T}$ (respectively, ${\sf F}$) exhibits maximal (respectively, minimal) truth, while ${\sf C}$ (respectively, ${\sf U}$) represents 
`too much' (respectively, lack of) information. In what follows, we denote by $-$ the $\leq_t$-involution on ${\cal FOUR}$ (that is, $-{\sf F}={\sf T}$, $-{\sf T}={\sf F}$, 
$-{\sf U}={\sf U}$ and $-{\sf C}={\sf C}$). 

\medskip
A {\em four-valued interpretation} of a program ${\cal P}$ is a pair $(x,y)$, where $x \subseteq {\cal A}_{\cal P}$ is the set of the atoms that are assigned a value in 
$\{{\sf T},{\sf C}\}$ and $y \subseteq {\cal A}_{\cal P}$ is the set of atoms assigned a value in $\{{\sf T},{\sf U}\}$.\footnote{Somewhat skipping ahead, the intuition 
here is that $x$ (respectively, $y$) is a lower (respectively, an upper) approximation of the true atoms.}

Interpretations are compared by two order relations, corresponding to the two partial orders of ${\cal FOUR}$:  
\begin{enumerate}
     \item the \emph{information order\/} $\leq_i$, where $(x,y)\leq_i (w,z)$ iff $x\subseteq w$ and $z\subseteq y$, and 
     \item the \emph{truth order\/} $\leq_t$, where $(x,y)\leq_t (w,z)$ iff $x\subseteq w$ and $y\subseteq z$. 
\end{enumerate}
The information order represents differences in the ``precisions'' of the interpretations. Thus, the components of higher values according to this order represent 
tighter evaluations. The truth order represents increased `positive' evaluations. Truth assignments to complex formulas are then recursively defined as follows:
\begin{itemize}
\item $(x,y)(\jss{p})=
\begin{cases}
      {\sf T} & \text{ if } \jss{p} \in x \text{ and } \jss{p} \in y,  \\
      {\sf U} & \text{ if } \jss{p} \not\in x \text{ and }\jss{p} \in y,  \\
      {\sf F} & \text{ if } \jss{p} \not\in x \text{ and } \jss{p} \not\in y,  \\
      {\sf C} & \text{ if } \jss{p} \in x \text{ and } \jss{p} \not\in y. 
\end{cases}$ \smallskip
\item $(x,y)(\lnot \phi)=- (x,y)(\phi)$,  
\item $(x,y)(\psi \land \phi)=lub_{\leq_t}\{(x,y)(\phi),(x,y)(\psi)\}$, 
\item $(x,y)(\psi \lor \phi)= glb_{\leq_t}\{(x,y)(\phi),(x,y)(\psi)\}$. 
\end{itemize}

A four-valued interpretation of the form $(x,x)$ may be associated with a {\em two-valued\/} (or {\em total\/}) interpretation $x$, 
in which for an atom $p$, $x(p) = {\sf T}$ if $p \in x$ and $x(p) = {\sf F}$ otherwise. We say that $(x,y)$ is a {\em three-valued\/} 
(or {\em consistent\/}) interpretation, if $x \subseteq y$. Note that in consistent interpretations there are no ${\sf C}$-assignments.

\medskip
We now consider semantics for dlp's. First, given a two-valued interpretation, an extension to dlp's of the immediate consequence operator for normal 
programs \cite{EmdenK76} is defined as follows:

\begin{definition}%
\label{def:operator:disj:lp}
Given a dlp ${\cal P}$ and a two-valued interpretation $x$, we define:
\begin{itemize}
\item $\HR_{\cal P}(x)=\{\Delta\mid \bigvee\!\Delta\leftarrow \psi \in{\cal P} \text{ and } (x,x)(\psi) =  {\sf T}\}$. 
\item $\IC_{\cal P}(x)=\{y\subseteq \:\bigcup\!\HR_{\cal P}(x)  \mid \forall \Delta \in \HR_{\cal P}(x), \ y \cap \Delta \neq \emptyset \}$. 
\end{itemize}
 \end{definition}

Thus, $\IC_{\cal P}(x)$ consists of sets of atoms, each sets contains at least one representative from every disjuncts of a rule in ${\cal P}$ whose body is $x$-satisfied 
(i.e, a representative from each set $\Delta \in\HR_{\cal P}(x)$). In other words, $\IC_{\cal P}(x)$ consists of the two-valued interpretations that validate all disjunctions 
which are derivable from ${\cal P}$ given $x$. Denoting by $\wp({\cal S})$ the powerset of ${\cal S}$, $\IC_{\cal P}$ is an operator on the lattice 
$\tup{\wp({\cal A}_{\cal P}),\subseteq}$.\footnote{The operator $\IC_{\cal P}$ is a generalization of the immediate consequence operator from \cite[Definition 3.3]{fernandez1995bottom}, 
where the minimal sets of atoms in  $IC_{\cal P}(x)$ are considered. We will see below that this requirement of minimality is neither necessary nor desirable in the consequence operator.}

\begin{example}
\label{examp:IC-deterministic-case}
Consider the dlp ${\cal P}=\{p\lor q\leftarrow \}$. For any two-valued interpretation $x$, 
$\HR_{\cal P}(x)=\{\{p,q\}\}$, since $p\lor q\leftarrow \in {\cal P}$ and the body of this rule is an empty  conjunction and therefore true under any interpretation.
Thus, $\IC_{\cal P}(x)=\{\{p\},\{q\},\{p,q\}\}$ for any two-valued interpretation $x$. This intuitively reflects the fact that
either $p$ or $q$ has to be true to validate the head of $p\lor q\leftarrow$. 
\end{example}

\smallskip
Other semantics for dlp's, this time based on three-valued interpretations, are defined next:

\begin{definition}
\label{def:3-val-sem-dlp}
 Given a dlp ${\cal P}$ and a consistent interpretation $(x,y)$. We say that $(x,y)$ is:
\begin{itemize}
\item a \emph{(three--valued) model\/} of ${\cal P}$, if for every $\phi\leftarrow \psi \in {\cal P}$, $(x,y)(\phi)\geq_t (x,y)(\psi)$.
         We denote by $mod({\cal P})$ the set of the three-valued models of ${\cal P}$.
\item a \emph{weakly supported model\/} of ${\cal P}$, if it is a model of ${\cal P}$ and for every $p \in  y$,
         there is a rule $\bigvee\!\Delta\leftarrow \phi\in {\cal P}$ such that $p\in\Delta$ and $(x,y)(\phi)\geq_t (x,y)(p)$.
\item a \emph{supported model\/} of ${\cal P}$, if it is a model of ${\cal P}$ and for every $p\in{\cal A}_{\cal P}$ such that $(x,y)(p)={\sf T} \ [(x,y)(p)={\sf U}]$, there is  
        $\bigvee\!\Delta\leftarrow \phi\in {\cal P}$ such that $p \in\Delta$ and $(x,y)(\phi)={\sf T} \ [(x,y)(p)={\sf U}]$ and $\Delta\cap x = \{p\} \ [\Delta\cap y=\{p\}]$.
\end{itemize}
\end{definition}

The intuition behind the notions above is the following. An interpretation is a model of ${\cal P}$, if for each rule in ${\cal P}$ there is at least one atom whose truth value is 
$\leq_t$-greater or equal to the truth value of the rule's body. Thus, the truth values of the rules' heads in the models of ${\cal P}$ are $\leq_t$-greater or equal to the 
truth values of the rules' bodies. Weakly supported models require that for every atom that is true (respectively undecided), we can find a rule in whose head this atom 
occurs and for which the body is true (respectively undecided). In other words, every atom is supported by an ``activated'' rule. Supported models strengthen this requirement 
by requiring that for every atom that is true (respectively undecided), there is an ``activated'' rule in whose head the atom under consideration is the \emph{only\/} true 
(respectively undecided) atom.

The semantical notions of Definition~\ref{def:3-val-sem-dlp} are illustrated in Example~\ref{ex:semantic:notions} and Example~\ref{ex:semantic:notions-2} below. 
 
\begin{remark}
Two-valued supported and weakly supported models are defined in~\cite{brass1995characterizations}. Their generalization to the 3-valued case is,  to the best of our knowledge, novel.
An alternative but equivalent definition of a supported model $(x,y)$ is the following: $(x,y)$ is a model, and for every  $p\in {\cal A}_{\cal P}$ such that $(x,y)(p)\neq {\sf F}$, there 
is a rule $\bigvee\Delta\leftarrow \phi$ such that for every other $p'\in\Delta$, $(x,y)(\phi)\geq_t (x,y)(p) >_t (x,y)(p')$. 
\end{remark}

Another common way of providing semantics to dlp's is by Gelfond-Lifschitz reduct~\cite{gelfond1991classical}:

\begin{definition}
\label{def:GL-reduct}
 The GL-transformation $\frac{\cal P}{(x,y)}$ of a disjunctively normal dlp ${\cal P}$ with respect to a consistent interpretation $(x,y)$, is the positive
program obtained by replacing in every rule in ${\cal P}$ of the form
$$p_1\lor\ldots\lor p_n \leftarrow \bigwedge_{i=1}^m q_i\land \bigwedge_{j=1}^n \lnot r_j$$ 
any negated literal $\lnot r_i$ ($1\leq i\leq k$) by: 
(1)~${\sf F}$ if $(x,y)(r_i)={\sf T}$, (2)~${\sf T}$ if $(x,y)(r_i)={\sf F}$, and (3)~${\sf U}$ if $(x,y)(r_i)={\sf U}$.
In other words, replacing $\lnot r_i$ by $(x,y)(\lnot r_i)$.

An interpretation $(x,y)$ is a {\em three-valued stable model\/} of ${\cal P}$ iff it is a $\leq_t$-minimal model of $\frac{{\cal P}}{(x,y)}$.\footnote
{\label{footnote:2-val mod} If $x=y$, $(x,y)$ is called a {\em two-valued\/} stable model of ${\cal P}$. For normal logic programs, the \emph{well-founded model\/} 
is the $\leq_i$-minimal three-valued stable model, which is guaranteed to exist and be unique~\cite{przymusinski1990well,van1991well}.}
\end{definition}

\begin{example}
\label{ex:semantic:notions}
Consider the dlp ${\cal P} \ = \ \{p\leftarrow \lnot p; \quad q\leftarrow \lnot r;\quad r\leftarrow \lnot q;\quad q\lor r\leftarrow\}$.
\begin{itemize}
\item The following interpretations are the (consistent) models of ${\cal P}$:
\[
\begin{matrix}
(\{p,q,r\},\{p,q,r\}),& 
(\{p,r\},\{p,r\}),&
(\{q,r\},\{q,r\}),&
(\{r\},\{p,r\}),&
(\{q\},\{p,q\}),
\smallskip \\
(\{r\},\{p,q,r\}),&
(\{q\},\{p,q,r\}),&
(\{p,r\},\{p,q,r\}),&
(\{p,q\},\{p,q,r\}).
\end{matrix}
\]
Notice that $(\emptyset,\{p,q,r\})$ is \emph{not} a model of ${\cal P}$, since $(\emptyset,\{p,q,r\})(q\lor r)={\sf U} <_t (\emptyset,\{p,q,r\})(\top)$\footnote{We use $\top$ 
to denote the empty body in the rule $q\lor r\leftarrow$. Notice that $(x,y)(\top)={\sf T}$ for any $x,y\subseteq {\cal A}_{\cal P}$.},
thus this interpretation is not a model of $q\lor r\leftarrow$.
\item The following interpretations are the weakly supported models of ${\cal P}$:
\[
\begin{matrix}
(\{q\},\{p,q\}), & (\{r\},\{p,r\}).
\end{matrix}
\]
\item These interpretations are also supported and stable models of ${\cal P}$. Indeed, note for instance that:
\[\frac{{\cal P}}{(\{q\},\{p,q\})} \ = \ \{p\leftarrow {\sf U}' \quad q\leftarrow {\sf T}; \quad r\leftarrow {\sf F}; \quad q\lor r\leftarrow\}.\] 
The minimal (and, in this case also the unique) model of $\frac{{\cal P}}{(\{q\},\{p,q\})}$ is $(\{q\},\{p,q\})$ and thus this interpretation is stable.
\end{itemize}
\end{example}

\begin{example}
\label{ex:semantic:notions-2}
Consider the dlp ${\cal P} \ = \ \{p\lor q\leftarrow q\}.$ 
\begin{itemize}
\item Then the following interpretations are weakly supported models of ${\cal P}$:
\[
\begin{matrix}
(\emptyset,\emptyset),& 
(\emptyset,\{q\}),&
(\{q\},\{q\}),&
(\emptyset,\{p,q\}),&
(\{q\},\{p,q\}),&
(\{p,q\},\{p,q\})
\end{matrix}\]
\item Of these interpretations, only the following are supported:
\[
\begin{matrix}
(\emptyset,\emptyset),& 
(\emptyset,\{q\}),&
(\{q\},\{q\})
\end{matrix}\]
One can see that e.g.\ $(\{p,q\},\{p,q\})$ is \emph{not\/} weakly supported, as there is no rule for which $p$ is the only atom that is true and occurs in the head 
(as also $q$ occurs in $p\lor q\leftarrow q$ and is true according to  $(\{p,q\},\{p,q\})$).
\item The only stable model of ${\cal P}$ is $(\emptyset,\emptyset)$. This can be seen by observing that for any interpretation $(x,y)$ it holds that $\frac{\cal P}{(x,y)}={\cal P}$, 
and that the minimal model of ${\cal P}$ is $(\emptyset,\emptyset)$. 
\end{itemize}
\end{example}

\subsection{Approximation Fixpoint Theory}
\label{sec:AFT}

We now recall basic notions from approximation fixpoint theory (AFT), as described by Denecker, Marek and Truszczy{\'n}ski~\cite{denecker2000approximations}.  
As we have already noted, AFT introduces constructive techniques for approximating the fixpoints of an operator $O$ over a lattice $L= \tup{{\cal L},\leq}$. 
This is particularly useful when $O$ is non-monotonic (as is often the case in logic programming, default logic and abstract argumentation, and other disciplines
for non-monotonic reasoning in AI), in which case such operators are not guaranteed to even have a fixpoint, or a unique least fixpoint that can be constructively obtained.
AFT generalizes the principles for the construction of fixpoints to the non-monotonic setting, by working with \emph{approximations\/} of such operators on a {\em bilattice\/} \cite{AA96,Be77a,Be77b,Fi06,Fi20,Gi88}, constructed on the basis of $L$. 

\begin{definition}
\label{def:bilattice}
 Given a lattice $L = \tup{{\cal L},\leq}$, a {\em bilattice\/} is the structure $L^2 =\tup{{\cal L}^2,\leq_i,\leq_t}$, in which
${\cal L}^2 = {\cal L} \times {\cal L}$, and for every $x_1,y_1,x_2,y_2 \in {\cal L }$, \smallskip \\
$\bullet$ $(x_1,y_1) \leq_i (x_2,y_2)$ if $x_1 \leq x_2$ and $y_1 \geq y_2$, \smallskip \\
$\bullet$ $(x_1,y_1) \leq_t (x_2,y_2)$ if $x_1 \leq x_2$ and $y_1 \leq y_2$.\footnote{Recall that we use small letters to denote elements of lattice, capital letters to denote 
sets of elements, and capital calligraphic letters to denote sets of sets of elements (Table~\ref{tab:set-notations}).}
 \end{definition}
 
Bilattices of the form $L^2$ are used for defining operators that approximate operators on $L$.

An {\em approximating operator\/} ${\cal O}:{\cal L}^2\rightarrow {\cal L}^2$ of an operator $O:{\cal L}\rightarrow {\cal L}$ is an operator that maps every 
approximation $(x,y)$ of an element $z$ to an approximation $(x',y')$ of another element $O(z)$, thus approximating the behavior of the approximated operator $O$.  
Approximation operators may be viewed as combinations of two operators: $({\cal O}(.,.))_1$ and $({\cal O}(.,.))_2$ which calculate, respectively, a \emph{lower\/} and 
an \emph{upper\/} bounds for the value of $O$ (where, as usual, $(x,y)_1$ respectively $(x,y)_2$ represents the first respectively second component of $(x,y)$). 
To avoid clutter, we will also denote $({\cal O}(x,y))_1$ by ${\cal O}_l(x,y)$  and $({\cal O}(x,y))_2$ by ${\cal O}_u(x,y)$.

Two fundamental requirements on approximating operators are the following:
\begin{enumerate}
     \item {\sl $\leq_i$-monotonicity:} the values of an approximating operator should be more precise as its arguments are more precise, and 
     \item {\sl exactness:} exact arguments are mapped 
              to exact values.
\end{enumerate}              
These requirements result in the following definition:

\begin{definition}
\label{def:approx-notions}
Let $O:{\cal L}\rightarrow {\cal L}$ and ${\cal O}:{\cal L}^2\rightarrow {\cal L}^2$.
\begin{itemize}
 \item  ${\cal O}$ is {\em $\leq_i$-monotonic\/}, if when $(x_1,y_1)\leq_i(x_2,y_2)$, also ${\cal O}(x_1,y_1)\leq_i {\cal O}(x_2,y_2)$; 
           ${\cal O}$ is \emph{approximating\/}, if it is $\leq_i$-monotonic and for any $x\in {\cal L}$, ${\cal O}_l(x,x) = {\cal O}_u(x,x)$.\footnote
           {In some papers (e.g.,~\cite{denecker2000approximations}), an approximation operator is defined as a symmetric $\leq_i$-monotonic operator, 
           i.e.\ a $\leq_i$-monotonic operator s.t.\ for every $x,y\in {\cal L}$, ${\cal O}(x,y)=({\cal O}_l(x,y),{\cal O}_l(y,x))$ for some 
           ${\cal O}_l:{\cal L}^2\rightarrow {\cal L}$. However, the weaker condition we take here (taken from \cite{denecker2002ultimate}) is actually 
           sufficient for most results on AFT. \label{footnote:symetry} }
 \item ${\cal O}$ is an {\em approximation\/} of $O$, if it is $\leq_i$-monotonic and
         ${\cal O}$ \emph{extends} $O$, that is: ${\cal O}(x,x) = (O(x), O(x))$ (for every $x\in {\cal L}$).
\end{itemize}
\end{definition}

\begin{remark}
 One can define an approximating operator ${\cal O}$ without having to specify which operator $O$ 
it approximates, and indeed it will often be convenient to study approximating operators without having to refer to the approximated operator. 
However, one can easily obtain the operator $O$ that ${\cal O}$ approximates by 
letting: $O(x) = {\cal O}_l(x,x)$.
 \end{remark}

Another operator that has a central role in AFT and which is used for expressing the semantics of many non-monotonic formalisms is the 
\emph{stable operator\/}, defined next.

\begin{definition}
\label{def:stable-op}
For a complete lattice $L = \tup{{\cal L},\leq}$,  let ${\cal O}:{\cal L}^2 \rightarrow {\cal L}^2$ be an approximating operator.
We denote: ${\cal O}_{l}(\cdot,y) = \lambda x.{\cal O}_{l}(x,y)$ and ${\cal O}_{u}(x,\cdot) = \lambda y.{\cal O}_{u}(x,y)$,
i.e.: ${\cal O}_{l}(\cdot,y)(x) = {\cal O}_{l}(x,y)$ and ${\cal O}_{u}(x,\cdot)(y) = {\cal O}_{u}(x,y)$.
The \emph{stable operator for ${\cal O}$\/} is: $S({\cal O})(x,y)=(\mathit{lfp}({\cal O}_l(.,y)),\mathit{lfp}({\cal O}_u(x,.))$. 
\end{definition}

Stable operators capture the idea of minimizing truth, since for any $\leq_i$-monotonic operator ${\cal O}$ on ${\cal L}^2$,
the fixpoints of the stable operator $S({\cal O})$ are  $\leq_t$-minimal fixpoints of ${\cal O}$ \cite[Theorem~4]{denecker2000approximations}.
Altogether, the following semantic notions are obtained:

\begin{definition}
\label{def:AFT-operators}
 Given a complete lattice $L = \tup{{\cal L},\leq}$,  let ${\cal O}:{\cal L}^2 \rightarrow {\cal L}^2$ be an approximating operator. \jesse{Then:}
\begin{itemize} 
\item $(x,y)$ is a \emph{Kripke-Kleene fixpoint\/} of ${\cal O}$ if $(x,y) =\lfp_{\leq_i}({\cal O}(x,y))$.  
\item $(x,y)$ is a \emph{three-valued stable fixpoint\/} of ${\cal O}$ if $(x,y)= S({\cal O})(x,y) $.  
\item $(x,x)$ is a \emph{two-valued stable fixpoints\/} of ${\cal O}$ if $(x,x)= S({\cal O})(x,x)$. 
\item $(x,y)$ is the \emph{well-founded  fixpoint\/} of ${\cal O}$ if it is the $\leq_i$-minimal (3-valued) stable fixpoint of ${\cal O}$. 
\end{itemize}
\end{definition}

Denecker, Marek and Truszscy\'nski~\cite{denecker2000approximations} show that every approximation operator admits a unique $\leq_i$-minimal stable fixpoint.
Pelov, Denecker and Bruynooghe~\cite{pelov2007well} show that for normal logic programs, the fixpoints based on the four-valued immediate consequence operator
for a logic program give rise to the following correspondences: the three-valued stable models coincides with the three-valued semantics as defined by Przymusinski~\cite{przymusinski1990well}, 
the well-founded model coincides with the homonymous semantics~\cite{przymusinski1990well,van1991well}, and the two-valued stable models coincide with the two-valued 
(or total) stable models of a logic program.

\begin{example}
For a normal logic program ${\cal P}$, an approximation of the operator $\IC_{\cal P}$~\cite{pelov2007well} can be obtained by first constructing a lower bound operator 
as follows:\footnote{Notice that the lattice under consideration is $\langle 2^{{\cal A}_{\cal P}},\subseteq\rangle$, i.e., elements of the lattice are \emph{sets\/} of atoms.\label{footnote:sets:as:elements}}
\[ {\cal IC}^l_{\cal P}(x,y)= \{p\in {\cal A}_{\cal P}\mid p\leftarrow \phi\in {\cal P}, (x,y)(\phi)\geq_t {\sf C}\} \]
and then defining:
\[ {\cal IC}_{\cal P}(x,y)= ({\cal IC}_{\cal P}^l(x,y),{\cal IC}_{\cal P}^l(y,x)) \]
Notice that the upper bound ${\cal IC}_{\cal P}^u(y,x)$ is defined as ${\cal IC}_{\cal P}^l(y,x)$, i.e.\ ${\cal IC}_{\cal P}$ is a \emph{symmetric\/} operator (see also 
Footnote~\ref{footnote:symetry}). Pelov, Denecker and Bruynooghe~\cite{pelov2007well} have shown that this operator approximates $\IC_{\cal P}$ for normal logic 
programs ${\cal P}$.

We now illustrate the behaviour of this operator with the following logic program: 
\[ {\cal P} \ = \ \{ p\leftarrow \lnot q; \quad q\leftarrow \lnot p; \quad r\leftarrow r\}. \]
$\IC_{\cal P}$ is thus an operator over the lattice $\langle 2^{\{p,q,r\}},\subseteq\rangle$. Now,
\begin{itemize}
\item Concerning the approximation ${\cal IC}_{\cal P}$ of $\IC_{\cal P}$, it holds, e.g., that: 
    \begin{itemize}
        \item ${\cal IC}^l_{\cal P}(\emptyset,\{p,q,r\})= \emptyset \mbox{ as }(\emptyset,\{p,q,r\})(\lnot q)=(\emptyset,\{p,q,r\})(\lnot p)=(\emptyset,\{p,q,r\})(r)={\sf U}$.
        \item ${\cal IC}^l_{\cal P}(\{p,q,r\},\emptyset)=\{p,q,r\} \mbox{ as }(\{p,q,r\},\emptyset)(\lnot q)=(\{p,q,r\},\emptyset)(\lnot p)=(\{p,q,r\},\emptyset)(r)={\sf C}$.
    \end{itemize}
\item We now illustrate the stable operator. 
    \begin{itemize}
         \item It holds that $\lfp({\cal IC}_{\cal P}(.,\{p\})=\{p\}$, and therefore $S({\cal IC}^l_{\cal P})(\{p\}) = \{p\}$. 
         \item By symmetry of ${\cal IC}_{\cal P}$, $(\{p\},\{p\})$ is a stable fixpoint of ${\cal IC}_{\cal P}$. 
         \item Likewise, it can be observed that $(\emptyset,\{p,q\})$ and $(\{q\},\{q\})$ are stable fixpoints of ${\cal IC}_{\cal P}$. 
         \item It follows that  $(\emptyset,\{p,q\})$ is the well-founded fixpoint of ${\cal IC}_{\cal P}$.
    \end{itemize}
\end{itemize}    
We shall generalize this operator to the disjunctive case in Section~\ref{sec:nd:aft}.
\end{example}

\section{Non-Deterministic Operators and Approximations}
\label{sec:nd:aft}

In this section, we generalize approximation fixpoint theory to allow for non-deterministic operators. In Section~\ref{sec:non-deterministic:operators} we formally define 
non-deterministic operators and the necessary order-theoretic background. In Section~\ref{sec:ndaos} we then define non-deterministic approximation operators and 
show some basic results on these operators.

\subsection{Non-Deterministic Operators}\label{sec:non-deterministic:operators}

In order to characterize (two-valued) semantics for disjunctive logic programming, in~\cite{pelov2004semantics} Pelov and Truszczy{\'n}ski introduced 
the notion of non-deterministic operators and accordingly extended AFT to non-deterministic AFT. 

\begin{definition}%
\label{def:non-deterministic-operator}
A {\em non-deterministic operator on ${\cal L}$} is a function $O : {\cal L}\rightarrow \wp({\cal L}) \setminus \{\emptyset\}$. 
\end{definition}

Intuitively, a non-deterministic operator assigns to every element $x$ of ${\cal L}$ a (nonempty) set of {\em choices\/} $O(x)=\{y_1,y_2,\ldots\}$ which can 
be seen as equally plausible alternatives of the outcome warranted by $x$. Thus, non-deterministic operators allow for multiple options or choices in their output 
(reflected by their co-domain being $\wp({\cal L})$). Just like deterministic operators, it is required that every element $x$ is mapped to at least one choice 
(which might be the $\leq$-least element $\bot$, if it exists). As an example of non-determinism, consider an activated disjunctive rule $\bigvee\Delta\leftarrow \phi$ 
(i.e., a rule for which the body is true and its head is a disjunction), requiring that at least one among the disjuncts $\delta\in\Delta$ is true.

\begin{example}%
\label{example:operator:disj:lp} 
 The operator $\IC_{\cal P}$ from Definition~\ref{def:operator:disj:lp} is a non-deterministic operator on the lattice $\tup{\wp({\cal A}_{\cal P}),\subseteq}$.
\end{example}

As the ranges of non-deterministic operators are {\em sets\/} of lattice elements, one needs a way to compare them. Next, we recall two such relations,
known as the {\em Smyth order\/}~\cite{smyth1976powerdomains} and the {\em Hoare order\/} (used in the context of DLP in several works (see, 
e.g.,~\cite{alcantara2005well,fernandez1995bottom})).

\begin{definition}%
\label{def:Smith-preorder}
Let $L = \tup{{\cal L},\leq}$ be a lattice, and let $X,Y \in \wp({\cal L})$. Then:
\begin{itemize}
     \item $X \preceq^S_L Y$ if for every $y\in Y$ there is an $x\in X$ such that $x\leq y$.
     \item  $X \preceq^H_L Y$ if for every $x\in X$ there is a $y\in Y$ such that $x\leq y$.
\end{itemize}
\end{definition}

Here, the sets of lattice elements $X$ and $Y$ represent values of a non-deterministic operator (i.e., these are non-deterministic states). Thus, according to the intuition described above, 
each one is a set of {\em choices\/}. Accordingly, $\preceq^S_L$ means that for every choice $y$ in $Y$, there is a $\leq$-smaller choice $x$ in $X$. Likewise, $\preceq^H_L$ means 
that for every choice $x$ in $X$, there is a $\leq$-greater choice $y$ in $Y$. In other words, $\preceq^H_L$ and $\preceq^S_L$ allow to compare non-deterministic states on the basis 
of the $\leq$-relationship of their constituent elements in $L$. We will see below that $\preceq^S_L$ is well-suited to compare lower bounds, whereas $\preceq^H_L$ is well-suited 
to compare upper bounds.

\begin{remark}
\label{rem:properties:of:smyth}
Both $\preceq^S_L$ and $\preceq^H_L$ are preorders (i.e., reflexive and transitive) on $\wp({\cal L})$.
 \end{remark}

\subsection{Non-Deterministic Approximation Operators}
\label{sec:ndaos}

We now develop a notion of approximation of non-deterministic operators. Such an approximation generalizes the analogue approximation of deterministic operators explained in 
Section~\ref{sec:AFT}. The benefits of such an approximation will be demonstrated in Sections~\ref{sec:theory:of:ndao} and~\ref{sec:stable:semantics}.

Before describing the formal details, we explain the intuition behind the approximation. Recall that, in deterministic AFT, an operator over $L$ is approximated by an 
\emph{approximation operator\/} that specifies a lower bound and an upper bound of an approximated element. Thus, an approximation operator essentially consists 
of two operators over $L$, the a lower bound operator ${\cal O}_l$ and the upper bound operator ${\cal O}_u$. We generalize this idea to the non-deterministic case.
As in deterministic AFT, a non-deterministic approximation operator can be seen as consisting of a non-deterministic lower bound operator and an non-deterministic upper 
bound operator. A non-deterministic approximation ${\cal O}$ of a non-deterministic operator $O$, maps a pair $(x,y)$ (intuitively representing an approximation of a 
single value $z$) to a pair of sets $X,Y\subseteq {\cal L}$ (intuitively representing sets of lower bounds $X$ and upper bounds $Y$ on the non-deterministic choices $O(z)$).
Thus, an approximation operator ${\cal O}$ is of the type ${\cal L}^2\rightarrow \wp({\cal L})\times\wp({\cal L})$.

As in the deterministic case, it is natural to assume two formal properties of the approximating operators (recall Definition~\ref{def:approx-notions}).
Below, we adjust the requirements in Definition~\ref{def:approx-notions} to the non-deterministic case, using the order relations in Definition~\ref{def:Smith-preorder}: 
\begin{enumerate}
     \item {\sl Exactness}: if $(x,y)$ is an approximation of $z$, every non-deterministic choice $z'\in O(z)$ should have at least one lower bound in 
              $x'\in {\cal O}_l(x,y)$  and at least one upper bound $y'\in {\cal O}_u(x,y)$. In other words, ${\cal O}_l(x,y)\preceq^S_L O(z)$
              and $O(z)\preceq^H_L {\cal O}_u(x,y)$.
              Informally, every choice in $O(z)$ is in between some lower bound and upper bound of ${\cal O}(x,y)$. In the extreme case where the h
              argument is an exact pair $(x,x)$, this means that ${\cal O}(x,x)=O(x,x)\times O(x,x)$  and so, for an exact pair $(x,x)$, both the lower and upper bound operator 
              coincide with the approximated operator $O$. Thus, ${\cal O}(x,x)$ represents a single set of choices. We shall use this as a defining condition: 
              for exact pairs, the lower and upper bound coincide. Intuitively, exact inputs give rise to exact (but non-deterministic) outputs. We will call such an operator \emph{\exactcompliant}.
    \item {\sl Monotonicity}: just like in the deterministic case, a non-deterministic approximation operator should be expected to be monotonic w.r.t.\ 
             the information ordering: more precise inputs give rise to more precise outputs. To make this notion formally precise, we need a way to compare the precision 
             of pairs of sets (interpreted as a set of lower bounds and a set of upper bounds). A set of lower bounds $X_1$ is more precise than a set of lower bounds 
             $X_2$ if, for every lower bound $x_1$ in the more precise set $X_1$, there is a less precise (i.e.\ $\leq$-smaller) lower bound $x_2$ in the less precise set $X_2$. 
             Thus: $X_2\preceq^S_L X_1$. Likewise, a set of upper bounds $Y_1$ is more precise than 
             a set of upper bounds $Y_2$ if, for every upper bound $y_1$ in the more precise set $Y_1$, there is a less precise (i.e.\ $\leq$-higher) upper bound $y_2$ 
             in the less precise set $Y_2$. Thus: $Y_1\preceq^H_L Y_2$. Altogether, a pair of bounds $(X_1,Y_1)$ is more precise than a second pair of bounds $(X_2,Y_2)$ 
             if the lower bounds are compared w.r.t.\ the Smyth-order $\preceq^S_L$ and the upper bounds are compared w.r.t.\ the Hoare-order $\preceq^H_L$.
             This idea was defined by Alc\^antara, Dam\'asio and Moniz Pereira~\cite{alcantara2005well} on pairs of sets, and is generalized here to an algebraic setting:

             \begin{definition}%
              Given some $X_1,X_2,Y_1,Y_2\subseteq {\cal L}$, $X_1\times Y_1 \preceq^A_i X_2\times Y_2$ iff $X_1\preceq^S_L X_2$ and $Y_2\preceq^H_L Y_1$. 
             \end{definition}
\end{enumerate}

We are now ready to define a \emph{non-deterministic approximation operator\/}: it is an operator ${\cal O}:{\cal L}^2\rightarrow \wp({\cal L})\times \wp({\cal L}) $ assigning to  every pair 
$(x,y)$ a set of lower bounds and upper bounds, that is $\preceq^A_i$-monotonic for which exact inputs give rise to exact outputs. Formally:

\begin{definition}
\label{def:ndao}
Let $L=\tup{{\cal L},\leq}$ be a lattice. An operator ${\cal O}:{\cal L}^2\rightarrow \wp({\cal L}){\setminus\emptyset}\times \wp({\cal L}){\setminus\emptyset}$ 
is called a {\em non-deterministic approximating operator\/} (ndao, for short), if satisfies the following properties:
\begin{itemize}
    \item ${\cal O}$ is $\preceq^A_i$-monotonic. 
    \item ${\cal O}$ is \emph{\exactcompliant}, i.e., for every $x\in {\cal L}$, ${\cal O}(x,x)={\cal O}_l(x,x)\times {\cal O}_\jessen{l}(x,x)$.\footnote{Recall that we 
             denote by ${\cal O}_l$  the operator defined by ${\cal O}_l(x,y)={\cal O}(x,y)_1$, and likewise by ${\cal O}_u$ the operator defined by  ${\cal O}_u(x,y)={\cal O}(x,y)_2$.}
\end{itemize}
\end{definition}

We also say that a ndao ${\cal O}$ is \emph{an approximation} of the non-deterministic operator, defined as $O(x)={\cal O}_l(x,x)$ (for every $x \in {\cal L}$)

\begin{remark}
\label{remark:deterministic}
We shall show below (Proposition~\ref{prop:deterministic:aft}) that Definition~\ref{def:ndao} extends Definition~\ref{def:approx-notions}: when an ndao is deterministic, 
in the sense that ${\cal O}(x,y)$ is singleton for every $x,y\in {\cal L}$, an ndao reduces to an approximation operator.
\end{remark}

\begin{remark}
It is sometime useful to assume the following property of ndaos:
${\cal O}$ is {\em symmetric\/}, if ${\cal O}_l(x,y)={\cal O}_u(y,x)$ for any $x,y\subseteq {\cal L}$.(Or, equivalently, if ${\cal O}_u(x,y)={\cal O}_l(y,x)$ for every every 
$x,y \in {\cal L}$). Notice that symmetric operators are \exactcompliant.
Just like in deterministic AFT, the assumption of symmetry is \emph{not} essential and we do not assume it unless specific results require it (see also Foonote~\ref{footnote:symetry}).
\end{remark}

We now give an example of an ndao in the context of disjunctive logic programming.
The operator ${\cal IC}_{\cal P}$ is constructed on the basis of the operators $\HRc_{\cal P}^l$ and $\HRc_{\cal P}^u$, which intuitively constitute a lower bound and 
and an upper bound on the activated heads. In more detail, $\HRc_{\cal P}^l(x,y)$ contains all the heads of rules whose body is at least (according to $\geq_t$) contradictory, i.e., 
whose body is ${\sf T}$ or ${\sf C}$. Likewise, $\HRc_{\cal P}^u(x,y)$ contains heads of rules whose body is  at least (according to $\geq_t$)
undecided, i.e., whose body is ${\sf T}$ or ${\sf U}$. The lower (respectively upper) bounds are then constructed by taking all sets of atoms containing only atoms in at least one 
of the heads in $\HRc_{\cal P}^l(x,y)$ (respectively $\HRc_{\cal P}^u(x,y)$), and containing at least one atom for every head in $\HRc_{\cal P}^l(x,y)$ (respectively $\HRc_{\cal P}^u(x,y)$).

\begin{definition}
\label{def:IC_P-ndo}
For a dlp ${\cal P}$ and an interpretation $(x,y)$, we define:
\begin{itemize}
  \item $\HRc^l_{\cal P}(x,y) = \{ \Delta \mid \bigvee\!\Delta \leftarrow \phi\in {\cal P}, (x,y)(\phi)\geq_t {\sf C}\}$, 
  \item $\HRc^u_{\cal P}(x,y) = \{ \Delta \mid \bigvee\!\Delta \leftarrow \phi\in {\cal P}, (x,y)(\phi)\geq_t {\sf U}\}$, 
  \item ${\cal IC}^l_{\cal P}(x,y)=\{x_1\subseteq \bigcup\HRc^l_{\cal P}(x,y) \mid \forall \Delta\in \HRc^l_{\cal P}(x,y), \ x_1 \cap \Delta \neq \emptyset \}$, 
  \item ${\cal IC}^u_{\cal P}(x,y)=\{y_1\subseteq \bigcup\HRc^u_{\cal P}(x,y) \mid \forall \Delta\in \HRc^u_{\cal P}(x,y), \ y_1 \cap \Delta \neq \emptyset \}$, 
  \item ${\cal IC}_{\cal P}(x,y)=({\cal IC}^l_{\cal P}(x,y), {\cal IC}^u_{\cal P}(x,y))$. 
\end{itemize}
\end{definition}

\begin{example}
\label{example:operator:disj:lp-2}
The operator ${\cal IC}_{\cal P}$  is an approximation of the non-deterministic operator 
$IC_{\cal P}$ in Example~\ref{example:operator:disj:lp} (and Definition~\ref{def:operator:disj:lp}). Furthermore, as we show next, it is a symmetric operator.
\end{example}

\begin{proposition}
${\cal IC}_{\cal P}$ is a symmetric ndao that approximates $IC_{\cal P}$.
\end{proposition}
\begin{proof}
It is clear that for any $x\in {\cal L}$, $\HRc_{\cal P}^l(x,x)=\HRc_{\cal P}^{u}(x,x)=\HR_{\cal P}(x,x)$ (as for any $\phi$, $(x,x)(\phi)\in \{{\sf T},{\sf F}\}$). 
Thus, ${\cal IC}_{\cal P}$ approximates $IC_{\cal P}$ and is \exactcompliant. We now show that it is $\preceq^A_i$-monotonic. For this, consider some 
$(x_1,y_1)\leq_i (x_2,y_2)$. We show by induction on $\phi$ that if $(x_1,y_1)(\phi)\geq_t {\sf C}$ then $(x_2,y_2)(\phi)\geq_t {\sf C}$. 
The base case is clear as $\phi\in x_1$ and $x_1\subseteq x_2$ implies $\phi\in x_2$. For the inductive case, notice that the cases for $\phi=\phi_1\land \phi_2$ 
and $\phi=\phi_1\lor \phi_2$ follow immediately from the inductive hypothesis. 
Suppose now thar $\phi=\lnot \phi_1$. $(x_1,y_1)(\lnot \phi_1)\geq_t{\sf C}$ means that $\phi_1\not\in y_1$. Since $y_2\subseteq y_1$, also $\phi_1\not\in y_2$.

We now show $\preceq^A_i$-monotonicity, which follows immediately from $\HRc_{\cal P}^l(x_1,y_1)\subseteq \HRc_{\cal P}^l(x_2,y_2)$. To see the latter, suppose that
$\Delta\in \HRc_{\cal P}^l(x_1,y_1)$, i.e.\ for some $\bigvee\Delta\leftarrow \phi$, $(x_1,y_1)(\phi)\geq_t {\sf C}$. Then $(x_2,y_2)(\phi)\geq_t{\sf C}$ and thus 
$\Delta\in \HRc_{\cal P}^l(x_2,y_2)$.

We finally show that ${\cal IC}_{\cal P}$ is symmetric. For this, we show the following lemma:

\begin{lemma}
\label{lemma:symmetry:of:ic}
For any $x,y\subseteq {\cal A}_{\cal P}$:
\begin{enumerate}
\item $(x,y)(\phi)= {\sf T}$ iff $(y,x)(\phi)= {\sf T}$, 
\item $(x,y)(\phi)= {\sf F}$ iff $(y,x)(\phi)= {\sf F}$, 
\item $(x,y)(\phi)= {\sf C}$ iff $(y,x)(\phi)= {\sf U}$.
\end{enumerate}
\end{lemma}

\begin{proof}
We show the third item, the first two items are shown similarly. We show this by induction on the structure of $\phi$. For the base case, 
let $\phi \in {\cal A}_{\cal P}$. Then $(x,y)(\phi)={\sf C}$ iff $\phi\in x\setminus y$, and, since $(y,x)(\phi)= {\sf U}$ iff $\phi\in x\setminus y$, the base case is proven.
For the inductive case, notice that the cases where $\phi=\phi_1\lor \phi_2$ or $\phi=\phi_1\land \phi_2$ follow immediately from the inductive hypothesis. 
Suppose that $\phi=\lnot \phi_1$. Since $(x,y)(\lnot\phi_1)={\sf C}$ iff $(x,y)(\phi_1)={\sf C}$, and $(y,x)(\lnot \phi_1)={\sf U}$ iff $(y,x)(\phi_1)={\sf U}$, we obtain
by the induction hypothesis that $(x,y)(\lnot\phi_1)={\sf C}$ iff $(y,x)(\lnot\phi_1)={\sf U}$.
\end{proof}
From Lemma~\ref{lemma:symmetry:of:ic} it follows that $\HRc_{\cal P}^l(x,y)=\HRc_{\cal P}^u(y,x)$ and thus ${\cal IC}_{\cal P}^l(x,y)={\cal IC}_{\cal P}^u(y,x)$.
\end{proof}

We now illustrate the operator ${\cal IC}_{\cal P}$ using two simple disjunctive logic program:

\begin{example}
\label{example:operator:disj:lp-2-b}
Consider the following dlp: ${\cal P}=\{ p\lor q\leftarrow\lnot q\}$.
\begin{description}
\item [] The corresponding operator ${\cal IC}^l_{\cal P}$ behaves as follows:
\begin{itemize}
\item For any interpretation $(x,y)$ for which $q\in x$, $\HRc^l_{\cal P}(x,y)=\emptyset$ and thus ${\cal IC}^l_{\cal P}(x,y)=\{\emptyset\}$.
\item For any interpretation $(x,y)$ for which $q\not\in x$, $\HR^l_{\cal P}(x,y)=\{\{p,q\}\}$ and thus ${\cal IC}^l_{\cal P}(x,y)=\{\{p\},\{q\},\{p,q\}\}$.
\end{itemize}
\item [] Since ${\cal IC}_{\cal P}^l(x,y)={\cal IC}_{\cal P}^u(y,x)$ (by Lemma~\ref{lemma:symmetry:of:ic}), this means that ${\cal IC}_{\cal P}$ behaves as follows:
\begin{itemize}
\item For any $(x,y)$ with $q\not\in x$ and $q\not\in y$, ${\cal IC}_{\cal P}(x,y)=\{\{p\},\{q\},\{p,q\}\}\times \{\{p\},\{q\},\{p,q\}\}$,
\item For any $(x,y)$ with $q\not\in x$ and $q\in y$, ${\cal IC}_{\cal P}(x,y)=\{\emptyset\}\times \{\{\{p\},\{q\},\{p,q\}\}$,
\item For any $(x,y)$ with $q\in x$ and $q\not\in y$, ${\cal IC}_{\cal P}(x,y)=\{\{p\},\{q\},\{p,q\}\}\times \{\emptyset\}$, and
\item For any $(x,y)$ with $q\not\in x$ and $q\in y$, ${\cal IC}_{\cal P}(x,y)=\{(\emptyset,\emptyset)\}$.
\end{itemize}
\end{description}
\end{example}

\begin{example} %
Consider the dlp from Example~\ref{ex:semantic:notions}. Then ${\cal IC}^l_{\cal P}$ behaves as follows (for arbitrary $y\subseteq {\cal A}_{\cal P}$):
\def\arraystretch{1,3}\tabcolsep=10pt
\begin{center}
\begin{tabular}{ll}
$x$& ${\cal IC}^l_{\cal P}(x,y)$ \\ \hline
$\emptyset$ & $\{\{p,q,r\}\}$
\\ $\{p\}$& $\{\{q,r\}\}$
\\ $\{q\}$& $\{\{p,q\},\{p,q,r\}\}$
\\ $\{r\}$& $\{\{p,r\},\{p,q,r\}\}$
\\ $\{p,q\}$& $\{\{q\},\{q,r\}\}$
\\ $\{p,r\}$& $\{\{r\},\{q,r\}\}$
\\ $\{q,r\}$& $\{\{p,q\},\{p,r\},\{p,q,r\}\}$
\\ $\{p,q,r\}$& $\{\{q\},\{r\},\{q,r\}\}$
\end{tabular}
\end{center}
Again, by the symmetry of ${\cal IC}_{\cal P}$, the behaviour of ${\cal IC}_{\cal P}$ can be easily derived on the basis of ${\cal IC}^l_{\cal P}$.
\end{example}

\begin{remark}
\label{remark:minimality:in:the:op}
In the literature (e.g., \cite{antic2013hex,pelov2004semantics}),  similar non-deterministic four-valued operators have been defined to characterize 
the semantics of disjunctive logic programs, inspired by deterministic approximation fixpoint theory. In some of these operators  minimality of the 
image of the operator was built-in. In our setting, this is defined as follows:
\begin{itemize}
     \item ${\cal IC}^{m,l}_{\cal P}(x,y)=\min_\subseteq(\{v \mid \forall \Delta\in \HRc^l_{\cal P}(x,y), \ v \cap \Delta \neq \emptyset \})$,\footnote{Recall that 
              $\min_\subseteq(X)=\{x\in X\mid \not\exists y\in X: y\subset x\}$.}
     \item ${\cal IC}_{\cal P}^{m}(x,y)=({\cal IC}^{m,l}_{\cal P}(x,y), {\cal IC}^{m,l}_{\cal P}(y,x))$. 
\end{itemize}             
However, there are some issues with this approach, e.g., that the operator ${\cal IC}_{\cal P}$ is \emph{not} $\preceq^A_i$-monotonic. To see this, consider the 
program ${\cal P}=\{a\lor b\leftarrow; \ \  a\leftarrow c\}$ and the two interpretations $(\emptyset,\{a,b,c\})$ and $(\emptyset,\{a,b\})$. Then we have:
\begin{itemize}
\item ${\cal IC}^{m,l}_{\cal P}(\emptyset,\{a,b,c\})=\{\{a\},\{b\}\}$ and  ${\cal IC}^{m,l}_{\cal P}(\{a,b,c\},\emptyset)=\{\{a\}\}$ 
         (the latter since $(\{a,b,c\},\emptyset)(c)={\sf C}$ and thus $\HRc^l_{\cal P}(\{a,b,c\},\emptyset)=\{\{a,b\},\{a\}\}$). 
\item ${\cal IC}^{m,l}_{\cal P}(\emptyset,\{a,b\})=\{\{a\},\{b\}\}$ and  ${\cal IC}^{m,l}_{\cal P}(\{a,b\},\emptyset)=\{\{a\},\{b\}\}$ 
        (the latter since $(\{a,b,c\},\emptyset)(c)={\sf F}$ and thus $\HRc^l_{\cal P}(\{a,b\},\emptyset)=\{\{a,b\}\}$). 
\end{itemize}
It follows that ${\cal IC}^{m,l}_{\cal P}(\{a,b\},\emptyset) \not\preceq^H_L {\cal IC}^{m,l}_{\cal P}(\{a,b,c\},\emptyset)$ i.e.,
$\{\{a\},\{b\}\} \not\preceq^H_L \{\{a\}\}$, since $\{b\} \not\subseteq \{a\}$. Hence, 
${\cal IC}^{m,u}_{\cal P}(\emptyset,\{a,b\}) = {\cal IC}^{m,l}_{\cal P}(\{a,b\},\emptyset) \not\preceq^H_L 
 {\cal IC}^{m,l}_{\cal P}(\{a,b,c\},\emptyset) = {\cal IC}^{m,u}_{\cal P}(\emptyset,\{a,b,c\})$. 

Thus, $(\emptyset,\{a,b,c\}) \leq_i (\emptyset,\{a,b\})$, yet ${\cal IC}^{m}_{\cal P}(\emptyset,\{a,b,c\}) \not\preceq^A_i {\cal IC}^{m}_{\cal P}(\emptyset,\{a,b\})$,
 and so ${\cal IC}^{m}_{\cal P}$ is not $\preceq^A_i$-monotonic.
 
Note that this is a counter-example to a wrong claim made in~\cite[Example~2]{DBLP:conf/kr/HeyninckA21}, about ${\cal IC}^m_{\cal P}$ as defined above being a $\preceq^A_i$-monotonic operator.

In the work of Pelov and Truszczy{\'n}ski~\cite{pelov2004semantics}, and of Anti\'c, Eiter and Fink \cite{antic2013hex}, $\preceq^A_i$-monotonicity is not studied, 
and only $\preceq^S_L$-monotonicity of the lower bound operator is shown and used (i.e., if $(x_1,y_1)\leq_i (x_2,y_2)$ then 
${\cal IC}^{m,l}(x_1,y_1)\preceq^S_L {\cal IC}^{m,l}(x_2,y_2)$). Thus, the above counter-example does not invalidate this claim (i.e., this counter-example does not 
show that ${\cal IC}^{m,l}$ is not $\preceq^S_L$-monotonic) .For two-valued stable semantics, this suffices, but as we shall see in what follows, when moving to three- 
and four-valued semantics, full $\preceq^A_i$-monotonicity is needed.

Altogether, this example shows that requiring minimality in the non-deterministic approximation operator leads to undesirable behavior. 
This is perhaps not surprising, as in deterministic approximation theory minimization is also not ensured by the operator itself, but by taking 
the stable fixpoints of an operator. In our work, minimization is not demanded in the definitions of the operators, but rather is achieved in the 
definitions of stable operators and fixpoints. And indeed,  we will be able to show that this works well, as stable fixpoints will be shown to be 
$\leq_t$-minimal fixpoints of an ndao (see Proposition~\ref{prop:stable:is:minimal:fp}). 
\end{remark}

\begin{remark}
\label{rem:convex-set}
A pair of sets can alternatively be viewed as a {\em convex set\/}.  A convex set is a set $X\subseteq {\cal L}$ that contains no ``holes'', i.e., for 
any $x,y\in X$, if $x\leq z \leq y$ then also $z\in X$. We can then view a pair of sets as as a convex set by viewing the two sets as a lower and an 
upper bound of a convex set. In that case, $\preceq^A_i$ reduces to comparing convex sets in terms of subset relations. 
This representation will play an important role in what we call the \emph{state semantics\/} (see Section~\ref{sec:kripke:kleene:state}). 
\end{remark}

\begin{remark}
A third way (in addition to those in the previous remarks) of viewing pairs of sets (which we conceived of as lower and upper bounds) is as  set of pairs, i.e., 
a set of pairs of lower and upper bounds. This is, of course, done by taking all combinations of lower and upper bounds. Also in that case, the order $\preceq^A_i$ 
makes intuitive sense. In more detail, it boils down to comparing the resulting sets of pairs using $\leq_i$ and the Smyth-ordering.
\end{remark}

\begin{definition}
\label{def:S-i-order}
Given some $\Xb,\Yb\subseteq {\cal L}^2$, $\Xb\preceq^S_i\Yb$ if for every $(y_1,y_2)\in\Yb$, there is some $(x_1,x_2)\in\Xb$ s.t.\ $(x_1,x_2)\leq_i (y_1,y_2)$. 
\end{definition}

Intuitively,  if $(x_2,y_2)$ is more precise than  $(x_1,y_1)$, then each interval in ${\cal O}(x_2,y_2)$ should be at least as precise as at least one interval in ${\cal O}(x_1,y_1)$. 
In other words, on more precise inputs, the produced intervals become more precise than (some) interval produced on the less precise input.
Thus, whereas $\preceq^A_i$ allows for comparison of a set of lower bounds and a set of upper bounds, $\preceq^S_i$ allows for the comparison of two sets of pairs.
The order $\preceq^A_i$ over pairs of sets is equivalent to $\preceq^S_i$ over pairs obtained on the basis of a pair of sets:

\begin{lemma}
\label{prop:smyth:iff:alcantara}
Let some $X_1,X_2,Y_1,Y_2\subseteq {\cal L}$ be given. Then $(X_1, Y_1) \preceq^A_i (X_2, Y_2)$ iff $X_1\times Y_1 \preceq^S_i X_2\times Y_2$.
\end{lemma}

\begin{proof}
$[\Rightarrow]$: Suppose that $X_1\times Y_1 \preceq^A_i X_2\times Y_2$ and consider some $(x_2,y_2)\in X_2\times Y_2$. Since $X_1\times Y_1 \preceq^A_i X_2\times Y_2$, 
there is some $x_1\in X_1$ s.t.\ $x_1\leq x_2$ and there is some $y_1\in Y_1$ s.t.\ $y_2\leq y_1$. Thus, there is
an $(x_1,y_1)\in X_1\times Y_1$ such that $(x_1,y_1)\leq_i (x_2,y_2)$, and so  $X_1\times Y_1 \preceq^S_i X_2\times Y_2$. 

\smallskip\noindent
$[\Leftarrow]$: Suppose that $X_1\times Y_1 \preceq^S_i X_2\times Y_2$ and consider some $y_2\in Y_2$. (The case for $x_1\in X_2$ is analogous). Then for every 
$x\in X_2$, $(x,y_2)\in X_2\times Y_2$ and, since $X_1\times Y_1 \preceq^S_i X_2\times Y_2$, there is some $(x_1,y_1)\in X_1\times Y_1$ s.t.\ $(x_1,y_1)\leq_i (x,y_2)$, 
which implies that $y_2\leq y_1$. Thus, there is some $y_1\in Y_1$ s.t.\ $y_2\leq y_1$, and so $Y_2\preceq^H_L Y_1$. The proof that $X_1\preceq^S_L X_2$ is
similar, and so $X_1\times Y_1 \preceq^A_i X_2\times Y_2$.
\end{proof}

By Lemma~\ref{prop:smyth:iff:alcantara} it immediately follows that an operator ${\cal L}^2\rightarrow \wp({\cal L})\times  \wp({\cal L})$ is 
$\preceq^A_i$-monotonic if and only if the corresponding operator ${\cal L}^2\rightarrow \wp({\cal L}^2)$ (obtained by taking the Cartesian product of the lower and 
upper bound) is $\preceq^S_i$-monotonic.

Next, we show that when an ndao is deterministic, in the sense that ${\cal O}(x,y)$ is a singleton for every $x,y\in {\cal L}$, an ndao reduces to an approximation operator. In other words, our notion of an ndao is a faithful generalization of a deterministic approximation operator.

\begin{proposition}
\label{prop:deterministic:aft}
Let an ndao ${\cal O}:{\cal L}^2\rightarrow \wp({\cal L})\times \wp({\cal L})$ be given s.t.\ ${\cal O}(x,y)$ is a pair of singleton sets for every $x,y\in{\cal L}$. Then ${\cal O}^{\sf AFT}$ 
defined by ${\cal O}^{\sf AFT}(x,y)=(w,z)$ where ${\cal O}(x,y)=(\{w\},\{z\})$, is an approximation operator. 
\end{proposition}

\begin{proof}
We have to show that ${\cal O}$ satisfies $\leq_i$-monotonicity and for every $x\in{\cal L}$, ${\cal O}^{\sf AFT}_l(x,x)={\cal O}^{\sf AFT}_u(x,x)$ (according to Definition \ref{def:approx-notions}).
We first show $\leq_i$-monotonicity. Suppose $(x_1,y_1)\leq_i (x_2,y_2)$. Then by $\preceq^A_i$-monotonicity of ${\cal O}$ and Lemma~\ref{prop:smyth:iff:alcantara}, ${\cal O}(x_1,y_1)\preceq^S_i {\cal O}(x_2,y_2)$. Thus, for every $(w_2,z_2)\in{\cal O}(x_2,y_2)$, there is some $(w_1,z_1)\in{\cal O}(x_1,y_1)$ such that $(w_1,z_1)\leq_i (w_2,z_2)$. Since both ${\cal O}(x_1,y_1)$ and ${\cal O}(x_2,y_2)$ are pairs of singleton sets, we obtain ${\cal O}(x_1,y_1)\leq_i {\cal O}(x_2,y_2)$.
For the second condition, notice that ${\cal O}_u(x,x)={\cal O}_l(x,x)$ in view of ${\cal O}$ being \exactcompliant. This implies that ${\cal O}^{\sf AFT}(x,x)=({\cal O}^{\sf AFT}_u(x,x),{\cal O}^{\sf AFT}_l(x,x))$.
\end{proof}

The following lemma shows that an ndao is composed of a $\preceq^S_L$-monotonic lower-bound operator and a $\preceq^S_L$-anti-monotonic upper-bound operator:

\begin{lemma}
\label{lemma:mon:op:comp:of:mon:and:antimon}
An operator ${\cal O}:{\cal L}^2\rightarrow \wp({\cal L})\times \wp({\cal L})$
is $\preceq^A_i$-monotonic iff for every $x,y \in {\cal L}$, ${\cal O}_l(\cdot,y)$ is $\preceq^S_L$-monotonic, ${\cal O}_l(x,\cdot)$ is $\preceq^S_L$-anti monotonic,  ${\cal O}_u(x,\cdot)$ is $\preceq^H_L$ monotonic and ${\cal O}_u(\cdot,y)$ is $\preceq^H_L$-anti monotonic. 
\end{lemma}

\begin{proof}
$[\Rightarrow]$: We first show the $\preceq^S_L$-anti monotonicity of ${\cal O}_l(x,\cdot)$. 
Consider some $y',y\in {\cal L}$ s.t.\ $y'\leq y$. Then $(x,y)\leq_i (x,y')$ and thus ${\cal O}(x,y)\preceq^A_i {\cal O}(x,y')$, which in particular means that ${\cal O}_l(x,y)\preceq^S_L {\cal O}_l(x,y')$.

We now show the case for ${\cal O}_u(x,.)$. Consider some $x,y',y\in {\cal L}$ s.t.\ $y'\leq y$. Then $(x,y)\leq_i (x,y')$ and thus ${\cal O}(x,y)\preceq^A_i {\cal O}(x,y')$. This means that ${\cal O}_u(x,y)\preceq^H_L {\cal O}_u(x,y')$.

$\preceq^S_L$ monotonicity of ${\cal O}_l(\cdot,y)$ and $\preceq^H_L$-monotonicity of ${\cal O}_u(\cdot,y)$ are shown similarly. 

\smallskip\noindent$[\Leftarrow]$: Suppose that for every $x,y \in {\cal L}$, (1)~${\cal O}_l(\cdot,y)$ is $\preceq^S_L$-monotonic, 
(2)~${\cal O}_l(x,\cdot)$ is $\preceq^S_L$-anti monotonic, (3)~${\cal O}_u(x,\cdot)$ is $\preceq^H_L$ monotonic, and 
(4)~${\cal O}_u(\cdot,y)$ is $\preceq^H_L$-anti monotonic. Consider now some $x_1,x_2,y_1,y_2\in {\cal L}$ s.t.\ $x_1\leq x_2$ and $y_2\leq y_1$, i.e.\ $(x_1,y_1)\leq_i (x_2,y_2)$. With (1), ${\cal O}_l(x_1,y_2)\preceq^S_L {\cal O}_l(x_2,y_2)$. With (2), ${\cal O}_l(x_1,y_1)\preceq^S_L {\cal O}_l(x_1,y_2)$. With transitivity, ${\cal O}_l(x_1,y_1)\preceq^S_L {\cal O}_l(x_2,y_2)$. The case for the upper bound (namely,  ${\cal O}_u(x_2,y_2)\preceq^H_L {\cal O}_u(x_1,y_1)$) is similar.
 \end{proof}
 
\begin{remark}
\label{rem:mon-in-sym-operators}
By Lemma \ref{lemma:mon:op:comp:of:mon:and:antimon}, we see that that for a symmetric operator ${\cal O}$, ${\cal O}_l(\cdot,z)$ is both $\preceq^S_L$-monotonic and $\preceq^H_L$-monotonic, and ${\cal O}_l(z,\cdot)$ is both $\preceq^S_L$-anti monotonic and $\preceq^H_L$-anti monotonic. This follows immediately from the fact that since ${\cal O}$ is symmetric, ${\cal O}_l(x,y)={\cal O}_u(y,x)$ for any $x,y\in {\cal L}$. 
 \end{remark}

The last remark means that for symmetric operators, the $\preceq^A_i$-monotonicity reduces to a far simpler order, obtained by comparing both the lower and upper bounds as follows:  
one set $X$ is smaller than another set $Y$, if for every $x\in X$ there is a $\leq$-larger element in $Y$, and for every 
$y\in Y$ there is  a $\leq$-smaller element in $X$. 

It seems that the larger part of useful non-deterministic approximation operators are symmetric. However, there are also useful non-symmetric ndaos. 
In the remainder of this section, we provide an example of a non-symmetric ndao that approximates $\IC_{\cal P}$. This operator is inspired by the ultimate 
semantics for normal (non-disjunctive) logic programs introduced by Denecker, Marek and Truszczy\'nski~\cite{denecker2002ultimate}. First, we recall the 
ultimate semantics for normal (non-disjunctive) logic programs:

\begin{definition}%
\label{def:ultimate:operator:DMT}
Given a normal logic program ${\cal P}$, we define:\:\footnote{We use the abbreviation {\sf DMT} for Denecker, Marek and Truszczy{\'n}ski to denote this operator, 
as to not overburden the use of ${\cal IC}^{\cal U}_{\cal P}$. Indeed, we will later see that the ultimate operator for non-disjunctive logic programs generalizes to an 
ndao that is different from the ultimate operator ${\cal IC}^{\cal U}_{\cal P}$.}
\[{\cal IC}_{\cal P}^{{\sf DMT},l}(x,y)=\bigcap_{x\subseteq z\subseteq y} \{\alpha\mid \alpha\leftarrow \phi\in {\cal P} \mbox{ and }z(\phi)={\sf T} \},\]
\[{\cal IC}_{\cal P}^{{\sf DMT},u}(x,y)=\bigcup_{x\subseteq z\subseteq y} \{\alpha\mid \alpha\leftarrow \phi\in {\cal P} \mbox{ and }z(\phi)={\sf T}\}.\]
The {\em ultimate approximation operator\/} is then defined in~\cite{denecker2002ultimate} by:
\[{\cal IC}_{\cal P}^{{\sf DMT}}(x,y)=({\cal IC}_{\cal P}^{{\sf DMT},l}(x,y),\:{\cal IC}_{\cal P}^{{\sf DMT},u}(x,y)).\]
\end{definition}

To generalize this operator to an ndao, we proceed as follows: we start by generalizing the idea behind ${\cal IC}_{\cal P}^{{\sf DMT},l}$ to an operator gathering the heads 
of rules that are true in every interpretation $z$ in the interval $[x,y]$:
\[{\cal HD}^{{\sf DMT},l}_{\cal P}(x,y)=\bigcap_{x\subseteq z\subseteq y} \HR_{\cal P}(z).\] 
The immediate consequence operator is then defined as usual, that is: by taking all interpretations that only contain atoms in ${\cal HD}^{{\sf DMT},l}_{\cal P}(x,y)$ and contain at least one member of every head $\Delta\in {\cal HD}^{{\sf DMT},l}_{\cal P}(x,y)$:
\[{\cal IC}_{\cal P}^{{\sf DMT},l}(x,y)=\{z\subseteq \bigcup{\cal HD}^{{\sf DMT},l}_{\cal P}(x,y) \mid \forall \Delta\in{\cal HD}^{{\sf DMT},l}_{\cal P}(x,y)\neq\emptyset: z\cap \Delta\neq \emptyset\}.\] 
Notice that for a non-disjunctive program ${\cal P}$, $ {\cal HD}^{{\sf DMT},l}_{\cal P}(x,y)={\cal IC}_{\cal P}^{{\sf DMT},l}(x,y)={\cal IC}_{\cal P}^{{\sf DMT},l}(x,y)$.
The upper bound operator is constructed entirely analogously, but now the heads of rules with bodies that are true in at least one interpretation in $[x,y]$ are gathered:
\[{\cal HD}^{{\sf DMT},u}_{\cal P}(x,y)=\bigcup_{x\subseteq z\subseteq y} \HR_{\cal P}(z)\}\] 
${\cal IC}_{\cal P}^{{\sf DMT},u}$ is defined in an identical way to ${\cal IC}_{\cal P}^{{\sf DMT},l}$, by just replacing 
${\cal HD}^{{\sf DMT},l}_{\cal P}(x,y)$ by ${\cal HD}^{{\sf DMT},u}_{\cal P}(x,y)$. Finally, the ${\sf DMT}$-ndao is defined as:
\[{\cal IC}_{\cal P}^{{\sf DMT}}(x,y)={\cal IC}_{\cal P}^{{\sf DMT},l}(x,y)\times {\cal IC}_{\cal P}^{{\sf DMT},u}(x,y).\]

We observe that ${\cal IC}_{\cal P}^{{\sf DMT}}$ is an ndao that approximates $IC_{\cal P}$:

\begin{proposition}
For any disjunctive logic program ${\cal P}$, ${\cal IC}_{\cal P}^{{\sf DMT}}$ is an ndao that approximates $IC_{\cal P}$.
\end{proposition}

\begin{proof}
It is clear that ${\cal IC}_{\cal P}^{{\sf DMT}}$ approximates $IC_{\cal P}$, as ${\cal HD}^{{\sf DMT},l}_{\cal P}(x,x)={\cal HD}^{{\sf DMT},u}_{\cal P}(x,x)=IC^{\sf DMT}_{\cal P}(x,x)$ 
for any $x\subseteq {\cal A}_{\cal P}$. We now show it is $\preceq^A_i$-monotonic. Consider some $x_1\subseteq x_2\subseteq y_2\subseteq y_1$. 
We show that ${\cal HD}_{\cal P}^{{\sf DMT},l}(x_1,y_1)\subseteq {\cal HD}_{\cal P}^{{\sf DMT},l}(x_2,y_2)$ and 
${\cal HD}_{\cal P}^{{\sf DMT},u}(x_2,y_2)\subseteq {\cal HD}_{\cal P}^{{\sf DMT},u}(x_1,y_1)$, which immediately implies that
${\cal HD}^{{\sf DMT},l}_{\cal P}(x_1,y_1)\preceq^S_L {\cal HD}^{{\sf DMT},l}_{\cal P}(x_2,y_2)$ and 
${\cal HD}^{{\sf DMT},u}_{\cal P}(x_2,y_2)\preceq^H_L {\cal HD}^{{\sf DMT},u}_{\cal P}(x_1,y_1)$. 
To see that ${\cal HD}_{\cal P}^{{\sf DMT},l}(x_1,y_1)\subseteq {\cal HD}_{\cal P}^{{\sf DMT},l}(x_2,y_2)$, consider some $\Delta\in {\cal HD}_{\cal P}^{{\sf DMT},l}(x_1,y_1)$. 
Then for every $z\in [x_1,y_2]$, there is some $\bigvee\Delta\leftarrow \phi$ s.t.\ $z(\phi)={\sf T}$. Since $[x_2,y_2]\subseteq [x_1,y_1]$, also 
$\Delta\in {\cal HD}_{\cal P}^{{\sf DMT},l}(x_2,y_2)$. The other claim is analogous.
\end{proof}

Notice that the operators ${\cal HD}^{{\sf DMT},l}_{\cal P}(x,y)$ and ${\cal HD}^{{\sf DMT},u}_{\cal P}(x,y)$ are only defined for consistent interpretations $(x,y)$, 
and thus ${\cal IC}_{\cal P}^{\sf DMT}$ it is not a symmetric operator.

\begin{example}
Consider again the program ${\cal P}=\{p\lor q\leftarrow \lnot q\}$ from Example~\ref{example:operator:disj:lp-2-b}. Then: 
\begin{description}
\item [] ${\cal IC}^{{\sf DMT},l}_{\cal P}$ behaves as follows:
\begin{itemize}
\item If $q\in y$ then ${\cal HD}^{{\sf DMT},l}_{\cal P}(x,y)=\emptyset$ and thus ${\cal IC}^{{\sf DMT},l}_{\cal P}(x,y)=\emptyset$.
\item If $q\not\in y$ then ${\cal HD}^{{\sf DMT},l}_{\cal P}(x,y)=\{\{p,q\}\}$ and thus ${\cal IC}^{{\sf DMT},l}_{\cal P}(x,y)=\{\{p\},\{q\},\{p,q\}\}$.
\end{itemize}
\item [] ${\cal IC}^{{\sf DMT},u}_{\cal P}$ behaves as follows:
\begin{itemize}
\item If $q\in x$ then ${\cal HD}^{{\sf DMT},u}_{\cal P}(x,y)=\emptyset$ and thus ${\cal IC}^{{\sf DMT},u}_{\cal P}(x,y)=\emptyset$.
\item If $q\not\in x$ then ${\cal HD}^{{\sf DMT},u}_{\cal P}(x,y)=\{\{p,q\}\}$ and thus ${\cal IC}^{{\sf DMT},u}_{\cal P}(x,y)=\{\{p\},\{q\},\{p,q\}\}$.
\end{itemize}
\end{description}
\end{example}

\section{Theory of Non-Deterministic AFT}
\label{sec:theory:of:ndao}

We now develop a general theory of approximation of non-deterministic operators. First, in Section~\ref{sec:ndao-consistency} we show
that the notion of consistency from deterministic AFT~\cite{denecker2000approximations} 
can be generalized to the non-deterministic setting and holds for any ndao. Then, in Section~\ref{sec:kripke:kleene} we introduce fixpoint semantics for ndaos. 
It is shown that in general such a semantics does not preserve the uniqueness or existence properties from the deterministic setting. 
In Section~\ref{sec:kripke:kleene:state} we consider Kripke-Kleene states that do not have these shortcomings.

\subsection{Consistency of Approximations }
\label{sec:ndao-consistency}

Recall that a \emph{pair} $(x,y)$ is consistent if $x\leq y$, i.e.,\ if it approximates at least one element $z$ ($x\leq z\leq y$). Consistency of a deterministic approximating 
\emph{operator\/} means that a consistent input, i.e.\ an input that approximates at least one element, gives rise to a consistent output, i.e.\ an output that approximates 
at least one element. For deterministic operators, consistency of any approximating operator is guaranteed by Proposition~9 in~\cite{denecker2000approximations}. 

This intuition is generalized straightforwardly to non-deterministic approximation operators by requiring that whenever the operator is applied to a consistent pair, 
there is at least one lower bound which is smaller than at least one upper bound, i.e., there is at least one element approximated by the operator. Formally, this 
comes down to the following definition:

\begin{definition}
Given a lattice $L=\tup{{\cal L},\leq}$ and an ndao ${\cal O}$ on ${\cal L}^2$, we say that
${\cal O}$ is \emph{consistent\/} if for every $x,y\in {\cal L}$ with $x\leq y$, there is some $(w,z)\in{\cal O}(x,y)$ with $w\leq z$.\footnote{To avoid clutter, we abuse 
the notation and write $(w,z)\in {\cal O}(x,y)$ to denote that $w\in {\cal O}_l(x,y)$ and $z\in{\cal O}_u(x,y)$.} 
\end{definition}

We now show the consistency of every ndao. 

\begin{proposition}
\label{prop:weak:consistency}
Any ndao ${\cal O}$ is consistent.
\end{proposition}

\begin{proof}
Consider some $x,y\in {\cal L}$ s.t.\ $x\leq y$. Then clearly, $(x,y)\leq_i (x,x)$ and thus, with $\preceq^A_i$-monotonicity of ${\cal O}$ and Lemma~\ref{prop:smyth:iff:alcantara}, 
${\cal O}(x,y)\preceq^S_i{\cal O}(x,x)$. Since ${\cal O}(x,x)={\cal O}_l(x,x)\times{\cal O}_{u}(x,x)$  (in view of the exactness of ${\cal O}$), for any $w \in {\cal O}_l(x,x)$, 
we have that $(w,w)\in {\cal O}(x,x)$. Thus, since ${\cal O}(x,y)\preceq^S_i{\cal O}(x,x)$, there is some $(z_1,z_2)\in {\cal O}(x,y)$ s.t.\ $(z_1,z_2)\leq_i (w,w)$, i.e.\ $z_1\leq w\leq z_2$.
\end{proof}

For symmetric operators, a stronger notion of consistency holds, namely for every $x\leq y$, it holds that: ${\cal O}_l(x,y)\preceq^S_L {\cal O}_u(x,y)$ 
and ${\cal O}_l(x,y)\preceq^H_L {\cal O}_u(x,y)$. Intuitively, this means that for \emph{every} upper bound we can find a lower bound below the upper bound 
in question, and likewise, for every lower bound we can find an upper bound above the lower bound in question. Thus, in symmetric operators, 
\emph{every\/} lower bound respectively upper bound approximates an element.

\begin{proposition}
Let a symmetric ndao ${\cal O}$ be given. Then for every $x\leq y$, ${\cal O}_l(x,y)\preceq^S_L {\cal O}_u(x,y)$ and ${\cal O}_l(x,y)\preceq^H_L {\cal O}_u(x,y)$.
\end{proposition}

\begin{proof}
Since $x\leq y$, it holds that $(x,y)\leq_i (x,x)$. Thus, ${\cal O}_l(x,y)\preceq^S_L {\cal O}_l(x,x)$ and (since ${\cal O}_u(x,y)={\cal O}_l(y,x)$), 
${\cal O}_u(x,x)\preceq^S_L {\cal O}_u(x,y)$. Since ${\cal O}_l(x,x)={\cal O}_u(x,x)$, we obtain ${\cal O}_l(x,y)\preceq^S_L {\cal O}_u(x,y)$. The other case is similar. 
\end{proof}

Notice that for non-symmetric operators, the above result might not hold, as we only know that ${\cal O}_l(x,y)\preceq^S_L {\cal O}_l(x,x)$ and 
${\cal O}_u(x,x)\preceq^H_L {\cal O}_u(x,y)$. This is demonstrated by the following example:

\begin{example}
Consider an ndao ${\cal O}$ over ${\cal FOUR}$, defined as follows:

\[ \begin{tabular}{lccl}
\ & ${\cal O}_l(x,y)$ & ${\cal O}_u(x,y)$ \\ \hline
when $x=y$ & $\{{\sf T}\}$ & $\{{\sf T}\}$ \\
when $x\neq y$ & $\{{\sf F},{\sf T}\}$ & $\{{\sf T}\}$
 \end{tabular} \]

It can be easily observed that ${\cal O}$ is a non-deterministic approximation operator: Exactness is clear, and, as $\{{\sf F},{\sf T}\}\preceq^S_L \{{\sf T}\}$, 
$\preceq^A_i$-monotonicity is also immediate. However, for the consistent pair $({\sf U},{\sf T})$, ${\cal O}_l({\sf U},{\sf T})=\{{\sf F},{\sf T}\}\not\preceq^H_L 
{\cal O}_u({\sf U},{\sf T})= \{{\sf T}\}$. 
\end{example}

\subsection{Fixpoint Semantics and Kripke-Kleene Interpretations}
\label{sec:kripke:kleene}

As its name suggests, a primary goal of AFT is to provide fixpoints of operators and their approximations.
In the context of deterministic AFT, fixpoints of an approximation operator are approximations for which 
applying the approximation operator to the lower and upper bound $(x,y)$ give rise to exactly the same upper and lower bound. Furthermore, Denecker, Marek and Truszscy\'nski~\cite{DMT07} 
show that a unique $\leq_i$-least consistent fixpoint of an approximating operator ${\cal O}$ exists and can be constructed by iterating ${\cal O}$, starting from $(\bot,\top)$. This fixpoint 
is termed the \emph{Kripke-Kleene\/} fixpoint. In this section, we look at fixpoint semantics for non-deterministic approximation operators, and show that existence and uniqueness of a 
$\leq_i$-least consistent fixpoint of an ndao are not preserved in the non-deterministic setting. Nevertheless, we will show the usefulness of such fixpoints for, e.g., representing the weakly 
supported semantics.

Recall that an ndao generates, on the basis of a lower bound $x$ and an upper bound $y$, a set of lower bounds $\{x_1,x_2,\ldots\}$ and a set of upper bounds $\{y_1,y_2,\ldots\}$. We can then generalize the notion of fixpoints of deterministic approximation operators by stating that $(x,y)$ is a fixpoint of the ndao ${\cal O}$ if the ``input'' lower bound $x$ and the ``input'' upper bound $y$ are among the ``output'' lower bounds respectively ``output'' upper bounds generated on the basis of $(x,y)$. 

The idea underlying the Kripke-Kleene fixpoint as being minimally informative can then be directly taken over from the deterministic setting and imposed by definition.
Accordingly, we can define fixpoints, and Kripke-Kleene interpretations, of an ndao ${\cal O}$ as follows:

\begin{definition}
Given an ndao ${\cal O}$ over $L=\langle {\cal L},\leq \rangle$ and some $x,y\in{\cal L}$:
\begin{itemize}
\item $(x,y)$ is a \emph{fixpoint\/} of ${\cal O}$, if $(x,y)\in {\cal O}_l(x,y)\times {\cal O}_u(x,y)$ (or, somewhat abusing the of notation, if $(x,y)\in {\cal O}(x,y)$),
\item $(x,y)$ is a \emph{Kripke-Kleene interpretation\/} of ${\cal O}$, if it is $\leq_i$-minimal among the fixpoints of ${\cal O}$. 
\end{itemize}
\end{definition}

Thus, Kripke-Kleene interpretations retain the type of Kripke-Kleene semantics in deterministic AFT, namely pairs of single elements $(x,y)$. Intuitively, these interpretations
represent approximations $(x,y)$ of elements such that, when making the `right choices' within the set ${\cal O}(x,y)$, ${\cal O}$ allows to derive exactly the same lower and upper bound.

The next example shows that uniqueness is  no longer guaranteed in the non-deterministic case.

\begin{example}
\label{sec:application:to:dlp:ex}
Consider the dlp ${\cal P}=\{p\lor q\leftarrow\}$ from Example~\ref{examp:IC-deterministic-case}. There are {\em  three\/} $\leq_i$-minimal consistent fixpoints 
of ${\cal IC}_{\cal P}$(Definition~\ref{def:IC_P-ndo}): $(\{p\},\{p\})$, $(\{q\},\{q\})$ and $(\{p,q\},\{p,q\})$.
This is easy to verify as ${\cal IC}_{\cal P}(x,y)=\{\{p\},\{q\},\{p,q\}\}\times \{\{p\},\{q\},\{p,q\}\}$ for any $x,y\subseteq \{p,q\}$.
Also, this program has three weakly supported models: $\{p\}$, $\{q\}$ and $\{p,q\}$. We shall see in Theorem~\ref{theo:correspondence:supported} that this is no coincidence. 
It is thus not surprising that there exists no $\leq_i$-minimal consistent fixpoint, as disjunctive logic programs allow for more than one ($\leq_i$-minimal) weakly supported model.
\end{example}
 
The following example shows that existence of a consistent fixpoint of an ndao is not guaranteed either:

\begin{example} 
Since we are interested in a consistent fixpoint, it suffices to restrict our attention to the consistent pairs. 
Consider an operator ${\cal O}$ over the billatice constructed on the powerset of $\{p,q\}$ and defined as follows:
\begin{eqnarray*}
{\cal O}(\emptyset,\{p,q\})&=&(\{\emptyset\},\{\{p\},\{q\}\})\\ %
{\cal O}(\{p\},\{p\})&=& (\{q\},\{q\})\\
{\cal O}(\{q\},\{q\})&=& (\{p\},\{p\})\\
{\cal O}(x,y) &=&(\{\{p\},\{q\}\}, \{\{p\},\{q\}\}) \quad \forall(x,y)\not\in \{(\{p\},\{p\}),(\{q\},\{q\})\} 
\end{eqnarray*}
It is easily observed that ${\cal O}$ is $\leq_i$-monotonic. Since ${\cal O}_l(x,x)={\cal O}_u(x,x)$ for any $x\subseteq \{p,q\}$, it is also \exactcompliant.
This operator can be visualized as follows:

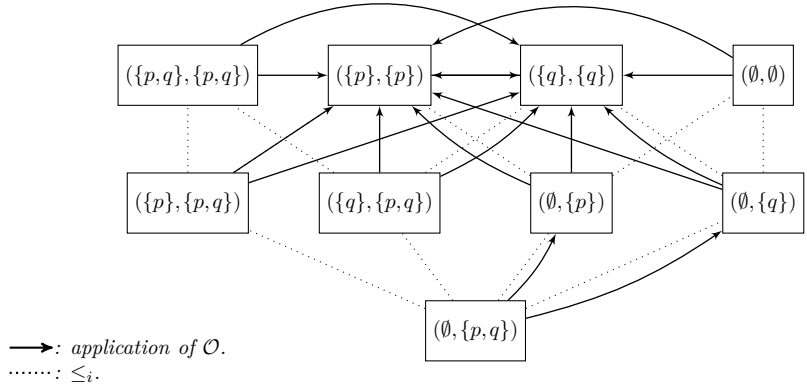
\begin{center}
\begin{tikzpicture}[scale=0.85]
\tikzset{edge/.style = {->,> = latex'}}

\node[draw, scale=0.8] (empty) at (0,0) {$(\emptyset,\{p,q\})$};
\node[draw, scale=0.8] (a) at (-4.5,2) {$(\{p\},\{p,q\})$};
\node[draw, scale=0.8] (b) at (-1.5,2) {$(\{q\},\{p,q\})$};
\node[draw, scale=0.8] (na) at (1.5,2) {$(\emptyset,\{p\})$};
\node[draw, scale=0.8] (nb) at (4.5,2) {$(\emptyset,\{q\})$};
\node[draw, scale=0.8] (ab) at (-4.5,4) {$(\{p,q\},\{p,q\})$};
\node[draw, scale=0.8] (nab) at (-1.5,4) {$(\{p\},\{p\})$};
\node[draw, scale=0.8] (anb) at (1.5,4) {$(\{q\},\{q\})$};
\node[draw, scale=0.8] (nanb) at (4.5,4) {$(\emptyset,\emptyset)$};

\draw[edge, bend right=10, line width=0.5] (empty) to (na);
\draw[edge, bend right=10, line width=0.5] (empty) to (nb);

\draw[edge,  line width=0.5] (na) to (anb);
\draw[edge, line width=0.5, bend left=10] (na) to (nab);
\draw[edge,  line width=0.5, bend left=10] (nb) to (anb);
\draw[edge, line width=0.5] (nb) to (nab);

\draw[edge,  line width=0.5] (a) to (anb);
\draw[edge, line width=0.5] (a) to (nab);
\draw[edge,  line width=0.5, bend right=10] (b) to (anb);
\draw[edge, line width=0.5] (b) to (nab);

\draw[edge, line width=0.5] (nanb) to (anb);
\draw[edge, line width=0.5] (ab) to (nab);
\draw[edge, bend right, line width=0.5] (nanb) to (nab);
\draw[edge, bend left, line width=0.5] (ab) to (anb);

\draw[edge, line width=0.5] (anb) to (nab);
\draw[edge, line width=0.5] (nab) to (anb);

\draw[edge, -, dotted] (empty) to (nb);
\draw[edge,-, dotted] (empty) to (na);
\draw[edge,-, dotted] (empty) to (a);
\draw[edge,-, dotted] (empty) to (b);
\draw[edge,-, dotted] (a) to (ab);
\draw[edge,-, dotted] (a) to (nab);
\draw[edge,-, dotted] (b) to (ab);
\draw[edge,-, dotted] (b) to (anb);
\draw[edge,-, dotted] (na) to (nab);
\draw[edge,-, dotted] (na) to (nanb);
\draw[edge,-, dotted] (nb) to (nanb);
\draw[edge,-, dotted] (nb) to (anb);

\node[scale=0.8] (empty) at (current bounding box.south west) {
\begin{tabular}{l}
$\sampleline{->}$: application of ${\cal O}$.\\
$\sampleline{-,dotted}$: $\leq_i$.
\end{tabular}
};

\node[scale=0.8] (empty) at (current bounding box.south east) {
\phantom{\begin{tabular}{l}
$\sampleline{->}$: application of ${\cal O}$.\\
$\sampleline{-,dotted}$: $\leq_i$.
\end{tabular}}
};
\end{tikzpicture}
\end{center}
It can be verified that this operator is $\preceq^A_i$-monotonic, yet it admits no consistent fixpoint. 
\end{example}
 
Although some properties of fixpoints of approximation operators do not carry over from the deterministic to the non-deterministic setting, 
fixpoints of a non-deterministic operator ${\cal O}$ have been studied in the literature. The following theorem shows that
in the context of disjunctive logic programming, the fixpoints of the operator ${\cal IC}_{\cal P}$ characterize the weakly supported  models of ${\cal P}$. 
It thus provides a first representation of semantics of logic programs that  are not covered by (fixpoints of) deterministic AFT. 
 
\begin{theorem}
\label{theo:correspondence:supported}
Given a dlp ${\cal P}$ and a consistent interpretation $(x,y)\in (\wp({\cal A}_{\cal P}))^2$, it holds that $(x,y)$ is a weakly supported model of ${\cal P}$
iff $(x,y)\in {\cal IC}_{\cal P}(x,y)$.
\end{theorem}

\begin{proof}
\noindent
$[\Rightarrow]$ Suppose that $(x,y)$ is weakly supported. 
We first show that for every $\Delta\in \HRc^l_{\cal P}(x,y)$, $\Delta\cap x\neq \emptyset$. 
Indeed, let $\Delta\in \HRc^l_{\cal P}(x,y)$, i.e.\ $\bigvee\!\Delta\leftarrow \phi \in {\cal P}$ and $(x,y)(\phi)={\sf T}$ (notice that since $(x,y)$ is consistent, $(x,y)(\phi)\neq {\sf C}$). 
Since $(x,y)$ is a model of ${\cal P}$, $(x,y)(\bigvee\!\Delta)\geq_t {\sf T}$, and so $\Delta\cap x\neq\emptyset$.

We now show that $y\cap \HRc^u_{\cal P}(x,y)\neq\emptyset$. Suppose for this that $(y,x)(\phi)\in \{{\sf C},{\sf T}\}$ for some $\bigvee\!\Delta\leftarrow\phi\in{\cal P}$.
We first show the following lemma:

\begin{lemma}
\label{lemma:values:of:formulas:switch}
For any formula $\phi$ and any $x, y\subseteq {\cal A}_{\cal P}$, $(y,x)(\phi)\in \{{\sf T},{\sf C}\}$ implies that $(x,y)(\phi)\in \{{\sf T},{\sf U}\}$ and 
$(y,x)(\phi)\in \{{\sf F},{\sf C}\}$ implies that $(x,y)(\phi)\in \{{\sf F},{\sf U}\}$.
\end{lemma}

\begin{proof}
We show this by induction on the structure  of $\phi$.
For the base case, suppose that $\phi\in{\cal A}_{\cal P}$. If $(y,x)(\phi)\in \{{\sf T},{\sf C}\}$, then $\phi\in y$ and thus $(x,y)(\phi)\in \{{\sf T},{\sf U}\}$. 
Likewise, $(y,x)(\phi)\in \{{\sf F},{\sf C}\}$ means that $\phi\not\in x$ and thus  $(x,y)(\phi)\in \{{\sf F},{\sf U}\}$.
For the inductive case, suppose that the claim holds for $\phi$ and $\psi$. If  $(y,x)(\lnot \phi)\in \{{\sf T},{\sf C}\}$, then 
$(y,x)(\phi)\in \{{\sf F},{\sf C}\}$, and with the inductive hypothesis, $(x,y)(\phi)\in \{{\sf F},{\sf U}\}$, which implies that $(y,x)(\lnot \phi)\in \{{\sf T},{\sf U}\}$. 
The other cases are similar.
\end{proof}

By Lemma~\ref{lemma:values:of:formulas:switch} and since $\Delta\in \HRc^{u}_{\cal P}(x,y)$ implies there is some $\bigvee\!\Delta\leftarrow \phi\in{\cal P}$ with 
$(y,x)(\phi)\in \{{\sf T},{\sf C}\}$, it follows that for every $\Delta\in \HRc^{u}_{\cal P}(x,y)$, $\Delta\cap y\neq \emptyset$. %

It remains to be shown is that $x\subseteq \bigcup\HRc^l_{\cal P}(x,y)$, and  $y\subseteq \bigcup\HRc^l_{\cal P}(y,x)=\bigcup\HRc^u_{\cal P}(x,y)$, 
which is immediate from the fact that since $(x,y)$ is weakly supported, for every atom that is true (respectively undecided) we can find a rule whose body is true 
(respectively undecided) that has this atom in the head.

\medskip\noindent
$[\Leftarrow]$ Suppose that $(x,y)\in {\cal IC}_{\cal P}(x,y)$. We first show that $(x,y)$ is a model of ${\cal P}$. Indeed, suppose that for 
$\bigvee\!\Delta\leftarrow\phi\in{\cal P}$, $(x,y)(\phi)={\sf T}$. Then $\Delta\in \HRc_{\cal P}^l(x,y)$ and thus 
(since $(x,y)\in {\cal IC}_{\cal P}(x,y)$), 
$\Delta\cap x\neq \emptyset$, i.e., $(x,y)(\bigvee\!\Delta)={\sf T}$. The case for $(x,y)(\phi)={\sf U}$ is similar and the case for $(x,y)(\phi)={\sf F}$ is trivial. 
We now show that $(x,y)$ is  weakly supported. 
Indeed, consider some $p\in x$, and suppose towards a contradiction that there is no $\bigvee\!\Delta\leftarrow\phi\in{\cal P}$ s.t.\ 
$(x,y)(\phi)={\sf T}$. But then $p\not\in \bigcup{\cal HD}^{l}_{\cal P}(x,y)$ and thus we have a contradiction to $x\subseteq \bigcup{\cal HD}^{l}_{\cal P}(x,y)$ 
(which we know since $x\in {\cal IC}_{\cal P}^l(x,y)$). The proof for $p\in y$ is similar. 
\end{proof}
 
We now turn to the representation of supported models. As supported models allow for the truth of less atoms than the weakly supported models, one might conjecture 
that supported models are the $\leq_t$-minimal fixpoints of ${\cal IC}_{\cal P}$, but this does not hold, not even for positive or normal programs:

\begin{example}
Take ${\cal P}=\{p\leftarrow r; \ r\leftarrow r\}$. Then $(\{p,r\},\{p,r\})$ and $(\emptyset,\emptyset)$ are both supported models of ${\cal P}$, 
but $(\{p,r\},\{p,r\})$ is not $\leq_t$-minimal.
\end{example}

The supported models can actually be characterized as the fixpoints of ${\cal IC}_{\cal P}^m$ as defined in Remark~\ref{remark:minimality:in:the:op}. 
We refer to~\cite{DBLP:conf/kr/HeyninckA21} for more details on how this can be done.
However, this operator is not $\preceq^A_i$-monotonic (as shown in Remark ~\ref{remark:minimality:in:the:op}).
Next, we show that supported models of a dlp ${\cal P}$ may also be characterized by the fixpoints of ${\cal IC}_{\cal P}$,
but this time together with $\subseteq$-minimization.

\begin{theorem}
Given a dlp ${\cal P}$ and a consistent interpretation $(x,y)\in \wp({\cal A}_{\cal P})^2$. Then $(x,y)$ is a supported model of ${\cal P}$
iff $x\in\min_{\subseteq}( {\cal IC}^l_{\cal P}(x,y))$ and $y\in\min_{\subseteq}( {\cal IC}^u_{\cal P}(x,y))$.
\end{theorem}

\begin{proof}
\noindent
$[\Rightarrow]$ Suppose that $(x,y)$ is a supported model of ${\cal P}$ and consider some  $\Delta\in \HRc^l_{\cal P}(x,y)$. Since $(x,y)$ is in particular weakly supported, by 
Theorem~\ref{theo:correspondence:supported} it follows that $\Delta\cap x\neq \emptyset$. 
We show that $x\in \min_\subseteq(\{ v\mid v\cap \Delta\neq\emptyset \text{ for every }\Delta\in\HRc^l_{\cal P}(x,y)\}$. 
Indeed, suppose towards a contradiction that there is some $x'\in \{ v\mid v\cap \Delta\neq\emptyset \text{ for every }\Delta\in\HRc^l_{\cal P}(x,y)\}$ such that $x'\subsetneq x$. 
Let $\alpha\in x\setminus x'$. Then, since $(x,y)$ is supported, there is a $\bigvee\!\Delta\leftarrow \phi \in {\cal P}$ such that $\Delta\cap x=\{\alpha\}$ and 
$(x,y)(\phi)={\sf T}$. But then $x'\not \in {\cal IC}^l_{\cal P}(x,y)$, in a contradiction to our assumption that $x'\in \min_\subseteq({\cal IC}^l{\cal P}(x,y))$. Analogously, we can show that 
$y\in \min_\subseteq(\{ v\mid v\cap \Delta\neq\emptyset \text{ for every }\Delta\in\HRc^u_{\cal P}(x,y)\}$.

\smallskip\noindent $[\Leftarrow]$ Suppose that $x\in\min_{\subseteq}( {\cal IC}^l_{\cal P}(x,y))$ and $y\in\min_{\subseteq}( {\cal IC}^u_{\cal P}(x,y))$. 
By Theorem~\ref{theo:correspondence:supported} it follows that $(x,y)$ is weakly supported and thus a model of ${\cal P}$. 

We now show that $(x,y)$ is supported. Indeed, let $\alpha\in \Delta \cap x$. Suppose first that there is no $\bigvee\!\Delta\leftarrow\phi\in{\cal P}$ s.t.\ 
$(x,y)(\phi)={\sf T}$. Then $x\not\in min_{\subseteq}(\{v\mid v\cap \Delta\neq \emptyset\text{ for every }\Delta\in\HRc^l_{\cal P}(x,y)\}$, since there is some 
$x' \subseteq x\setminus \{\alpha\}$ such that $x' \in min_{\subseteq} (\{v\mid v\cap \Delta\neq \emptyset\text{ for every }\Delta\in\HRc^l_{\cal P}(x,y)\}$. 
Similarly for $(x,y)(\phi)={\sf U}$. The proof of the second condition in the definition of supported models is similar.
\end{proof}

To summarize the results in this section, we conclude that in contradistinction to deterministic AFT, a $\leq_i$-minimal fixpoint of an ndao is neither guaranteed to be 
unique nor guaranteed to exist. Nevertheless, fixpoint semantics allow for an operator-based characterization of weakly supported and supported models of disjunctive logic programs.
In the next section we introduce the state semantics which is guaranteed to exist and be unique.

\subsection{Kripke-Kleene States}
\label{sec:kripke:kleene:state}

In the previous section, we saw that a (unique) $\leq_i$-minimal fixpoint of an ndao is not guaranteed to exist. This is perhaps not surprising,  as the application of an ndao 
to a pair $(x,y)$ gives rise to a set of lower and upper bounds instead of a single lower and a single upper bound. This means that the method for constructing a 
$\leq_i$-minimal fixpoint of a deterministic approximation fixpoint operator ${\cal O}_{\sf det}$ over a complete lattice by iteratively constructing more and more precise 
approximations by the converging sequence
\[(\bot,\top)\leq_i {\cal O}_{\sf det}(\bot,\top)\leq_i {\cal O}_{\sf det}^2(\bot,\top)\leq_i \ldots\] 
(where $\bot$ and $\top$ respectively represent the minimal and the maximal lattice element) 
is not well-defined, as an ndao ${\cal O}$ cannot be applied to the pair of sets ${\cal O}^i(\bot,\top)$.  However, we can circumvent this deficit by looking at 
operators which allow for pairs of sets in their input. Intuitively, we start with a set of lower bounds $\{x_1,x_2,\ldots\}$ and a set of upper bounds $\{y_1,y_2,\ldots\}$, 
and construct a new, more precise set of lower respectively upper bounds on the basis of them. Thus, instead of approximating a single element $z$ by a 
pair of elements $(x,y)$, we are now approximating a set of elements  $\{z_1,z_2,\ldots\}$ by a pair of sets of elements $\{x_1,x_2,\ldots\}$ and $\{y_1,y_2,\ldots\}$. 
An approximation of such a set of elements can then be seen as a {\em convex set\/} (Remarks~\ref{rem:convex-set}), bounded below by the lower bounds 
$\{x_1,x_2,\ldots\}$ and bounded above by the upper bounds $\{y_1,y_2,\ldots\}$. We first recall some basic notions and notations concerning convex sets.

\subsubsection{Preliminaries on Convex Sets}

Recall from Remarks~\ref{rem:convex-set} that a convex set is a set without ``holes'', that is, a set $X$ such that if two elements $x$ and $y$ are in $X$, then also any 
element between $x$ and $y$ is  in $X$.
Such as set can be viewed as consisting of all elements in between  lower bounds and an upper bounds. A convex set can be obtained by the upwards closure of their 
lower bound and the downwards closure of their upper bound, as defined next:

\begin{definition}
Given a lattice $L=\langle {\cal L},\leq\rangle$ and an element $x\in {\cal L}$, then:
\begin{itemize}
\item the \emph{upwards closure\/} of $x$ is $\upclosure{x}:=\{y\in {\cal L}\mid x\leq y\}$.
\item the \emph{downwards closure\/} of $x$ is $\downclosure{x}:=\{y\in {\cal L}\mid x\geq y\}$.
\end{itemize}
We lift this to sets of elements $X\subseteq {\cal L}$ as follows:
\begin{itemize}
\item the \emph{upwards closure\/} of $X$ is $\upclosure{X}:=\bigcup_{x\in X}\upclosure{x}$,
\item the \emph{downwards closure\/} of $X$ is  $\downclosure{X}:=\bigcup_{x\in X}\downclosure{x}$.
\end{itemize}
A set is \emph{upwards [downwards] closed\/}, if $X=\upclosure{X}[X=\downclosure{X}]$.
We denote the set of upwards closed (respectively downwards closed) subsets of $\mathcal{L}$ by $\wp_\uparrow({\cal L})$ (respectively $\wp_\downarrow({\cal L})$).
\end{definition}

It can be shown that every set of elements is $\preceq^S_L$-equivalent to its upwards closure and $\preceq^H_L$-equivalent to its downwards closure:

\begin{lemma}
\label{lem:S-H-equivalence}
Given a lattice $L=\langle {\cal L},\leq\rangle$ and a set  $X\subseteq {\cal L}$, it holds that:
\begin{enumerate}
\item $\upclosure{X}\preceq^S_L X$ and $ X \preceq^S_L \upclosure{X}$, and
\item  $\downclosure{X}\preceq^H_L X$ and $ X \preceq^H_L  \downclosure{X}$.
\end{enumerate}
\end{lemma}

\begin{proof}
Item~1.\ Since $\upclosure{X}\supseteq  X$, it immediately holds that $ \upclosure{X}\preceq^S_L  X$. We now show $ X \preceq^S_L   \upclosure{X}$. 
Consider some $x\in \upclosure{X}$. This means there is some $y\in  X$ s.t.\ $y\leq x$. 

Item~2.\ Since $\downclosure{X}\supseteq X$, it immediately holds that $ X \preceq^H_L \downclosure{X}$. 
We now show $  \downclosure{X}\preceq^H_L  X$. Consider some $x\in \downclosure{X}$. By the definition of $\downclosure{X}$,
there is some $y\in {\cal X}$ s.t.\ $x\leq y$. 
\end{proof}

The last lemma gives further motivation to using $\preceq^S_L$ for comparing lower bounds and 
$\preceq^H_L$ for comparing upper bounds. Indeed, if $x\in X$ is a lower bound for $z$, then any $x\leq y$ is also a good candidate for being (or better, approximating) $z$,
and likewise, if $x\in X$ is an upper bound for $z$, then any $y\leq x$ is also a good candidate for being (or approximating) $z$.

Lemma~\ref{lem:S-H-equivalence} also means that we can restrict attention to upwards closed sets when using $\preceq^S_L$ and 
to downwards closed sets when using $\preceq^H_L$, without losing any information.

The sets $\wp_\uparrow({\cal L})$ and $\wp_\downarrow({\cal L})$ admit greatest lower bounds under $\preceq^S_L$ and least upper bounds under $\preceq^H_L$
(respectively):

\begin{lemma}
Let $\mathcal{X}\subseteq \wp_\uparrow({\cal L})$ be given. Then $\bigcup{\cal X}$ is the greatest lower bound (under $\preceq^S_L$) of $\mathcal{X}$ 
and $\bigcap{\cal X}$ is the least upper bound (under $\preceq^S_L$) of $\mathcal{X}$.%
\footnote{Recall that we use small letters to denote elements of lattice, capital letters to denote sets of elements, and capital calligraphic letters to denote sets of sets of elements 
(Table~\ref{tab:set-notations}).} 
\end{lemma}

\begin{proof}
We show that $\bigcup{\cal X}$ is the greatest lower bound by showing the following claims:
\begin{itemize}
\item $\bigcup{\cal X}\in \wp_\uparrow({\cal L})$. To see this, consider some $x\in \bigcup{\cal X}$. Then there is some $X\in {\cal X}$ s.t.\ $x\in X$, and thus since 
$X\in\wp_\uparrow({\cal L})$, every $y\in X$ s.t.\ $x \leq y$ is also in $X$, and therefore in $\bigcup{\cal X}$. 
\item $\bigcup{\cal X}\preceq^S_L X$ for any $X\in {\cal X}$. This immediately follows from the fact that for any 
$x\in X$, $x\in \bigcup{\cal X}$.
\item For any $Z\in  \wp_\uparrow({\cal L})$ s.t.\ $Z\preceq^S_L X$ for every $X\in {\cal X}$, $Z\preceq^S_L \bigcup{\cal X}$. For this, consider some $Z\in  \wp_\uparrow({\cal L})$ s.t.\ $Z\preceq^S_L X$ for every $X\in {\cal X}$ and consider some $x\in \bigcup{\cal X}$. Since $x\in X$ for some $X\in {\cal X}$, there is some $z\in Z$ s.t.\ $z\leq x$. Thus, for every $x\in \bigcup{\cal X}$, there is some $z\in Z$ s.t.\ $z\leq x$, which implies that $Z\preceq^S_L \bigcup{\cal X}$.
\end{itemize}

We show that $\bigcap{\cal X}$ is the least upper bound by showing the following claims: 
\begin{itemize}
\item $\bigcap{\cal X}\in \wp_\uparrow({\cal L})$. To see this, consider some $x\in \bigcap{\cal X}$. Then for every $X\in {\cal X}$, it holds that $x\in X$, and thus since 
$X\in\wp_\uparrow({\cal L})$, every $y\in X$ s.t.\ $x\leq y$ is also in $X$, and therefore (since this holds for every $X\in {\cal X}$), $y$ is also in $\bigcap{\cal X}$. 
\item $X \preceq^S_L \bigcap{\cal X}$ for any $X\in {\cal X}$. This follows from the fact that for any $x\in \bigcap{\cal X}$, $x\in X$ for every $X\in{\cal X}$.
\item For any $Z\in  \wp_\uparrow({\cal L})$ s.t.\ $X\preceq^S_L Z$ for every $X\in {\cal X}$, $ \bigcap{\cal X}\preceq^S_L Z$. For this, consider some $Z\in  \wp_\uparrow({\cal L})$ s.t.\ $X\preceq^S_L Z$ for every $X\in {\cal X}$ and consider some $z\in Z$. Since $X\preceq^S_L Z$, there is some $x\in X$ s.t.\ $x\leq z$. Since $X$ is upwards closed, $z\in X$. 
Notice that this holds for any $X\in {\cal X}$, and thus $z\in \bigcap{\cal X}$, which means that $Z\subseteq \bigcap{\cal X}$ and thus $ \bigcap{\cal X}\preceq^S_L Z$. \qedhere
\end{itemize}
\end{proof}

The next lemma is proven like the previous one.

\begin{lemma}
Let some $\mathcal{X}\subseteq \wp_\downarrow({\cal L})$ be given. Then $\bigcap{\cal X}$ is the greatest lower bound (under $\preceq^H_L$) of $\mathcal{X}$ and $\bigcup{\cal X}$ is the least upper bound (under $\preceq^H_L$) of $\mathcal{X}$.
\end{lemma}

By the last two lemmas, one can construct lower respectively upper bounds for sets of pairs of sets under $\preceq^A_i$:

\begin{corollary}
\label{fact:lub:and:glb:under:preceq}
Let some $\mathbfcal{Z} \subseteq \wp_\uparrow({\cal L})\times \wp_\downarrow({\cal L})$ be given. Then 
$(\bigcup \{X\mid (X,Y)\in \mathbfcal{Z}\},\bigcup \{Y\mid (X,Y)\in \mathbfcal{Z}\})$ is the greatest lower bound of (under $\preceq^A_i$) of $\mathbfcal{Z}$ 
and $(\bigcap \{X\mid (X,Y)\in \mathcal{Z}\},\bigcap \{Y\mid (X,Y)\in \mathbfcal{Z}\})$ is the least upper bound of (under $\preceq^A_i$) of $\mathbfcal{Z}$. 
\end{corollary}

\subsubsection{Non-deterministic State Operators and Their Fixpoints}

We can now introduce \emph{non-deterministic state approximation operators\/}\footnote{The use of the word \emph{state\/} comes from disjunctive logic programming, 
where a set of interpretations is often called a state~\cite{minker2002disjunctive,seipel1998alternating}.} that generate a convex set on the basis of a convex set. 
Thus, the type of a non-deterministic state approximation operator ${\cal O}'$ is
$\wp_\uparrow({\cal L})\times \wp_\downarrow({\cal L})\rightarrow  \wp_\uparrow({\cal L})\times \wp_\downarrow({\cal L})$. 
As outlined above, the idea is that a non-deterministic state approximation operator generates an approximation of a set of elements $\{z_1,z_2,\ldots\}$ 
on the basis of an approximation of a set of elements. Just like an ndao, we require information-monotonicity, i.e.\ more precise inputs give rise to more
precise outputs, and exactness, i.e., a convex set consisting of a single element as input gives rise to at least one element being both in the generated lower and upper bound.

\begin{definition}
\label{def:ndsao}
Let a lattice $L=\langle {\cal L},\leq\rangle$ be given. Then an operator ${\cal O}': \wp_\uparrow({\cal L})\times \wp_\downarrow({\cal L})\rightarrow  \wp_\uparrow({\cal L})\times \wp_\downarrow({\cal L})$ is a \emph{non-deterministic state operator\/} (in short, ndso) if it satisfies the following properties:
\begin{itemize}
\item ${\cal O}'$ is  $\preceq^A_i$-monotonic, and
\item ${\cal O}'$ is \emph{exact}, i.e.\ for every $x\in {\cal L}$, there is some $z\in{\cal L}$ such that $z\in {\cal O}'_l(z,z)\cap {\cal O}'_u(z,z)$.\footnote{Alternatively, with a slight abuse of the notations, $(z,z)\in{\cal O}'(\{x\}\times\{x\})$.} 
\end{itemize}
An ndso \emph{approximates\/} a non-deterministic operator $O$ iff for any $x\in {\cal L}$, ${\cal O}'(\{(x,x)\})=\upclosure{O(x)}\times \downclosure{O(x)}$. 
\end{definition}

\begin{remark}
\label{remark:ndso:from:ndao}
A non-deterministic state operator can be straightforwardly derived on the basis of an ndao as follows:
\[
{\cal O}'(\Xb)=
\bigcup_{(x,y)\in\Xb}\upclosure{{\cal O}_l(x,y)}\times 
\bigcup_{(x,y)\in\Xb}\downclosure{{\cal O}_u(x,y)} 
\]
In that case, we say ${\cal O}'$ is \emph{derived from ${\cal O}$}. If ${\cal O}$ approximates $O$, then so will ${\cal O}'$. Furthermore, ${\cal O}'$ is $\preceq^A_i$-monotonic. 
In fact, the  $\preceq^A_i$-monotonicity of ${\cal O}'$ is independent of the  $\preceq^A_i$-monotonicity of ${\cal O}$, as we see in the following proposition:

\begin{proposition}
\label{lemma:O':is:alcantara:monotonic}
For any operator ${\cal O}:{\cal L}^2\rightarrow \wp({\cal L})^2$, it holds that
${\cal O}': \wp_\uparrow({\cal L})\times \wp_\downarrow({\cal L})\rightarrow  \wp_\uparrow({\cal L})\times \wp_\downarrow({\cal L})$ is $\preceq^A_i$-monotonic.
\end{proposition}

\begin{proof}
Consider some $X,X',Y,Y'\subseteq {\cal L}$ s.t.\ $X\times Y\preceq^A_i X'\times Y'$, i.e.\ $X\preceq^S_L X'$ and $Y'\preceq^H_L Y$. We show that ${\cal O}'(X,Y) \preceq^A_i  {\cal O}'(X',Y')$.

Consider first $z\in ({\cal O}'(X',Y'))_1$\footnote{Recall that we use $(X\times Y)_1$ to denote the first component of the pair $X\times Y$, i.e.\ $(X\times Y)_1:=X$. 
Similarly for $(X\times Y)_2$.}. This means that there are some 
$x'\in \upclosure{X'}$ and $y'\in \downclosure{Y'}$ s.t.\ $z\in \upclosure{({\cal O}(x',y'))_1}$. Since $X\preceq^S_L X'$ and $Y'\preceq^H_L Y$, there is some $x\in X$ s.t.\ 
$x\leq x'$ and some $y\in Y$ s.t.\ $y'\leq y$. Thus, $x'\in \upclosure{X}$ and $y'\in \downclosure{Y}$, 
which means that $z\in ({\cal O}'(X,Y))_1$. The case for $({\cal O}'(X',Y'))_2$ is similar. We thus have shown that $({\cal O}'(X',Y'))_1 \subseteq({\cal O}'(X',Y'))_1$ and $({\cal O}'(X',Y'))_2 \subseteq({\cal O}'(X',Y'))_2$, which implies $({\cal O}'(X,Y))_1 \preceq^S_L({\cal O}'(X,Y))_1$ and $({\cal O}'(X',Y'))_2 \preceq^H_L({\cal O}'(X',Y'))_2$, and thus we obtain ${\cal O}'(X,Y) \preceq^A_i  {\cal O}'(X',Y')$.
\end{proof}

Likewise, if one is given a ndso ${\cal O'}$, one can obtain an $\preceq^A_i$-monotonic operator ${\cal O}: {\cal L}^2\rightarrow\wp({\cal L}^2)$  
by simply letting ${\cal O}(x,y)={\cal O}'(\{x\}\times\{y\})$. Such an operator is not guaranteed to be \exactcompliant, though.
\end{remark}

\begin{example}
An ndso that approximates $\IC_{\cal P}$ can be defined as follows:
\[{\cal IC}'_{\cal P}({\bf X})= \bigcup_{(x,y)\in\Xb}\upclosure{{\cal IC}^l_{\cal P}(x,y)} \ \ \times 
\bigcup_{(x,y)\in\Xb}\downclosure{{\cal IC}^u_{\cal P}(x,y)} .\]
\end{example}

We show the following property of non-deterministic state operators: 

\begin{lemma}
\label{prop:decomposition:of:ndao}
An operator ${\cal O}':\wp({\cal L})^2\rightarrow \wp({\cal L})^2$ is $\preceq^A_i$-monotonic iff for any $X\subseteq {\cal L}$, 
${\cal O}'_l(\cdot,X)$ is $\preceq^S_L$-monotonic, ${\cal O}'_l(X,\cdot)$ is $\preceq^S_L$-anti monotonic, ${\cal O}'_u(X,\cdot)$ is $\preceq^H_L$-monotonic, 
and ${\cal O}'_u(\cdot,X)$ is $\preceq^H_L$-anti monotonic. \
\end{lemma}

\begin{proof}
$[\Rightarrow]$: Suppose that ${\cal O}':\wp({\cal L})^2\rightarrow \wp({\cal L})^2$ is $\preceq^A_i$-monotonic and consider some 
$X, Y_1,Y_2\subseteq {\cal L}$ s.t.\ $Y_1\preceq^S_L Y_2$. We first show that ${\cal O}'_l(Y_1,X)\preceq^S_L {\cal O}'(Y_2,X)$. For this, notice that $(Y_1,X)\preceq^A_i (Y_2,X)$ 
(because $Y_1\preceq^S_L Y_2$). Since ${\cal O}'$ is $\preceq^A_i$-monotonic, 
${\cal O}'(Y_1,X)\preceq^A_i{\cal O}'(Y_2,X)$, which on its turn means that ${\cal O}'_l(Y_1,X)\preceq^S_L {\cal O}'_l(Y_2,X)$. 

We now show that  ${\cal O}'_u(X,.)$ is $\preceq^H_L$-monotonic. For this, 
consider some $X, Y_1,Y_2\subseteq {\cal L}$ s.t.\ $Y_2\preceq^H_L Y_1$. Then $(X,Y_1)\preceq^A_i (X,Y_2)$ (since $X\preceq^S_L X$). With the $\preceq^A_i$-monotonicity 
of ${\cal O}'$, ${\cal O}'(X,Y_1)\preceq^A_i {\cal O}'(X,Y_2)$, which implies that ${\cal O}'_u(X,Y_2)\preceq^H_L {\cal O}'_u(X,Y_1)$. Thus, $Y_2\preceq^H_L Y_1$ implies 
${\cal O}'_u(X,Y_2)\preceq^H_L {\cal O}'_u(X,Y_1)$ and ${\cal O}'_u(X,.)$ is $\preceq^H_L$-monotonic (for any $X\subseteq {\cal L}$). The two other cases are similar.

\smallskip\noindent$[\Leftarrow]$:  Suppose that for any $X\subseteq {\cal L}$, ${\cal O}'_l(\cdot,X)$ is $\preceq^S_L$-monotonic,
${\cal O}'_l(X,\cdot)$ is $\preceq^S_L$-anti monotonic, ${\cal O}'_u(X,\cdot)$ is $\preceq^H_L$-monotonic, and ${\cal O}'_u(\cdot,X)$ is $\preceq^H_L$-anti monotonic.
Consider some $X_1,X_2, Y_1,Y_2\subseteq {\cal L}$ s.t.\ $X_1\times Y_1\preceq^A_i X_2\times Y_2$. We first show that ${\cal O}'_l(X_1,Y_1)\preceq^S_L {\cal O}'_l(X_2,Y_2)$. 
For this, notice that $Y_2\preceq^S_L Y_1$ and thus, with the $\preceq^S_L$-anti monotonicity of ${\cal O}'_l(X_1,\cdot)$, 
${\cal O}'_l(X_1,Y_1)\preceq^S_L {\cal O}'_l(X_1,Y_2)$.  Likewise, it can be shown that  $ {\cal O}'_l(X_1,Y_2)\preceq^S_L {\cal O}'_l(X_2,Y_2)$ 
and sot  ${\cal O}'_l(X_1,Y_1)\preceq^S_L{\cal O}'_l(X_2,Y_2)$. The proof for ${\cal O}'_u$ is analogous.
\end{proof}

We define consistency for an ndso analogously as for an ndao, namely:

\begin{definition}
An ndso ${\cal O}'$ is \emph{consistent} if for every $x,y\in {\cal L}$ s.t.\ $x\leq y$, there is a $w\in {\cal O}'_l(\{x\}\times \{y\})$ and $z\in {\cal O}'_u(\{x\}\times \{y\})$, 
(or, slightly abusing notation, $(w,z)\in{\cal O}'(\{x\}\times \{y\})$, such that $w\leq z$. 
\end{definition}

We note that for a consistent ndso ${\cal O}'$, for any $X, Y\subseteq {\cal L}$ for which $\upclosure{X} \cap \downclosure{Y}\neq \emptyset$, it holds that 
$\upclosure{{\cal O}'_l(X\times Y)}\cap \downclosure{{\cal O}'_u(X\times Y)}\neq \emptyset$. Therefore, for any consistent ndso, a non-empty convex set 
in the input (i.e.\ $\upclosure{X} \cap \downclosure{Y}\neq \emptyset$) gives rise to a non-empty convex set in the output 
($\upclosure{{\cal O}'_l(X\times Y)}\cap \downclosure{{\cal O}'_u(X\times Y)}\neq \emptyset$).

The next proposition shows that \exactcompliant state operators are consistent (cf.\ Proposition~\ref{prop:weak:consistency} for ndso's).

\begin{proposition}
\label{theorem:ndso:consistency}
Let ${\cal O}': \wp({\cal L})\times \wp({\cal L})\rightarrow \wp({\cal L})\times \wp({\cal L})$ be an ndso.
Then ${\cal O}'$ is consistent.
\end{proposition}

\begin{proof}
Suppose that ${\cal O}'$ is exact, i.e., for every $x\in {\cal L}$, and for some $z\in{\cal L}$, $z\in{\cal O}_l'(\{x\},\{x\})$ and $z\in{\cal O}_u'(\{x\},\{x\})$. 
This means that there is some $z\in {\cal L}$ s.t.\ $z\in {\cal O}'_l(\{x\},\{x\})\cap {\cal O}'_u(\{x\},\{x\}$.
By Lemma~\ref{prop:decomposition:of:ndao}, ${\cal O}'_l(\{x\},\{y\})\preceq^S_L{\cal O}'_l(\{x\},\{x\})$ and ${\cal O}'_u(\{x\},\{x\})\preceq^H_L{\cal O}'_u(\{x\},\{y\})$. 
Thus, there is some $w\in {\cal O}'_l(\{x\},\{y\})$ s.t.\ $w\leq z$ and there is some $w'\in {\cal O}'_u(\{x\},\{y\})$ s.t.\ $z\leq w'$. With transitivity of $\leq$, $w\leq w'$ and thus 
(slightly abusing notation) there is a consistent pair $(w,w')\in {\cal O}'(x,y)$.
\end{proof}

Another useful property is the fact that the computation of ${\cal O}'(\upclosure{X}\times \downclosure{Y})$ can be simplified by computing ${\cal O}'(X\times Y)$:

\begin{lemma}
Let ${\cal O}':\wp({\cal L})\times \wp({\cal L})\rightarrow\wp({\cal L})\times \wp({\cal L})$ be a ndso.
Then for any $X,Y\subseteq {\cal L}$, ${\cal O}'(\upclosure{X} \times \downclosure{Y}) = {\cal O}'(X\times Y)$.
\end{lemma}

\begin{proof}
The $\supseteq$-direction is clear.
Suppose now that $w\in ({\cal O}'(\upclosure{X}\times \downclosure{Y}))_1$ and $z\in ({\cal O}'(\upclosure{X}\times \downclosure{Y}))_1$. 
Then there are some $x\in  \upclosure{X}$ and $y\in \downclosure{Y}$ s.t.\ $w\in \upclosure{{\cal O}_l(x,y)}$ and $z\in \downclosure{{\cal O}_u(x,y)}$. 
Thus, there are some $x'\in X$ and some $y'\in Y$ s.t.\ $x'\leq x$ and $y\leq y'$. This implies that $(x',y')\leq_i(x,y)$ and with the $\preceq^A_i$-monotonicity of 
${\cal O}$, we have ${\cal O}(x',y')\preceq^A_i{\cal O}(x,y)$. Thus, ${\cal O}_{l}(x',y')\preceq^S_L {\cal O}_{l}(x,y)$, which implies that
${\cal O}'_{l}(x',y')=\upclosure{{\cal O}_{l}(x',y')}\supseteq \upclosure{{\cal O}_{l}(x,y)}$. The case for the upper bound is analogous.
\end{proof}

We now show that an ndso admits a unique $\preceq^A_i$-minimal fixpoint:

\begin{theorem}
\label{theorem:ndso:fixpoint}
Let $L=\tup{{\cal L},\leq}$ be a complete lattice.
Every $\preceq^A_i$-monotonic operator ${\cal O}':\wp_\uparrow({\cal L})\times \wp_\downarrow({\cal L})\rightarrow  \wp_\uparrow({\cal L})\times \wp_\downarrow({\cal L})$ 
admits a unique $\preceq^A_i$-minimal fixpoint that can be constructed by iterative application of ${\cal O}'$ to $(\bot,\top)$.  If ${\cal O}'$ is \exactcompliant, this fixpoint is consistent.
\end{theorem}

\begin{proof}
Let $L=\tup{{\cal L},\leq}$ be a complete lattice and let
 ${\cal O}':\wp_\uparrow({\cal L})\times \wp_\downarrow({\cal L})\rightarrow  \wp_\uparrow({\cal L})\times \wp_\downarrow({\cal L})$  be a $\preceq^A_i$-monotonic operator. Then $\langle\wp_\uparrow({\cal L})\times \wp_\downarrow({\cal L}), \preceq^A_i\rangle$ forms a lattice (in view of 
Corollary~\ref{fact:lub:and:glb:under:preceq} and since $\preceq^A_i$ is a reflexive, transitive and anti-symmetric order over  $\wp_\uparrow({\cal L})\times \wp_\downarrow({\cal L})$). We can apply Knaster and Tarski's fixpoint theorem to show that  the set of fixed-points of ${\cal O}'$ forms a complete lattice, and thus
the $\preceq^A_i$-monotonic operator ${\cal O}'$ admits a unique $\preceq^A_i$-minimal fixpoint. Consistency follows from Proposition~\ref{theorem:ndso:consistency}.
\end{proof}

We call the $\preceq^A_i$-minimal fixpoint of ${\cal O}'$  that is guaranteed by Theorem~\ref{theorem:ndso:fixpoint} the \emph{Kripke-Kleene state of ${\cal O}'$\/},
and denote it by ${\sf KK}({\cal O}')$. Some examples of Kripke-Kleene state for logic programs can be found in Examples~\ref{examp:state-KK-1} and~\ref{examp:state-KK-2} 
below.

For the Kripke-Kleene state of an ndso that is derived from an ndao ${\cal O}$, we can show the following additional results:

\begin{proposition}
\label{prop:KK-states-properties}
 Let $L=\tup{{\cal L},\leq}$ be a complete lattice.
Given an ndso ${\cal O}'$ derived from an ndao ${\cal O}$, we have the following:
\begin{itemize}
\item For any fixpoint $(x,y)$ of ${\cal O}$, ${\sf KK}({\cal O}')\preceq^A_i (x,y)$,
\item ${\sf KK}({\cal O}')$ contains at least one consistent pair.
\item If ${\cal O}'$ approximates a non-deterministic operator $O$,
then for any $x\in {\cal L}$ s.t.\ $x\in O(x)$, it holds that ${\sf KK}({\cal O}')\preceq^A_i (x,x)$. 
\end{itemize}
\end{proposition}

\begin{proof}
We first show two lemmas:

\begin{lemma}
\label{fact:O:as:precise:as:O'}
For any $x,y\in {\cal L}$, ${\cal O}'(x,y)\preceq^A_i {\cal O}(x,y)$.
\end{lemma}

\begin{proof}
Since ${\cal O}'(x,y)=\upclosure{{\cal O}_l(x,y)}\times \downclosure{{\cal O}_u(x,y)}$, for every $w\in {\cal O}_l(x,y)$, there is some $w'\in \upclosure{{\cal O}_l(x,y)}$ s.t.\ 
$w'\leq w$ and for every $z\in {\cal O}_u(x,y)$ there is some $z'\in \downclosure{{\cal O}_u(x,y)}$ s.t.\ $z\leq z'$ (namely, $w$ and $z$ themselves). 
Thus, ${\cal O}'_l(x,y) \preceq^S_L {\cal O}_l(x,y)$ and ${\cal O}_u(x,y) \preceq^H_L {\cal O}'_u(x,y)$, so the lemma is obtained. 
\end{proof}

\begin{lemma}
\label{fact:element:more:precise:than:set}
For any $X,Y\subseteq \wp({\cal L})$, $x\in X$ and $y\in Y$, $X\times Y\preceq^A_i \{x\}\times\{y\}$. 
\end{lemma}

\begin{proof}
Clearly, there are some $x'\in X$ and $y'\in Y$ (namely $x'=x$ and $y'=y$) s.t.\ $x'\leq x$ and $y\leq y'$. 
\end{proof}

The proof of Proposition~\ref{prop:KK-states-properties} now continues as follows:
For the first item, by Proposition~\ref{theorem:ndso:fixpoint}, we have that ${\sf K}({\cal O}')=({\cal O}')^\alpha(\bot,\top)$ for some ordinal $\alpha$. Furthermore, $(\bot,\top)\leq_i (x,y)$. 
By the $\preceq^A_i$-monotonicity of ${\cal O}$, ${\cal O}'(\bot,\top)\preceq^A_i {\cal O}'(\{(x,y)\})$.  By Lemma~\ref{fact:O:as:precise:as:O'}, ${\cal O}'(\{(x,y)\})\preceq^A_i {\cal O}(x,y)$, 
and by Lemma~\ref{fact:element:more:precise:than:set} and since $(x,y)\in{\cal O}(x,y)$, ${\cal O}'(\bot,\top)\preceq^A_i (x,y)$. We can repeat this process until we reach the ordinal 
$\alpha$, and thus $({\cal O}')^\alpha(\bot,\top)={\sf KK}({\cal O}')\preceq^A_i (x,y)$.

The second item is an immediate consequence of Proposition~\ref{theorem:ndso:consistency}.

The proof of the last item is similar to that of the first item.
\end{proof}

Next, we show that for deterministic operators, the Kripke-Kleene state coincides with the Kripke-Kleene fixpoint:

\begin{proposition}
\label{prop:KK-singletons}
Let  ${\cal O}:{\cal L}^2\rightarrow \wp({\cal L})^2$ be an ndao s.t.\ ${\cal O}(x,y)$ is a pair of singleton sets for every $x,y\in{\cal L}$, let be ${\cal O}^{\sf AFT}$ 
defined by ${\cal O}^{\sf AFT}(x,y)=(w,z)$ where ${\cal O}(x,y)=(\{w\},\{z\})$ and let the Kripke-Kleene fixpoint of ${\cal O}^{\sf AFT}$ be given by $(x^{\sf kk},y^{\sf kk})$. 
Then ${\sf KK}({\cal O}')=\upclosure{x^{\sf kk}}\times \downclosure{y^{\sf kk}}$, where ${\cal O}'$ is obtained from ${\cal O}$ as described in Remark~\ref{remark:ndso:from:ndao}. 
\end{proposition}

\begin{proof}
By Proposition~\ref{prop:deterministic:aft}, ${\cal O}^{\sf AFT}$ is an approximation operator, and thus admits a Kripke-Kleene fixpoint. 
For every $x,y\in {\cal L}$, it holds that
\[{\cal O}'(x,y)=\upclosure{\{{\cal O}^{\sf AFT}_l(x,y)\}}\times \downclosure{\{{\cal O}^{\sf AFT}_u(x,y)\}}.\]
Notice furthermore that for any $x,y\in {\cal L}$, if $x'\in \upclosure{x}$ and $y'\in \downclosure{y}$, ${\cal O}(x,y)\preceq_i {\cal O}(x',y')$,
and so ${\cal O}_l(x,y)\preceq^S_L{\cal O}_l(x',y')$, i.e., for every $w'\in {\cal O}_l(x',y')$ there is a $w\in {\cal O}_l(x,y)$ s.t.\ $w\leq w'$. Thus, 
${\cal O}_l(x',y')\subseteq \upclosure{{\cal O}_l(x,y)}$ (and similarly for the upper bound: ${\cal O}_u(x',y') \supseteq {\cal O}_u(x,y)$). 
This means that for any ordinal $\alpha$, 
\[{\cal O}'^\alpha(\bot,\top)=\upclosure{\{({\cal O}^{\sf AFT}_l)^\alpha(\bot,\top)\}}\times \downclosure{\{({\cal O}^{\sf AFT}_u)^\alpha(\bot,\top)\}}.\]
A simple inductive argument then shows the proposition, as ${\sf KK}({\cal O}')={\cal O}'^\alpha(\bot,\top)$ for the smallest ordinal $\alpha$ under which a fixpoint is reached, 
and $(x^{\sf KK},y^{\sf KK})=({\cal O}^{\sf AFT})^\alpha(\bot,\top)$ for the smallest ordinal $\alpha$ under which a fixpoint is reached.
\end{proof}

From Proposition \ref{prop:KK-singletons}, we immediately obtain the following corollary in the context of disjunctive logic programs:

\begin{corollary}
Given a dlp ${\cal P}$, ${\sf KK}({\cal IC}_{\cal P})\preceq_i^A (x,y)$ for every weakly supported model $(x,y)$ of ${\cal P}$.
\end{corollary}

Intuitively, the convex set represented by ${\sf KK}({\cal IC}_{\cal P})$ contains every weakly supported model of ${\cal P}$.

We now show some examples for computing Kripke-Kleene states of dlps:

\begin{example}
\label{examp:state-KK-1}
Let ${\cal P}=\{p\lor q\leftarrow\}$. We calculate ${\sf KK}({\cal IC}_{\cal P})$ as follows:
\begin{itemize}
\item ${\cal IC}'_{\cal P}(\emptyset,\{p,q\})= \upclosure{\{ \{p\},\{q\},\{p,q\}\}} \times \downclosure{\{\{p\},\{q\},\{p,q\}\}}$.\footnote{Even though this is in principle 
redundant, we added $\{p,q\}$ to the first component for clarity.} This can be seen by observing that ${\cal IC}^l_{\cal P}(\emptyset,\{p,q\})=\{ \{p\},\{q\},\{p,q\}\}$ and 
${\cal IC}^l_{\cal P}(\emptyset,\{p,q\})=\{\{p\},\{q\},\{p,q\}\}$.
\item ${\cal IC}'_{\cal P}( \upclosure{\{ \{p\},\{q\},\{p,q\}\}}\times \downclosure{\{\{p\},\{q\},\{p,q\}\}})=\upclosure{\{ \{p\},\{q\},\{p,q\}\}}\times \downclosure{\{\{p\},\{q\},\{p,q\}\}}$ 
(This can be seen by observing that ${\cal IC}^z_{\cal P}(x,y)=  \{\{p\},\{q\},\{p,q\}\}$ for any $z\in \{l,u\}$ and any $x,y\subseteq \{p,q\}$). Thus, a fixpoint is reached. 
\end{itemize}
The Kripke-Kleene state in this case thus corresponds to the convex set $\{\{p\},\{q\},\{p,q\}\}$.
\end{example}

\begin{example}
\label{examp:state-KK-2}
Let ${\cal P}=\{p\lor q\leftarrow; \ \  r\lor s\leftarrow \lnot q\}$. We calculate ${\sf KK}({\cal IC}_{\cal P})$ as follows: 
\begin{itemize}
\item ${\cal IC}'_{\cal P}(\emptyset,\{p,q,r,s\})= \upclosure{\{ \{p\},\{q\}\}}\times\downclosure{\{\{p,r\},\{p,s\},\{q,r\}\{q,r\}\}}$. We obtain this by observing that 
${\cal IC}^l_{\cal P}(\emptyset,\{p,q,r,s\})=\{ \{p\},\{q\}\}$ and ${\cal IC}^u_{\cal P}(\emptyset,\{p,q,r,s\})=\{\{p,r\},\{p,s\},\{q,r\}\{q,r\}\}$.
\item ${\cal IC}'_{\cal P}(\upclosure{\{ \{p\},\{q\}\}}\times \downclosure{\{\{p,r\},\{p,s\},\{q,r\}\{q,r\}\}})=\upclosure{\{ \{p\},\{q\}\}}\times\downclosure{ \{\{p,r\},\{p,s\},\{q,r\}\{q,r\}\}}$ 
and thus a fixpoint is reached. \smallskip \\
(It should be noticed that in the second step of the iteration, two more precise upper bounds are obtained as well: in view of $(\{p\},\{p,r\})(\lnot q)=(\{p\},\{p,s\})(\lnot q)={\sf T}$, 
$\upclosure{{\cal IC}_{\cal P}(\{p\},\{p,r\})}=\upclosure{{\cal IC}_{\cal P}(\{p\},\{p,s\})}=\{\{p,r\},\{p,s\},\{q,r\},\{q,s\}\}\uparrow$, but this lower bound becomes ``nulified'' 
since e.g.\ $\upclosure{{\cal IC}_{\cal P}(\{q\},\{q,r\})}$ contains $\upclosure{\{p\}}$.)
\end{itemize}
The Kripke-Kleene fixpoint that is reached is represented by the convex set \[\{\{p\},\{q\},\{p,r\},\{p,s\},\{q,r\},\{q,s\}\}.\] Intuitively, the meaning of this Kripke-Kleene state is that 
every two-valued (stable) model of this program (if they exist), are in this convex set. This is indeed the case, as the stable models of this program are $\{p,r\}$, $\{p,s\}$ and $\{q\}$. 

This example thus illustrates how the Kripke-Kleene state is an approximation of the semantics of disjunctive logic programs. This has as a benefit it is guaranteed to uniquely exist 
and be constructively computed.
\end{example}

To summarize the results in this section, we have shown the following:
There are two ways to generalize the Kripke-Kleene semantics from deterministic AFT to a non-deterministic setting: one can keep the type the same, resulting in 
\emph{Kripke-Kleene interpretations\/} or simply fixpoints of an ndao. We have shown that such fixpoints of ${\cal IC}_{\cal P}$ correspond to weakly supported models. 
Uniqueness and existence are not guaranteed. This is solved by the alternative generalization of Kripke-Kleene semantics from deterministic AFT, by
\emph{Kripke-Kleene state\/}, which is a set of interpretations instead of a single interpretation. Existence, uniqueness, and consistency for exact ndso's are 
guaranteed for these states, and, moreover, for deterministic operators the two Kripke-Kleene fixpoints coincide.

\section{Stable Semantics}
\label{sec:stable:semantics}

In this section, the stable semantics from deterministic AFT is generalized to the non-deterministic setting. In Section~\ref{subsec:stable:interpretation:sem}, 
we first introduce and study the stable operator and the corresponding stable interpretation semantics. Then, in Section~\ref{subsec:well-foundedstate} we study the 
well-founded state semantics. In Section~\ref{sec:well-founded:alcantara}, we show the usefulness of the well-founded state semantics by relating it to
the well-founded semantics with disjunction~\cite{alcantara2005well} for disjunctive logic programming.

\subsection{Stable Interpretation Semantics}
\label{subsec:stable:interpretation:sem}
 
In deterministic  AFT, the idea behind the stable operator is to find fixpoints that are minimal w.r.t.\ the truth order by constructing a new lower bound and upper bound on the basis 
of the current upper respectively lower bound. In more detail, given the current upper bound $y$, we look for the $\leq$-least fixpoint of ${\cal O}(.,y)$ (which is guaranteed to exist for
a deterministic approximation fixpoint operator ${\cal O}$ and coincides with the greatest lower bound of fixpoints of ${\cal O}(.,y)$). Thus, informally ,we look for the smallest lower bound 
s.t.\ the operator ${\cal O}$ lets us derive nothing more and nothing less when assuming $y$ as an upper bound.
 
Instead of generating a single lower and a single upper bound, an ndao generates a set of lower bounds $\{x_1,x_2,\ldots\}$ and a set of upper bounds $\{y_1,y_2,\ldots\}$ on the basis of 
a lower and upper bound $(x,y)$. We can again look for a (not necessarily unique) smallest lower bound that an ndao allows us to derive in view of a given upper bound $y$ by looking for 
$\leq$-minimal fixpoints of ${\cal O}(.,y)$. We will see below that only under certain assumptions on the ndao such a fixpoint is guaranteed to exist. Furthermore, other properties of the 
stable operator and fixpoints, such as $\preceq^A_i$-montonicity, the existence of stable fixpoints and the $\leq_t$-minimality of stable fixpoints, do not generalize or only generalize under 
certain assumptions from the deterministic setting. Nevertheless, the fixpoint and state semantics based on this construction are useful in general, as they can e.g.\ characterize the stable semantics 
for disjunctive logic programs and the \emph{well-founded semantics with disjunction} by Alc{\^a}ntara, Dam{\'a}sio and Pereira~\cite{alcantara2005well}, and for positive programs it is the set of 
minimal models (Proposition~\ref{prop:positive:programs:state}. 

This section is organized as follows. We first define stable non-deterministic operators. Then, we study their properties, mainly by looking at how properties from the deterministic 
setting can be generalized. We first show that non-deterministic stable operators and their fixpoints faithfully generalize the corresponding deterministic notions 
(see Proposition~\ref{prop:stble:coincide:deterministic} and Corollary~\ref{corol:stle:fixpoints:coincide:deterministic}). Next, we study conditions under which the stable operator is 
well-defined (culminating in Propositions~\ref{proposition:downward:closed:then:fp}). Thereafter, we show that stable operators are not guaranteed to be $\preceq^A_i$-monotonic 
(Example~\ref{ex:stable:not:exist}), and that stable fixpoints are not guaranteed to exist (Example~\ref{ex:stable:not:exist}). Finally, we shows that, under certain 
conditions, stable fixpoints are $\leq_t$-minimal fixpoints of an ndao (Proposition~\ref{prop:stable:is:minimal:fp}). This general study of stable operators and their fixpoints is followed by
an illustration of their usefulness, where it is shown that (partial) stable interpretations of disjunctive logic programs are stable fixpoints of ${\cal IC}_{\cal P}$ 
(Proposition~\ref{prop:stable:fixpoints:represent:stable:models}).
 
\begin{definition}
\label{def:stable:op}
Let ${\cal O}:{\cal L}^2\rightarrow\wp({\cal L})\times \wp({\cal L})$ be an ndao. We define:
\begin{itemize}
\item The \emph{complete lower stable operator\/}: (for any $y\in {\cal L}$)
         \[C({\cal O}_l)(y) \ = \  \{x \in {\cal L}\mid x\in {\cal O}_l(x,y) \mbox{ and }\lnot \exists x'< x: x'\in {\cal O}_l(x',y)\}. \] 
\item The \emph{complete upper stable operator\/}:  (for any $x\in {\cal L}$)
         \[C({\cal O}_u)(x) \ = \  \{y \in {\cal L}\mid y\in {\cal O}_u(x,y) \mbox{ and }\lnot \exists y'<y:  y' \in {\cal O}_u(x,y')\}. \]
\item The \emph{stable operator\/}: $S({\cal O})(x,y)=C({\cal O}_l)(y)\times C({\cal O}_u)(x)$
\item A \emph{stable fixpoint\/} of ${\cal O}$ is any $(x,y)\in{\cal L}^2$ such that $(x,y)\in S({\cal O})(x,y)$.\footnote{Notice that we slightly abuse notation and write $(x,y)\in S({\cal O})(x,y)$ 
to abbreviate $x\in  (S({\cal O})(x,y))_1$ and $y\in ( S({\cal O})(x,y))_2$, i.e.\ $x$ is a lower bound generated by $ S({\cal O})(x,y)$ and $y$ is an upper bound generated by $ S({\cal O})(x,y)$.}
\end{itemize}
\end{definition}

\begin{example}
\label{ex:IC-stable:operator}
Consider the dlp ${\cal P}=\{p\lor q\leftarrow\}$ from Example~\ref{examp:IC-deterministic-case}. It holds that for any $x,y\subseteq \{p,q\}$, ${\cal IC}_{\cal P}^l(x,y)=\{\{p\},\{q\},\{p,q\}\}$ 
and thus $C({\cal IC}^l_{\cal P}(y))=\{\{p\},\{q\}\}$.  It thus follows that $(\{p\},\{p\})$ and $(\{q\},\{q\})$ are the stable fixpoints of ${\cal IC}_{\cal P}$.
\end{example}

The complete operator based on ${\cal IC}_{\cal P}$ produces the minimal models of the reducts w.r.t.\ the input.  Before showing this, we first recall that the \emph{two-valued models\/} 
of a positive program ${\cal P}$ are the sets $x\subseteq {\cal A}_{\cal P}$ s.t.\ for every $\bigvee\Delta\leftarrow \phi \in {\cal P}$, $(x,x)(\phi)={\sf T}$ implies $x\cap \Delta\neq\emptyset$
(Footnote~\ref{footnote:2-val mod}). In wht follows, we denote the set of two-valued models ${\cal P}$ by $\model_2({\cal P})$.

\begin{proposition}
\label{lemma:stable:ic:is:reduct}
Consider a dlp ${\cal P}$ and some $y\subseteq {\cal A}_{\cal P}$. Then $C({\cal IC}_{\cal P})(y)=\min_\subseteq(\model_2(\frac{{\cal P}}{y}))$.
\end{proposition}

\begin{proof}
We first show that $x\in {\cal IC}^l_{\cal P}(x,y)$ implies $x\in \model_2(\frac{{\cal P}}{y})$. Indeed, suppose that $x\in {\cal IC}^l_{\cal P}(x,y)$ and consider some 
$\bigvee\Delta\leftarrow \bigwedge_{i=1}^n\alpha_i \land \bigwedge_{j=1}^m \lnot \beta_j$. If $\beta_j\in y$ for some $j=1,\ldots,m$, then $(x,y)(\lnot \beta_j)\in \{{\sf U},{\sf F}\}$ 
and thus we can ignore this rule.
Suppose then that $\beta_j\not\in y$ for $j=1,\ldots,m$. Suppose now that $(x,y)(\alpha_i)\in \{{\sf T},{\sf C}\}$  for $i=1,\ldots,n$, i.e., 
$\alpha_i\in x$ for $i=1,\ldots,n$. Then $\Delta\in \HRc^l_{\cal P}(x,y)$ and thus, since $x\in  {\cal IC}^l_{\cal P}(x,y)$, $x\cap \Delta\neq\emptyset$.

We now show that $x\in \model_2(\frac{{\cal P}}{y})$ implies that $x\in {\cal IC}^l_{\cal P}(x,y)$. Indeed, consider some rule of the form
$\bigvee\Delta\leftarrow \bigwedge_{i=1}^n\alpha_i \in \frac{{\cal P}}{y}$, i.e.\ there is some 
$\bigvee\Delta\leftarrow \bigwedge_{i=1}^n\alpha_i \land \bigwedge_{j=1}^m \lnot \beta_j\in {\cal P}$ s.t.\ $\beta_j\not\in y$ for $j=1,\ldots,m$. 
If $\alpha_i\not\in x$ for some $i=1,\ldots,n$ then we can safely ignore the rule. Suppose thus that  $\alpha_i\in x$ for some $i=1,\ldots,n$. 
Then $(x,y)(\alpha_i)\in \{{\sf T},{\sf C}\}$ for $i=1,\ldots,n$ and $(x,y)(\beta_j)\in  \{{\sf T},{\sf C}\}$ 
for $j=1,\ldots,m$ and thus, as $x\in \model_2(\frac{{\cal P}}{y})$, $x\cap \Delta \neq\emptyset$. Thus, we have shown that for every $\Delta\in \HRc^l_{\cal P}(x,y)$, 
$\Delta\cap x\neq\emptyset$. Thus, $x\in {\cal IC}^l_{\cal P}(x,y)$.

We now show  that $x\in \lfp({\cal IC}^l_{\cal P}(.,y))$ implies $x\in \min_\subseteq(\model_2(\frac{{\cal P}}{y}))$. Indeed, suppose $x\in \lfp({\cal IC}^l_{\cal P}(.,y))$. 
Since this means that $x\in  {\cal IC}^l_{\cal P}(x,y)$, with the first item, $x\in \model_2(\frac{{\cal P}}{y})$. Suppose towards a contradiction there is some $x'\subset x$ s.t.\ 
$x'\in  \model_2(\frac{{\cal P}}{y})$. Then with the second item, $x'\in {\cal IC}^l_{\cal P}(x',y)$, contradicting   $x\in \lfp({\cal IC}^l_{\cal P}(.,y))$. 

The proof that $x\in \min_\subseteq(\model_2(\frac{{\cal P}}{y}))$ implies  $x\in \lfp({\cal IC}^l_{\cal P}(.,y))$ is analogous.
\end{proof}

\begin{remark}
Notice that we did \emph{not} define the complete stable operator as the glb of fixpoints of ${\cal O}_l(.,x)$. Indeed, this leads to several problems. First, the glb of fixpoints of 
${\cal O}_l(\cdot,x)$ might itself not be a fixpoint of ${\cal O}_l(.,x)$ (since we cannot apply Tarski-Knaster fixpoint theorem to the operator ${\cal O}_l(.,x):{\cal L}\rightarrow \wp({\cal L})$). 
Secondly, and more importantly, taking the glb might lead to a loss of information. Consider e.g.\ the program $\{p\lor q\leftarrow\}$. Then ${\cal IC}_{\cal P}(.,x)$ has three fixpoints 
(for any $x\subseteq \{p,q\}$), namely $\{p\}$, $\{q\}$ and $\{p,q\}$. If we take the glb, however, we obtain $\emptyset$. This would be counter-intuitive, since we clearly should be 
more interested in the more informative $\{p\}$ and $\{q\}$, which represent two possible choices to be made in view of $p\lor q\leftarrow$. 
 \end{remark}

\begin{remark}
\label{remark:complete:might:not:exist}
Notice that since ${\cal O}(.,x)$ maps from ${\cal L}$ to $ \wp({\cal L})$,  i.e., it is not a function or a deterministic operator, 
we cannot use Tarski-Knaster fixpoint theorem to guarantee the existence of fixpoint of ${\cal O}(.,x)$, i.e., $C({\cal O})$ is not guaranteed to be non-empty. 
Indeed, as a case in point, consider the following example:

\begin{example}
\label{ex:infinite:complete}
Consider the lattice $L=\langle\mathbb{Z}\cup\{\infty,-\infty\},\leq\rangle$ where $\leq$ is defined as usual. Consider the operator ${\cal O}_l(x,y)$ defined by ${\cal O}_l(x,y)=\{x\}$ for 
any $x,y\in \mathbb{Z}$, and ${\cal O}_l(x,y)=\mathbb{Z}$ for any $x\in \{\infty,-\infty\}$ and any $y\in \mathbb{Z}$.\footnote{Notice that for this example, the value of ${\cal O}_l(x,y)$ 
for $y\in \{\infty,-\infty\}$ is not significant.} It can be verified that this operator is $\preceq^S_L$-monotonic. Furthermore, for any $y\in\mathbb{Z}$, every $x\in \mathbb{Z}$ is a fixed point 
of ${\cal O}_l(.,y)$ (as $\{x\}={\cal O}_l(x,y)$). Since $\mathbb{Z}$ has no $\leq$-minimal element, $C({\cal O})(y)=\emptyset$ for every $y\in \mathbb{Z}$.
\end{example}
\end{remark}

In what follows (see Definition \ref{def:downwards:closed}), we will delineate a condition that guarantees that the complete lower and upper stable operator is non-empty.

Next, we show that in {\em finite\/} bilattices (in which greatest lower bounds coincide with the minimum), 
the definitions of complete stable operators coincide with those for deterministic AFT:

\begin{proposition}
\label{prop:stble:coincide:deterministic}
Let ${\cal O}:{\cal L}^2\rightarrow \wp({\cal L})\times\wp({\cal L})$ be a ndao over a  finite, complete lattice ${\cal L}$, where ${\cal O}(x,y)$ is a pair of singleton 
sets for every $x,y\in{\cal L}$. Let ${\cal O}^{\sf AFT}$ be the deterministic approximation operator defined by ${\cal O}^{\sf AFT}(x,y)=(w,z)$ where ${\cal O}(x,y)=(\{w\},\{z\})$.\footnote{Recall also 
Proposition~\ref{prop:deterministic:aft}.} Then $C({\cal O}_l)(y)=\{ C({\cal O}^{\sf AFT}_l)(y)\}$, and $C({\cal O}_u)(x)=\{ C({\cal O}^{\sf AFT}_u)(x)\}$ for every $x,y\in{\cal L}$.\footnote
{Notice that since ${\cal O}^{\sf AFT}$ is a deterministic approximation operator, $C({\cal O}^{\sf AFT}_l)(y)$ and $C({\cal O}^{\sf AFT}_u)(y)$ are taken as in Definition~\ref{def:stable-op}.}
\end{proposition}

\begin{proof}
By Proposition  \ref{prop:deterministic:aft} ${\cal O}^{\sf AFT}$ is a deterministic approximation operator.
The equalities in the proposition are immediate since for finite lattices, $\{x\in {\cal L}\mid x\in {\cal O}_l(x,y) \mbox{ and }\lnot \exists x'<y: x'\in {\cal O}_l(x',y)\} =
glb_{\leq}\{ x\in {\cal L}\mid x\in {\cal O}_l(x,y)\}$.
\end{proof}

\begin{remark}
For approximation operators over infinite bilattices, the coincidence in Proposition~\ref {prop:stble:coincide:deterministic} cannot be guaranteed. 
The reason is that the minimum taken in the non-deterministic complete operators $C({\cal O}^u)$ (and $C({\cal O}^l)$) might not coincide with the 
glb taken in the deterministic complete operators $C({\cal O}^{\sf AFT}_l)(y)$ and $C({\cal O}^{\sf AFT}_u)(y)$. A case in point is the following example:

\begin{example}
Consider a lattice $L=\langle \{\bot,\top\}\cup\{x_i\mid i\in\mathbb{N}\},\leq\rangle$ where $\bot< x_i < \top$ for every $i\in\mathbb{N}$. 
Consider the operator ${\cal O}$ defined as follows:
\begin{itemize}
\item  ${\cal O}_l(x_i,y)={\cal O}_u(y,x_i)=x_i$ for any $y\in \{\bot,\top\}\cup\{x_i\mid i\in\mathbb{N}\}$ and $i\in\mathbb{N}$,
\item ${\cal O}_l(\bot,x_i)={\cal O}_u(x_i,\bot) = x_i$ for any $i\in\mathbb{N}$,
\item ${\cal O}_l(x,\top)={\cal O}_u(\top,x)=\top$ for any $x\in \{\top\}\cup\{x_i\mid i\in\mathbb{N}\}$, and
\item ${\cal O}_l(\bot,\top)=\{\bot\}$.
\end{itemize}
Then $C({\cal O}_l)(x_i)=\{x_i\mid i \in\mathbb{N}\}$, whereas $C({\cal O}^{\sf AFT}_l)(x_i)=glb_{\leq}\{x_i\mid i \in\mathbb{N}\}=\bot$ for any $i\in\mathbb{N}$.
\end{example}
\end{remark}

Thus, in view of Proposition~\ref{prop:stble:coincide:deterministic}, for finite lattices the notion of stable fixpoint is not changed in the non-deterministic case:

\begin{corollary}
\label{corol:stle:fixpoints:coincide:deterministic}
Let ${\cal O}:{\cal L}^2\rightarrow \wp({\cal L})^2$ be a ndao over a finite, complete lattice ${\cal L}$, where ${\cal O}(x,y)$ is a pair of singleton 
sets for every $x,y\in{\cal L}$. Let ${\cal O}^{\sf AFT}$ be defined by ${\cal O}^{\sf AFT}(x,y)=(w,z)$ where ${\cal O}(x,y)=(\{w\},\{z\})$. Then $(x,y)$ is a stable fixpoint of ${\cal O}$ 
according to Definition~\ref{def:stable:op} iff $(x,y)$ is a stable fixpoint of ${\cal O}^{\sf AFT}$ according to Definition~\ref{def:stable-op}.
\end{corollary}

The next property we study is the well-definedness of the complete (and thus stable) operator, which is not guaranteed in view of Remark~\ref{remark:complete:might:not:exist}.
Next, we show two useful lemmas (the first one is based on~\cite{pelov2004semantics}), which show that minimal 
pre-fixpoints are also minimal fixpoints and vice versa. For this, we first generalize the notion of a pre-fixpoint from the deterministic setting to non-deterministic operators. 

\begin{definition}
We say that $w\in {\cal L}$ is a {\em pre-fixpoint\/} of the non-deterministic operator $O$ on ${\cal L}$, if $O(w)\preceq^S_L w$. 
\end{definition}

\begin{lemma}
\label{lemma:minimalpre-fix:is:minima:fix}
Let $O:{\cal L}\rightarrow \wp({\cal L})$ be a $\preceq^S_L$-monotonic non-deterministic operator. 
Then if $w$ is a $\leq$-minimal pre-fixpoint of $O$, it is a $\leq$-minimal fixpoint of $O$.\footnote{Notice that  for any ndao ${\cal O}$, ${\cal O}_l(.,x)$ is an operator of the type 
${\cal O}:{\cal L}\rightarrow \wp({\cal L})$ and thus a non-deterministic operator, which we denote by $O$.}
\end{lemma}

\begin{proof}
Suppose that $w$ is a $\leq$-minimal pre-fixpoint of $O$, i.e., $O(w)\preceq^S_L w$. This means that there is some $z\in O(w)$ s.t.\ $z\leq w$. 
Since $O$ is $\preceq^S_L$-monotonic, $O(z)\preceq^S_L O(w)$. Since $z\in O(w)$, $O(z)\preceq^S_L z$, i.e.\ $z$ is a pre-fixpoint of $O$. Since $w$ is a $\leq$-minimal pre-fixpoint 
and $z\leq w$, $z=w$. Minimality is immediate  since fixpoints are in particular pre-fixpoints.
\end{proof}

\begin{lemma}
\label{minimalfp:then:minimal:prefp}
Let $O:{\cal L}\rightarrow \wp({\cal L})$ be a $\preceq^S_L$-monotonic non-deterministic operator. Then if $w$ is a $\leq$-minimal fixpoint of $O$, it is a $\leq$-minimal pre-fixpoint of $O$.
\end{lemma}

\begin{proof}
Suppose that $w$ is a $\leq$-minimal fixpoint of $O$. Then $w\in O(w)$ and so $O(w)\preceq^S_L w$. Thus, $w$ is a pre-fixpoint. Suppose now towards a contradiction 
that for some $w'<w$, $O(w')\preceq^S_L w'$. Without loss of generality, we may assume that $w'$ is a minimal pre-fixpoint. But then, by Lemma~\ref{lemma:minimalpre-fix:is:minima:fix},
it is a minimal fixpoint of $O$, a contradiction.
\end{proof}

Uniqueness of this $\leq$-minimal fixpoint cannot be guaranteed (as can be seen in e.g.\ Example~\ref{sec:application:to:dlp:ex}).  This is a crucial difference with deterministic operators. 
Thus, to summarize, non-determinism forces us to take $\leq$-minimal fixpoints instead of the greatest lower bound. This choice, on its turn, means that existence is not guaranteed on infinite lattices. 
Therefore, we now turn to conditions that ensure the existence of $\leq$-minimal fixpoint of non-deterministic operators over infinite lattices. As we shall see (Proposition~\ref{proposition:downward:closed:then:fp} below), the next property (inspired by~\cite{pelov2004semantics}) assures existence.

\begin{definition}%
\label{def:downwards:closed}
A non-deterministic operator $O:{\cal L}\rightarrow \wp({\cal L})$ is \emph{downward closed} if for every sequence $X=\{x_\epsilon\}_{\epsilon<\alpha}$ of elements in ${\cal L}$ such that:
\begin{enumerate}
\item for every $\epsilon<\alpha$, $O(x_\epsilon)\preceq^S_L \{x_\epsilon\}$, and
\item for every $\epsilon <\epsilon'<\alpha$, $x_{\epsilon'}< x_\epsilon$,
\end{enumerate}
it holds that $O(glb(X))\preceq^S_L glb(X)$.
\end{definition}

Definition~\ref{def:downwards:closed} is a generalization of a similar definition in~\cite{pelov2004semantics}. Its says that an operator is downward closed if the greatest lower bound of every chain 
of pre-fixpoints is itself a pre-fixpoint. As we will see in Proposition~\ref{prop:stable:is:minimal:fp}, this ensures that $O$ admits a fixpoint. 

As an example of an operator that is downward closed, we consider ${\cal IC}^l_{\cal P}(\cdot,y)$
(recall Definition~\ref{def:IC_P-ndo})\footnote{This is inspired by the proof of a similar result in~\cite{pelov2004semantics}.}

\begin{proposition}
\label{proposition:IC^l_P-is-downward:closed}
For any dlp ${\cal P}$ and any $y\subseteq {\cal A}_{\cal P}$, ${\cal IC}_{\cal P}(\cdot,y)$ is downward closed.
\end{proposition}

\begin{proof}
Let $\{x_\epsilon\}_{\epsilon<\alpha}$ be a descending chain of sets of atoms of post-fixpoints of ${\cal IC}^{l}_{\cal P}(\cdot,y)$ for some $y\subseteq {\cal A}_{\cal P}$,
and let $x=\bigcap \{x_\epsilon\}_{\epsilon<\alpha}$.  We show that ${\cal IC}_{\cal P}(x,y)\preceq^S_L \{x\}$. If the chain is finite, this is trivial. 
Suppose therefore that $\{x_\epsilon\}_{\epsilon<\alpha}$  is infinite.

We first show the following lemma:

\begin{lemma}
$(x,y)(\phi)\in \{{\sf T},{\sf C}\}$ implies $(x_\epsilon,y)(\phi)\in  \{{\sf T},{\sf C}\}$ for every $\epsilon<\alpha$, and $(x_\epsilon,y)(\phi)\in  \{{\sf F},{\sf C}\}$ implies  
$(x_\epsilon,y)(\phi)\in  \{{\sf F},{\sf C}\}$  for every $\epsilon<\alpha$.
\end{lemma}

\begin{proof}
By induction on the structure of $\phi$. Base case: suppose that $\phi=p \in {\cal A}_{\cal P}$. If $(x,y)(p)\in \{{\sf T},{\sf C}\}$ then $p \in x$ and so $p \in x_\alpha$ for every $\epsilon <\alpha$, thus
$(x_\epsilon,y)(p)\in \{{\sf T},{\sf C}\}$ as well. If $(x,y)(p)\in \{{\sf F},{\sf C}\}$ then $p \not\in y$, and so $(x_\epsilon,y)(p)\in \{{\sf F},{\sf C}\}$ as well. Inductive case: the cases where 
$\phi=p_1 \land p_2$ and $\phi=p_1\lor p_2$ are straightforward. Suppose now that $\phi=\lnot p$ and $(x,y)(\phi)\in  \{{\sf T},{\sf C}\}$. Thus, $(x,y)(p)\in \{{\sf F},{\sf C}\}$ and by the 
inductive hypothesis, $(x_\epsilon,y)(p)\in \{{\sf F},{\sf C}\}$ for every  $\epsilon<\alpha$, which implies that $(x_\epsilon,y)(\lnot p)\in \{{\sf T},{\sf C}\}$ for every  $\epsilon<\alpha$. The proof of the other case is similar.
\end{proof}

Back to the proof of the proposition. We first show that $x\cap \Delta\neq\emptyset$ for every $\bigvee\Delta\leftarrow\phi\in {\cal P}$ s.t.\ $(x,y)(\phi)\in \{{\sf T},{\sf C}\}$. 
Indeed, consider some $\bigvee\Delta\leftarrow \phi\in {\cal P}$ and $(x,y)(\phi)\in \{{\sf T},{\sf C}\}$. By the lemma above, $(x_\epsilon,y)(\phi)\in \{{\sf T},{\sf C}\}$ for every 
$\epsilon<\alpha$.  Thus, for every $\epsilon<\alpha$, $x_\epsilon \cap\Delta\neq\emptyset$. 
 Since $\{x_\epsilon\}_{\epsilon<\alpha}$ is an infinite descending chain and $\Delta$ is finite, 
there is a $\delta\in \Delta$ s.t.\ $\delta$ is part of an infinite number of sets in $\{x_\epsilon\}_{\epsilon<\alpha}$. 
Since $\{x_\epsilon\}_{\epsilon<\alpha}$  is a $\subseteq$-descending chain, $\delta\in x_\epsilon$ for every $\epsilon<\alpha$,
and thus $\delta\in x$. 

We can now  show that ${\cal IC}_{\cal P}(x,y)\preceq^S_L \{x\}$. Indeed, since $x\cap \Delta\neq\emptyset$ for every $\bigvee\Delta\leftarrow\phi\in {\cal P}$, 
$z=x\cap \bigcup{\cal HD}_{\cal P}(x,y)\in {\cal IC}_{\cal P}(x,y)$ and thus we have found our interpretation $z\in  {\cal IC}_{\cal P}(x,y)$  s.t.\ $z\subseteq x$.
\end{proof}

Next we show that downward closure is indeed a sufficient condition for assuring the existence of a fixpoint:

\begin{proposition}
\label{proposition:downward:closed:then:fp}
Let ${\cal L}=\langle L,\leq\rangle$ be a lattice and let a $O:{\cal L}\rightarrow \wp({\cal L})$ be a downward closed, $\preceq^S_L$-monotonic non-deterministic operator.
Then $O$ admits a $\leq$-minimal fixpoint.
\end{proposition}

\begin{proof}
By Lemma~\ref{lemma:minimalpre-fix:is:minima:fix} it is sufficient to show that $O$ admits a $\leq$-minimal pre-fixpoint. 
The set of pre-fixpoints of $O$ is clearly a partially ordered set. With downwards closedness, every chain of pre-fixpoints has a lower bound (which is also a pre-fixpoint). 
Thus, by Zorn's lemma, the set of pre-fixpoints has a minimum. 
\end{proof}

We obtain the following corollary for the approximation operator ${\cal IC}^l_{\cal P}$:

\begin{corollary}
For every dlp ${\cal P}$, ${\cal IC}^l_{\cal P}(\cdot,y)$ has a $\leq$-minimal fixpoint.
\end{corollary}

\begin{proof}
By Propositions~\ref{proposition:IC^l_P-is-downward:closed} and~\ref{proposition:downward:closed:then:fp}, and since  ${\cal IC}^l_{\cal P}(\cdot,y)$ is $\preceq^S_L$-monotonic.
\end{proof}

{As we have now established conditions under which the stable operator is well-defined, we turn to the property of $\preceq^A_i$-monotonicity of the stable operator.
We notice that in general, $S({\cal O})$ is {\em not\/} a $\preceq^A_i$-monotonic operator: 

\begin{example}
Consider the program ${\cal P}=\{ p\lor q\leftarrow; \quad p\leftarrow \lnot r\}$. We calculate the applications of the stable operators as follows:
\begin{itemize}
\item Since $(\{p\},\{r\})(\lnot r)=(\{q\},\{r\})(\lnot r)={\sf F}$, it holds that $C({\cal IC}_{\cal P}^l)(\{r\})=\{\{p\},\{q\}\}$ for any $x\subseteq {\cal A}_{\cal P}\setminus\{r\}$.
\item Since $(\{r\},x)(\lnot r)={\sf F}$ for any $x\subseteq {\cal A}_{\cal P}\setminus \{r\}$, it holds that ${\cal IC}^u_{\cal P}(x,\{r\})={\cal IC}^l_{\cal P}(\{r\},x)=\{\{p\},\{q\}\}$ 
and thus $C({\cal IC}_{\cal P}^u(\{r\})=\{\{p\},\{q\}\}$.
\item Since $(\emptyset,x)(\lnot r)={\sf T}$ for any for any $x\subseteq {\cal A}_{\cal P}\setminus \{r\}$, it holds that $C({\cal IC}_{\cal P}^u)(\{r\})=\{\{p\}\}$. 
\end{itemize}
Altogether, this means that $S({\cal IC}_{\cal P})(\emptyset,\{r\}))=\{\{p\},\{q\}\}\times \{\{p\}\}$ whereas $S({\cal IC}_{\cal P})(\{r\},\{r\})=\{\{p\},\{q\}\}\times\{\{p\},\{q\}\}$. 
However, $\{\{p\},\{q\}\}\not\preceq^H_L\{\{p\}\}$, or, equivalently, there is no $(x,y)\in S({\cal IC}_{\cal P})(\emptyset,\{r\}))$ s.t.\ 
$(x,y)\leq_i (\{p\},\{q\})\in S({\cal IC}_{\cal P})(\emptyset,\emptyset)$. Thus, although $(\emptyset,\{r\}) \leq_i (\{r\},\{r\})$, it does {\em not\/} hold that 
$S({\cal IC}_{\cal P})(\emptyset,\{r\})) \preceq^A_i S({\cal IC}_{\cal P})(\{r\},\{r\})$.
\end{example}

Recall that (in contrast to the last example), by Proposition~\ref{lemma:O':is:alcantara:monotonic},
the state version of the stable operator {\em is\/} still $\preceq^A_i$-monotonic. Thus, we can still construct a state by iteratively applying the state version of the stable operator. 
We will detail this construction in Section~\ref{subsec:well-foundedstate} and show that this state, which we call the \emph{well-founded state\/}, exists, is unique, is more precise 
than the Kripke-Kleene state, and coincides with the well-founded fixpoint for deterministic operators.

\medskip
A third property we investigate is the existence of stable fixpoints\footnote{Notice that this is not the same as the existence of $\leq$-minimal fixpoints of ${\cal O}^l(.,y)$: this establishes merely that $S({\cal O})$ 
is well-defined, but does not guarantee that stable fixpoints exist as we will see now.}. Even when the complete stable operator is non-empty for every element of a lattice, stable fixpoints may not exist:

\begin{example}
\label{ex:stable:not:exist}
Consider the following dlp: ${\cal P}=\{p\lor q\lor r \leftarrow; \quad p\leftarrow \lnot q; \quad r\leftarrow \lnot p; \quad q\leftarrow \lnot r\}$.
It can be checked that there are no $x,y\subseteq \{p,q,r\}$ s.t.\ $(x,y)\in S({\cal IC}_{\cal P})(x,y)$. To make this clearer, we calculate some of the outcomes of the complete stable operator $C({\cal IC}^l_{\cal P})$:
\begin{itemize}
\item $C({\cal IC}^l_{\cal P})(\emptyset)=C({\cal IC}^u_{\cal P})(\emptyset)=\{\{p,q,r\}\}$. 
\item $C({\cal IC}^l_{\cal P})(\{p\})=C({\cal IC}^u_{\cal P})(\{p\})=\{\{p,q\}\}$.
\item $C({\cal IC}^l_{\cal P})(\{p,q\})=C({\cal IC}^u_{\cal P})(\{p\})=\{\{q\}\}$.
\item $C({\cal IC}^l_{\cal P})(\{p,q,r\})=C({\cal IC}^u_{\cal P})(\{p\})=\{\{p\},\{q\},\{r\}\}$.
\end{itemize}
Notice that $C({\cal IC}^u_{\cal P}(x))=C({\cal IC}^l_{\cal P}(x))$ for any $x\subseteq {\cal A}_{\cal P}$ as ${\cal IC}_{\cal P}$ is symmetric.
Other cases can be easily derived in view of symmetry of ${\cal IC}_{\cal P}^l$ w.r.t.\ $p$, $q$ and $r$. As there is no $x,y\subseteq{\cal A}_{\cal P}$ for which 
$x\in C({\cal IC}^l_{\cal P})(y)$ and $y\in C({\cal IC}^u_{\cal P})(x)=C({\cal IC}^l_{\cal P})(x)$, we conclude that no stable interpretation exists, which is in accordance 
with the stable model semantics for disjunctive logic programming, where a well-founded or more generally three-valued stable model does not exist in this case.
\end{example}

We now move to the last property of stable fixpoints considered here: their $\leq_t$-minimality. To show that stable fixpoints are $\leq_t$-minimal, we will have to assume that the upper bound operator 
${\cal O}_u$ is $\preceq^S_L$-monotonic. As the following example shows, if this condition is not satisfied, $\leq_t$-minimality of stable fixpoints is not guaranteed:

\begin{example}
\label{ex:stable:not:t:minimal}
Consider a lattice ${\cal L}=\langle \{\bot,x,x',y,y',\top\},\leq\rangle$ with $\leq$ as follows:
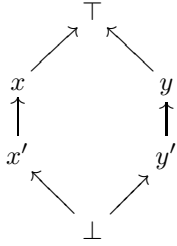
\[\xymatrix@R-10pt@C-10pt{
&\top\\
x\ar[ru]&&y \ar[lu]\\
x'\ar[u] &&y'\ar[u]
\\ 
&\bot\ar[ur]\ar[lu]}
\]
Consider an ndao ${\cal O}$ such that the following hold:
\def\arraystretch{1,3}\tabcolsep=10pt
\begin{center}
\begin{tabular}{l|l}
$(w,z)$&  ${\cal O}(w,z)$  \\ \hline
$(x,y)$ & $\{x\}\times\{y\}$ \\
$(x',y)$ & $\{\bot,x\}\times\{y\}$ \\
$(x,y')$ & $\{x\}\times\{y\}$\\
$(x',y')$ & $\{x'\}\times\{y,y'\}$\\
$(\bot,y)$& $\{x'\}\times \{y,y'\}$\\
$(x,\bot)$& $\{x,x'\}\times \{y'\}$
\end{tabular}
\end{center}
and such that for any $(w,z)$ not occuring in the table, ${\cal O}(w,z)=\{w\}\times \{z\}$.
It can be verified that this operator is $\preceq^A_i$-monotonic.
Also, $(x,y)\in S({\cal O})(x,y)$, since:
\begin{itemize}
\item $x\in {\cal O}_l(x,y)=\{x\}$, (and $x'\not\in {\cal O}_l(x',y)$, i.e., $x$ is a $\leq$-minimal fixpoint of ${\cal O}_l(.,y)$), and
\item $y\in {\cal O}_u(x,y)=\{y\}$ (and $y'\not\in {\cal O}_u(x,y')$, i.e., $y$ is a $\leq$-minimal fixpoint of ${\cal O}_u(.,x)$).
\end{itemize}
Furthermore, $(x',y')\in {\cal O}(x',y')$. Finally, since $(x',y')\leq_t (x,y)$, the stable fixpoint $(x,y)$ is not a $\leq_t$ minimal fixpoint.
\end{example}

The next proposition shows that if stable fixpoints of ${\cal O}$ do exist, they must be fixpoints of ${\cal O}$. Furthermore, 
if ${\cal O}_l(.,x)$ is downward closed and ${\cal O}_u(.,z)$ is $\preceq^S_L$-anti monotonic, they are $\leq_t$-minimal fixpoints of ${\cal O}$. 

\begin{proposition}
\label{prop:stable:is:minimal:fp}
Let $L=\langle {\cal L},\leq\rangle$ be a finite complete lattice  and let ${\cal O}:{\cal L}^2 \rightarrow \wp({\cal L})\times\wp({\cal L})$ be an ndao. Then:
\begin{enumerate}
     \item A stable fixpoint of ${\cal O}$ is also a fixpoint of ${\cal O}$. 
     \item If ${\cal O}_l(.,x)$ is downward closed and ${\cal O}_u(.,x)$ is $\preceq^S_L$-anti monotonic for every $x\in L$, then every stable fixpoint of ${\cal O}$ is a $\leq_t$-minimal fixpoint of ${\cal O}$.
\end{enumerate}
\end{proposition}

\begin{proof}
Part~1 of the proposition immediately follows from the definition of stable fixpoint. For Part~2, suppose that $(x,y)\in S({\cal O})$. Then  $x\in C({\cal O})(y)$ and $y\in C({\cal O})(x)$, 
which implies that $x\in {\cal O}(x,y)$ and $y\in {\cal O}_u(x,y)$. Thus, $(x,y)\in {\cal O}(x,y)$. Consider now some $(x',y')\leq_t (x,y)$ with $(x',y')\in {\cal O}(x',y')$. Since ${\cal O}_u(\cdot,y')$ 
is $\preceq^S_L$-anti monotonic and $x'\leq x$, we have that ${\cal O}_u(x,y')\leq_t {\cal O}_u(x',y')$, which implies (with $y'\in {\cal O}_u(x',y')$) that ${\cal O}_u(x,y')\preceq^S_L y'$, i.e.\
$x$ is a pre-fixpoint of ${\cal O}_u(\cdot,y')$. By Lemma~\ref{minimalfp:then:minimal:prefp}, $y$ is a minimal pre-fixpoint of ${\cal O}_l(x,\cdot)$. Since $y'\leq y$, necessarily $y'=y$. 
In a similar way one shows that $x'=x$ .
\end{proof}

This means in particular that for a symmetric operator ${\cal O}$ on a finite and complete lattice, every stable fixpoint of ${\cal O}$ is a $\leq_t$-minimal fixpoint of ${\cal O}$.

Even though they are not guaranteed to exist, stable fixpoints are useful in knowledge representation.
For example, the stable fixpoints of ${\cal IC}_{\cal P}$ characterize with the (three-valued) stable models of ${\cal P}$:

\begin{theorem}
\label{prop:stable:fixpoints:represent:stable:models}
Consider a normal disjunctive logic program ${\cal P}$ and a consistent interpretation $(x,y)\in \wp({\cal A}_{\cal P}) \times \wp({\cal A}_{\cal P})$. 
Then $(x,y)$ is a stable model of ${\cal P}$ iff $(x,y)\in S({\cal IC}_{\cal P})(x,y)$.
\end{theorem}

\begin{proof}
$[\Rightarrow]$ Suppose that $(x,y)$ is a stable model of ${\cal P}$. We show that $x\in C({\cal IC}_{\cal P}^l)(y))$. 
(The proof that $y\in C({\cal IC}_{\cal P}^l)(x))$ (thus $y \in C(({\cal IC}_{\cal P}^u)(x))$, since ${\cal IC}_{\cal P}$ is symmetric) is analogous). We first show that 
$x\in {\cal IC}^{l}_{\cal P}(x,y)$. Indeed, this immediately follows from the fact that any stable interpretation is weakly supported and that any weakly supported model is a 
fixpoint of ${\cal IC}_{\cal P}$ (Theorem~\ref{theo:correspondence:supported}). It remains to show $\subseteq$-minimality of $x$ among fixpoints of ${\cal IC}^l_{\cal P}(.,y)$ 
and of $y$ among fixpoints of ${\cal IC}^u_{\cal P}(x,.)$. Suppose towards a contradiction that there is some $x'\subset x$ such that\ $x'\in {\cal IC}^{l}_{\cal P}(x',y)$.
We show that $(x',y)\in mod(\frac{\cal P}{(x,y)})$, which contradicts $(x,y)\in \min_{\leq_t}(mod(\frac{\cal P}{(x,y)}))$ (the latter follows from the assumption 
that $(x,y)$ is stable). Indeed,  let $\bigvee\!\Delta \leftarrow \bigwedge\!\Theta \land \bigwedge_{i=1}^n\lnot\beta_i\in\frac{\cal P}{(x,y)}$. We consider three cases:
\begin{itemize}
\item  $(x,y)(\lnot \beta_i)={\sf T}$ for every $1\leq i\leq n$. This means that $\bigvee\!\Delta \leftarrow \bigwedge\!\Theta \land {\sf T} \in \frac{\cal P}{(x,y)}$.
         Notice that for any $\alpha\in {\cal A}_{\cal P}$,  $(x',y)(\alpha)\leq_t (x,y)(\alpha)$ (since $x'\subseteq x$). Now,
         \begin{itemize}
              \item If $(x',y)(\bigwedge \Theta)={\sf F}$, $(x',y)(\bigvee\!\Delta \leftarrow \bigwedge\!\Theta \land {\sf T})$ is trivially satisfied.
              \item If $(x',y)(\bigwedge \Theta)={\sf U}$, $(x',y)(\bigwedge\!\Theta)\leq_t(x,y)(\bigwedge\!\Theta)\leq_t (x,y)(\bigvee\!\Delta)$ implies  
                       $(x,y)(\bigwedge\!\Theta)\in\{{\sf T},{\sf U}\}$ and thus, since $(x,y)$ is a stable model of $\frac{{\cal P}}{(x,y)}$, 
                       $(x,y)(\bigvee\!\Delta)\in\{{\sf T},{\sf U}\}$, i.e.\ $\Delta\cap y\neq \emptyset$. Thus, $(x',y)(\bigvee\!\Delta)\in \{{\sf T},{\sf U}\}$. 
              \item If  $(x',y)(\bigwedge\!\Theta)={\sf T}$, then $\Delta \in \HRc^l_{\cal P}(x',y)$ (since $x'\in {\cal IC}_{\cal P}(x',y)$), and so (since $(x',y)\in  {\cal IC}_{\cal P}^{\sf cons}(x',y)$), 
                       $x'\cap \Delta\neq\emptyset$.
          \end{itemize}             
\item $(x,y)(\lnot \beta_i) \in  \{{\sf T},{\sf U}\}$ for every $1 \leq i \leq n$, and $(x,y)(\lnot\beta_i)={\sf U}$ for some $1 \leq i \leq n$. 
         Then $\bigvee\!\Delta \leftarrow \bigwedge\!\Theta \land {\sf U} \in \frac{\cal P}{(x,y)}$. It can be shown that $(x',y)$ satisfies 
         $\bigvee\!\Delta \leftarrow \bigwedge\!\Theta \land {\sf U}$ just like the previous case.
\item $(x,y)(\lnot \beta_i)={\sf F}$ for some $1 \leq i \leq n$: trivial. 
\end{itemize}
Altogether, we have shown that $(x',y)$ satisfies any $\bigvee\Delta\leftarrow \bigwedge\Theta\land \bigwedge_{i=1}^n\lnot\beta\in\frac{\cal P}{(x,y)}$, contradicting the assumption that 
$(x,y)\in \min_{\leq_t}(mod(\frac{\cal P}{(x,y)}))$. We conclude then that $x$ is a $\subseteq$-minimal fixpoint of ${\cal O}_l(.,y)$.
Analogously, it can be shown that $y$ is $\subseteq$-minimal among fixpoints of ${\cal IC}^u_{\cal P}(x,.)$. \medskip

\noindent$[\Leftarrow]$ Suppose now that $(x,y) \in S({\cal IC}_{\cal P})(x,y)$. We first show that $(x,y)$ is a model of $\frac{\cal P}{(x,y)}$. Indeed, by
Proposition~\ref{prop:stable:is:minimal:fp} $(x,y)\in{\cal IC}_{\cal P}(x,y)$, thus by Theorem~\ref{theo:correspondence:supported} $(x,y)$ is a weakly supported model 
of ${\cal P}$. Since any model $(x,y)$ of ${\cal P}$ is also a model of $\frac{\cal P}{(x,y)}$, we have that  $(x,y)$ is a model of $\frac{\cal P}{(x,y)}$.
For $\leq_t$-minimality, suppose towards a contradiction that there is some $(x',y')<_t(x,y)$ such that $(x',y')\in mod(\frac{\cal P}{(x,y)})$. Since
$(x',y')<_t(x,y)$, either $x'\subsetneq x$ or $y'\subsetneq y$. Suppose first that $x'\subsetneq x$. By Lemma~\ref{prop:decomposition:of:ndao}, 
$ \HRc^l_{\cal P}(x',y) \subseteq \HRc^l_{\cal P}(x,y)$, and for a similar reason $\HRc^l_{\cal P}(x',y)\subseteq \HRc^l_{\cal P}(x',y')$,
thus $\HRc^l_{\cal P}(x',y') \subseteq \HRc^l_{\cal P}(x,y)$. We have: $ {\cal IC}_{\cal P}(x',y')=\min_{\leq_t}( mod(\frac{{\cal P}}{(x',y')})
\subseteq \min_{\leq_t}( mod(\frac{{\cal P}}{(x,y)})) = {\cal IC}_{\cal P}(x,y)$. Hence $x'\in {\cal IC}_{\cal P}(x,y)$, but this contradicts the fact that 
$x$ is a $\leq$-minimal fixpoint of ${\cal IC}_{\cal P}^l(.,y)$ (which follows from the assumption that  $(x,y)\in S({\cal IC}_{\cal P})(x,y)$). The proof of the case where
$y'\subsetneq y$ is similar.
\end{proof}

Notice that this also means that two-valued stable model coincide with total stable fixpoints (i.e.\ $x$ is a two-valued stable model iff 
$(x,x)\in S({\cal IC}_{\cal P})(x,x)$).

\begin{example}
\label{ex:IC-stable:operator:correspondence}
Consider the dlp ${\cal P}=\{p\lor q\leftarrow\}$ from Example~\ref{examp:IC-deterministic-case} (see also Example \ref{ex:IC-stable:operator}). In view of Theorem~\ref{prop:stable:fixpoints:represent:stable:models}, it is not a coincidence that $(\{p\},\{p\})$ and $(\{q\},\{q\})$ are the stable fixpoints of ${\cal IC}_{\cal P}$ and the stable 
interpretations of ${\cal P}$.
\end{example}

\subsection{Well-founded state Semantics}
\label{subsec:well-foundedstate}

The well-founded fixpoint in determinstic AFT is obtained by iteratively applying the stable operator $S({\cal O})$ to the least precise pair $(\bot,\top)$, which results in a fixpoint 
that approximates any fixpoint of the operator $O$ approximated by ${\cal O}$ and guaranteed to exist, be unique, and be more precise than the Kripke-Kleene fixpoint. 
In this section, we generalize this construction to the non-deterministic setting by defining the \emph{well-founded state}, a convex set that is unique, guaranteed to exist, is more precise 
than the Kripke-Kleene state, approximates any fixpoint of the non-deterministic operator $O$ approximated by ${\cal O}$, and obtained by iteratively applying $S({\cal O})$. 
Thus, the well-founded state ${\sf WF}({\cal O})$ is obtained, for an approximation operator ${\cal O}$, by first taking the stable operator $S({\cal O})$ on the basis of ${\cal O}$ and then 
taking the Kripke-Kleene state ${\sf KK}(S({\cal O}))$ of this operator (i.e., iterating the application the state-version of the operator starting from the least precise element until a fixpoint is reached), 
thus resulting in ${\sf WF}({\cal O})={\sf KK}(S({\cal O}))$.

We first show the following useful lemma, which shows that, at least when building up the well-founded state, applications of $S({\cal O})'$ result in at least as precise convex
 sets as applications of ${\cal O}$.
 
 \begin{lemma}
 \label{lemma:stable:is:more:precise:than:kk}
Let a lattice ${\cal L}=\langle L,\leq\rangle$ and an ndao ${\cal O}$ over ${\cal L}$ be given s.t.\  ${\cal O}_l(.,x)$ is downward closed for every $x\in L$. Then for any ordinal 
$\alpha$, $({\cal O}')^\alpha(\bot,\top)\preceq^A_i ((S({\cal O}))')^\alpha(\bot,\top)$.
\end{lemma}

\begin{proof}
We show this by induction on $\alpha$. 

For the base case, we show that ${\cal O}(\bot,\top)\preceq^A_i S({\cal O})(\bot,\top)$. Let $S({\cal O})(\bot,\top)=X\times Y$. 
Consider some $z\in C({\cal O})(\top)$. I.e.\ $z\in {\cal O}_l(z,\top)$ and for every $w\in {\cal O}(w,\top)$, $w\not <z$. Since $(\bot,\top)\leq_i (z,\top)$, with the $\preceq^A_i$-monotonicity 
of ${\cal O}$, ${\cal O}_l(\bot,\top)\preceq^S_L {\cal O}_l(z,\top)$
 Since this argument holds for an arbitrary $z\in C({\cal O})(\top)$, we have established ${\cal O}_l(\bot,\top)\preceq^S_L C({\cal O})(\top)$. The proof for $C({\cal O})(\bot)$ 
(i.e., that $C({\cal O})(\bot)) \preceq^H_L {\cal O}_u(\bot,\top$) is similar. 

For the inductive case, consider two ordinals $\alpha$ and $\beta$,  
let $\beta$ be the successor ordinal of $\alpha$, and assume that $({\cal O}')^\alpha(\bot,\top)\preceq^A_i ((S({\cal O}))')^\alpha(\bot,\top)$.
Let $((S({\cal O}))')^\alpha(\bot,\top)=X\times Y$ and consider some $y\in Y$.  Let $x \in X$. Recall that $x\in C({\cal O})(y)$ means hat $x$ is a 
$\leq$-minimal fixpoint of ${\cal O}_l(.,y)$, and thus $x\in {\cal O}_l(x,y)$. Notice that $(\bot,\top)\leq_i (x,y)$ and so ${\cal O}(\bot,\top)\preceq^A_i {\cal O}(x,y)$, which implies 
that ${\cal O}_l(\bot,\top)\preceq^S_L {\cal O}_l(x,y)$. This means that ${\cal O}_l(\bot,\top)\preceq^S_L \{x\}$ (since $x\in {\cal O}_l(x,y)$). By the inductive hypothesis, 
$\{y\}\preceq^H_L {\cal O}_l^\alpha (\bot,\top)$, which means 
(for any ordinal $\gamma$ smaller than $\alpha$), in view of ${\cal O}'$ is $\preceq^A_i$-monotonic and ${\cal O}^\alpha_u(\bot,\top)\preceq^H_L{\cal O}^\gamma_u(\bot,\top)$, that
$\{y\}\preceq^H_L{\cal O}_l^\gamma (\bot,\top)$. Thus, ${\cal O}(\bot,\top)\preceq^A_i (x,y)$. We can now use the same line of reasoning recursively until we reach the ordinal $\alpha$ 
to obtain $({\cal O}')^\alpha(\bot,\top)\preceq^A_i (x,y)$. Applying ${\cal O}'$ one more time gives us $({\cal O}')^\beta(\bot,\top)\preceq^A_i {\cal O}(x,y)$, which, with $x\in {\cal O}(x,y)$ 
means $({\cal O}_l')^\beta(\bot,\top)\preceq^S_L \{x\}$, as desired.
The proof that $\{y\} \preceq^H_L ({\cal O}_u')^\beta(\bot,\top)$ for every $y \in Y$ is similar. 

For a limit ordinal $\alpha$, we have to show that $lub_{\preceq^A_i}\{ ({\cal O}')^\beta(\bot,\top)\mid \beta<\alpha\}\preceq^A_i lub_{\preceq^A_i}\{ (S({\cal O})')^\beta(\bot,\top)\mid \beta<\alpha\}$ under the assumption that  $({\cal O}')^\beta(\bot,\top)\preceq^A_i  ((S({\cal O}))')^\beta(\bot,\top)$ for any $\beta<\alpha$. This is immediate, in view of the following considerations: 
\begin{enumerate}
     \item By Corollary~\ref{fact:lub:and:glb:under:preceq}, $lub_{\preceq^A_i}\{ ({\cal O}')^\beta(\bot,\top)\mid \beta<\alpha\}=
     (\bigcap_{\beta<\alpha} (({\cal O}')^\beta(\bot,\top))_1, \bigcap_{\beta<\alpha} (({\cal O}')^\beta(\bot,\top))_2)$, 
     and  $lub_{\preceq^A_i}\{ (S({\cal O})')^\beta(\bot,\top)\mid \beta<\alpha\}=
     (\bigcap_{\beta<\alpha} ((S({\cal O}'))^\beta(\bot,\top))_1, \bigcap_{\beta<\alpha} ((S({\cal O}'))^\beta(\bot,\top))_2)$.
     \item By the inductive hypothesis, it holds that 
     $\bigcap_{\beta<\alpha} (({\cal O}')^\beta(\bot,\top))_1\preceq^S_L \bigcap_{\beta<\alpha} ((S({\cal O}'))^\beta(\bot,\top))_1$ 
     and $((S({\cal O}'))^\beta(\bot,\top))_2\preceq^H_L \bigcap_{\beta<\alpha} (({\cal O}')^\beta(\bot,\top))_2$. \qedhere
\end{enumerate}     
\end{proof}

Using the above lemma, we can now show that the well-founded state ${\sf WF}({\cal O})$, which we define as the Kripke-Kleene-state of the stable operator ${\sf KK}(S({\cal O}))$, 
is more precise than the Kripke-Kleene-state of ${\cal O}$ (for downward-closed operators):

\begin{theorem}
\label{proposition:properties:of:well-founded:state}
Let ${\cal L}=\langle L,\leq\rangle$ be a lattice and ${\cal O}$ an ndao  over ${\cal L}$  s.t.\ ${\cal O}_l(.,x)$ is downward closed 
for every $x\in L$. Then ${\sf WF}({\cal O})$ exists, is unique, and has the following properties:
\begin{itemize}
\item ${\sf KK}({\cal O})\preceq^A_i {\sf WF}({\cal O})$,
\item ${\sf WF}({\cal O})$ approximates any stable interpretation $(x,y)\in S({\cal O})(x,y)$, i.e.\ ${\sf WF}({\cal O})\preceq^A_i (x,y)$ for any $(x,y)\in S({\cal O})(x,y)$.
\item If ${\cal O}$ approximates $O$, for any fixpoint $x\in O(x)$, ${\sf WF}({\cal O})\preceq^A_i (x,x)$. 
\end{itemize} 
\end{theorem}

\begin{proof}
The proof of existence and uniqueness is similar to that of Theorem~\ref{theorem:ndso:fixpoint}, since $S({\cal O})'$ is $\preceq^A_i$-monotonic 
(Proposition~\ref{lemma:O':is:alcantara:monotonic}).\footnote{Recall that 
$S({\cal O})'$ is obtained by taking the state-version (Remark~\ref{remark:ndso:from:ndao}) of the stable operator (Definition~\ref{def:stable:op}) based on ${\cal O}$.}   
The first property follows from Lemma~\ref{lemma:stable:is:more:precise:than:kk}, using again Knaster-Tarski fixpoint theorem. 

For the second property, consider some $(x,y)\in S({\cal O})(x,y)$. 
Since $(\bot,\top)\leq_i (x,y)$ and $S({\cal O})'$ is $\preceq^A_i$-monotonic, $S({\cal O})'(\bot,\top)\preceq^A_i S({\cal O})'(x,y)$. Let $S({\cal O})'(x,y)= X\times Y$. 
Since $(x,y)$ is a stable fixpoint, $x\in X$ and $y\in Y$. Thus, $S({\cal O})'(\bot,\top)\preceq^A_i (x,y)$. We can repeat this argument until we reach ${\sf KK}(S({\cal O}))={\sf WF}({\cal O})$. 

The proof of the third property is similar to that of the second property.
\end{proof}

\begin{example}
We illustrate the construction in the proof of Theorem~\ref{proposition:properties:of:well-founded:state} by the program from Example~\ref{ex:stable:not:exist}.
Recall that:
$${\cal P}=\{p\lor q\lor r \leftarrow; \quad p\leftarrow \lnot q; \quad r\leftarrow \lnot p; \quad q\leftarrow \lnot r\}.$$ 
The well-founded state is constructed as follows:
\begin{itemize}
\item $S({\cal IC}_{\cal P})'({\upclosure{\emptyset}},{\downclosure{\{p,q,r\}}})=\{\{p\},\{q\},\{r\}\}\uparrow\times \{\{p,q,r\}\}\downarrow$. 
In more detail, this is obtained as follows (recall that $C({\cal IC}^l_{\cal P})$ is described in Example~\ref{ex:stable:not:exist}):
\begin{align*}
S({\cal IC}_{\cal P})'({\emptyset\uparrow},{\{p,q,r\}\downarrow}) \hspace*{3mm} & =
     \bigcup_{y\in \downclosure{\{p,q,r\}\:}} \hspace*{-3mm} \upclosure{C({\cal IC}^l_{\cal P})(y)} \ \times \ \bigcup_{x \in \upclosure{\emptyset\:\:}} \downclosure{C({\cal IC}^l_{\cal P})(y)} \\
     \ & = \bigcup_{y\subseteq\{p,q,r\}} \hspace*{-2mm} \upclosure{C({\cal IC}^l_{\cal P})(y)}  \ \times \ \hspace*{-3mm} \bigcup_{x\subseteq \{p,q,r\}} \hspace*{-2mm} \downclosure{C({\cal IC}^l_{\cal P})(y)}
\end{align*}
\item  $(S({\cal IC}_{\cal P}))'(S({\cal IC}_{\cal P})({\upclosure{\emptyset}},{\downclosure{\{p,q,r\}}}))=\upclosure{\{\{p\},\{q\},\{r\}\}}\times \downclosure{\{\{p,q\},\{q,r\},\{p,r\}\}}$.
\item $(S({\cal IC}_{\cal P}))'^2(S({\cal IC}_{\cal P})({\upclosure{\emptyset}},\downclosure{\{p,q,r\}}))=(S({\cal IC}_{\cal P}))'(S({\cal IC}_{\cal P})({\emptyset\uparrow},{\{p,q,r\}\downarrow}))$ 
and thus a fixed point is reached.
\end{itemize}
We thus see that 
\[{\sf WF}({\cal IC}_{\cal P})={\sf KK}(S({\cal IC}_{\cal P}))=\upclosure{\{\{p\},\{q\},\{r\}\}}\times \downclosure{\{\{p,q\},\{q,r\},\{p,r\}\}}\]
This is represented by the convex set: 
\[ \{\{p\},\{q\},\{r\},\{p,q\},\{q,r\},\{p,r\}\}\]
Intuitively, the well-founded state expresses that at least one among $p$, $q$ or $r$ is true (i.e.\ $p\lor q\lor r$ is true), and at least one among $p$, $q$ or $r$ is false (i.e.\ $\lnot p\lor \lnot q\lor \lnot r$ is true). 

It is interesting to note that this is exactly the same outcome as the well-founded semantics with disjunction (see~\cite[Example 6]{alcantara2005well}). We will see in 
Section~\ref{sec:well-founded:alcantara} that this close resemblance is not a coincidence. For this program, the Kripke-Kleene state coincides with the well-founded state. 
It shows that the state semantics give meaning to programs which do not have (partial) 
stable interpretations. Thus, this example illustrates the existence and uniqueness-properties of the well-founded state (and the Kripke-Kleene state).
\end{example}

\begin{example}
Consider the dlp ${\cal P}=\{p\lor q\leftarrow\lnot s; \quad s\leftarrow r; \quad r\leftarrow s\}$. We calculate ${\sf WF}({\cal IC}_{\cal P})$ as follows (using Theorem~\ref{prop:stable:fixpoints:represent:stable:models}):

\begin{eqnarray*}
 S({\cal IC}_{\cal P})'(\emptyset,\{p,q,s\})&=& \upclosure{\min_\subseteq {\sf Mod}(\frac{{\cal P}}{{\cal A}_{\cal P}})}\times \downclosure{
\min_\subseteq {\sf Mod}(\frac{{\cal P}}{\emptyset})} =\upclosure{\{\emptyset\}}\times \downclosure{\{\{p\},\{q\}\}}\\
S({\cal IC}_{\cal P})'^2(\emptyset,\{p,q\})&=&
\upclosure{\left(\min_\subseteq {\sf Mod}(\frac{{\cal P}}{\{p\}})
\cup
\min_\subseteq {\sf Mod}(\frac{{\cal P}}{\{q\}})\right)}
\times \downclosure{
\min_\subseteq {\sf Mod}(\frac{{\cal P}}{\emptyset})} \\
&=&
\upclosure{\{\{p\},\{q\}\}}\times \downclosure{\{\{p\},\{q\}\}}\\
S({\cal IC}_{\cal P})'^3(\emptyset,\{p,q\})&=& \upclosure{\left(\min_\subseteq {\sf Mod}(\frac{{\cal P}}{\{p\}})
\cup
\min_\subseteq {\sf Mod}(\frac{{\cal P}}{\{q\}})\right)}
\times \downclosure{\left(\min_\subseteq {\sf Mod}(\frac{{\cal P}}{\{p\}})
\cup
\min_\subseteq {\sf Mod}(\frac{{\cal P}}{\{q\}})\right)}\\
&=&
\upclosure{\{\{p\},\{q\}\}}\times \downclosure{\{\{p\},\{q\}\}}\\
\end{eqnarray*}

and thus a fixpoint is reached after two iterations. The well-founded state is thus represented by the convex set $\{\{p\},\{q\}\}$. 
This can be compared to the Kripke-Kleene state ${\sf KK}({\cal IC}_{\cal P})$:
\begin{itemize}
\item $({\cal IC}_{\cal P}(\emptyset,\{p,q,s\}))'=\upclosure{\emptyset}\times\downclosure{\{\{p,s,r\},\{q,s,r\}\}}$.
\item $({\cal IC}_{\cal P}(\emptyset,\{p,q,s\}))'^2=\upclosure{\emptyset}\times\downclosure{\{\{p,s,r\},\{q,s,r\}\}}$ and thus a fixpoint is reached. 
\end{itemize}
The Kripke-Kleene state is thus represented by the convex set $\wp({{\cal A}_{\cal P}})$. This means that in this case, the well-found state is significantly more precise than the Kripke-Kleene state. 

This example also illustrates that the well-founded state approximates the stable interpretations of ${\cal P}$. Indeed, the stable interpretations are $\{p\}$ and $\{q\}$, and it holds that 
$\upclosure{\{\{p\},\{q\}\}}\times \downclosure{\{\{p\},\{q\}\}}\preceq^A_i (\{p\},\{p\})$ and  $\upclosure{\{\{p\},\{q\}\}}\times \downclosure{\{\{p\},\{q\}\}}\preceq^A_i (\{q\},\{q\})$.
\end{example}

We conclude this section by showing that the well-founded state coincides with the well-founded fixpoint for deterministic approximation operators for 
approximation operators over finite lattices, thus showing that the well-founded state is a faithful generalization of the deterministic well-founded fixpoint:

\begin{theorem}
Consider an ndao ${\cal O}:{\cal L}^2\rightarrow \wp({\cal L})^2$ over a finite lattice ${\cal L}$ s.t.\ ${\cal O}(x,y)$ is a pair of singleton sets for every $x,y\in{\cal L}$.
Let ${\cal O}^{\sf AFT}$ be defined by ${\cal O}^{\sf AFT}(x,y)=(w,z)$ where ${\cal O}(x,y)=(\{w\},\{z\})$, and let $(x^{\sf WF},y^{\sf WF})$ be the well-founded fixpoint of 
${\cal O}^{\sf AFT}$.\footnote{Recall Definition~\ref{def:stable-op} for the definition of the well-founded fixpoint of a deterministic approximation operator.} Then 
${\sf WF}({\cal O})= \upclosure{x^{\sf WF}}\times \downclosure{y^{\sf WF}}$.
\end{theorem} 

\begin{proof}
Since the well-founded fixpoint of ${\cal O}^{\sf AFT}$ is a fixpoint of $S({\cal O})$, by Proposition \ref{prop:stble:coincide:deterministic} $(x^{\sf WF},y^{\sf WF})$ is 
a stable fixpoint of $S({\cal O})$. By the second item of Theorem~\ref{proposition:properties:of:well-founded:state}, ${\sf WF}({\cal O})\preceq^A_i (x^{\sf WF},y^{\sf WF})$. 
We now show that $\upclosure{x^{\sf WF}}\times\downclosure{y^{\sf WF}}\preceq^A_i {\sf WF}({\cal O})$ by induction on the number of iterations for reaching a fixpoint. 
For the base case notice that, again by Proposition~\ref{prop:stble:coincide:deterministic}, 
$\upclosure{\{S({\cal O}_l^{\sf AFT})(.,\top)\}}\times \downclosure{ \{S({\cal O}_u^{\sf AFT})(\bot,.)\}}= S({\cal O})'(\bot,\top)$.
For the inductive case, suppose that for some $i\in\mathbb{N}$, $(x_i,y_i)$ corresponds to the result of applying $i$ times $S({\cal O})$ to $(\bot,\top)$ and suppose that 
$\upclosure{\{x_i\}}\times \downclosure{\{y_i\}}\preceq^A_i  S({\cal O})'^i(\bot,\top)$. This means that there are some $x',y'\in {\cal L}$ where $x'$ occurs in the first component 
of $S({\cal O})'^i(\bot,\top)$ and $y'$ occurs in the second component of $S({\cal O})'^i(\bot,\top)$ s.t.\ $x_i \leq x'$ and $y'\leq y_i$, i.e.\ $(x_i,y_i)\leq_i (x',y')$. 
Since $S({\cal O}^{\sf AFT})$ is $\leq_i$-monotonic (see~\cite[Proposition 20]{denecker2000approximations}), $S({\cal O}^{\sf AFT})(x_i,y_i)\leq_i S({\cal O}^{\sf AFT})(x',y')$. 
By definition, 
\begin{align*}
S({\cal O})'^{i+1}(\bot,\top)&=& \hspace{-15mm} \bigcup_{(x,y)\in  S({\cal O})'^i(\bot,\top)}& \upclosure{C({\cal O}_l)(y)}\times \downclosure{ C({\cal O}_u)(x)} \\
&=& \hspace{-15mm} \bigcup_{(x,y)\in  S({\cal O})'^i(\bot,\top)} &\upclosure{\{C({\cal O}^{\sf AFT}_l)(y)\}}\times \downclosure{\{ C({\cal O}^{\sf AFT}_u)(x)\}}
\end{align*}
and thus $(S({\cal O}^{\sf AFT}))^{i+1}(x',y') \preceq^A_i S({\cal O})'^{i+1}(\bot,\top)$. Hence, with a slight abuse of the notations, we have shown that 
$S({\cal O}^{\sf AFT})(x_i,y_i)\preceq^A_i S({\cal O})'^{i+1}(\bot,\top)$. 
\end{proof}

An interesting property of the well-founded state of ${\cal IC}_{\cal P}$ is that for positive logic programs, the well-founded state coincides with the \emph{minimal\/} 
models of a the logic program.

\begin{proposition}
\label{prop:positive:programs:state}
If ${\cal P}$ is a positive dlp, then ${\sf WF}({\cal IC}_{\cal P})=\upclosure{\min_\subseteq \model({\cal P})}\times \downclosure{\min_\subseteq \model({\cal P})}$.
\end{proposition}

\begin{proof}
Notice first that, since ${\cal P}$ is positive, $\frac{{\cal P}}{y}={\cal P}$ for any $y\subseteq {\cal A}_{\cal P}$ (no rule in ${\cal P}$ contains negative literals in the body, and so no rule is deleted 
or transformed in the construction of $\frac{{\cal P}}{y}$). By Proposition \ref{lemma:stable:ic:is:reduct}, this means that, for any $y\subseteq {\cal A}_{\cal P}$, 
$C({\cal IC}^l_{\cal P})(y)=\min_\subseteq(\model({\cal P}))$. Thus, $S({\cal IC}_{\cal P})'(\emptyset,{\cal A}_{\cal P})=\upclosure{\min_\subseteq \model({\cal P})}\times \downclosure{\min_\subseteq \model({\cal P})}$ 
and a fixpoint is reached at the first iteration of the construction of the well-founded state.
\end{proof}

\subsection{The Relationship with Well-Founded Semantics with Disjunction}
\label{sec:well-founded:alcantara}

As another indication for the usefulness of our constructions, we show that the well-founded state approximates the existing well-founded semantics for disjunctive logic programs, 
which is  the well-founded semantics for disjunction ${\sf WFS}_d$ by Alc{\^a}ntara, Dam{\'a}sio and Pereira~\cite{alcantara2005well}.

Alc{\^a}ntara, Dam{\'a}sio and Pereira~\cite{alcantara2005well} define the well-founded semantics for disjunction ${\sf WFS}_d$ which bears similarities to our well-founded semantics. 
We present this semantics here, adapting the notation and technicalities to our setting.

The basic idea behind the well-founded semantics with disjunction as defined by Alc{\^a}ntara, Dam{\'a}sio and Pereira~\cite{alcantara2005well}
is the following: a sequence of a set of lower bounds and upper bounds approximating the stable models of a disjunctive logic program is constructed iteratively. 
Given a set of lower bounds and a set of upper bounds, 
new (more precise) lower and upper bounds can be obtained by applying the two-valued immediate consequence operator $IC$ to the reducts of the program obtained on the basis of 
the upper bounds and lower bounds (respectively). The lower bounds are thus used to obtain new upper bounds and vice versa.
This iteration leads to more and more precise lower and upper bounds, until a fixpoint is reached. This well-founded semantics is guaranteed to exist, be unique, and coincide 
with the well-founded model for normal logic programs~\cite{alcantara2005well}. 

We now develop this notion in more technical details. The operator $\Gamma_{\cal P}(X)$ is defined as follows:\footnote{Recall that the \emph{two-valued models} of a positive 
program ${\cal P}$ are the sets 
$x\subseteq {\cal A}_{\cal P}$ s.t.\ for every $\bigvee\Delta\leftarrow \phi \in {\cal P}$, $(x,x)(\phi)={\sf T}$ implies $x\cap \Delta\neq\emptyset$ (Footnote~\ref{footnote:2-val mod}).}
\[\Gamma_{\cal P}(X)=\bigcup_{x\in X}\min_\subseteq(\model_2(\frac{{\cal P}}{x}))\]
Instead of just applying the $\Gamma_{\cal P}$-operator to the lower bound and closing it downwards to obtain a new upper bound, a form of closed world reasoning is performed 
in~\cite{alcantara2005well} by applying $\Gamma_{\cal P}$ to the interpretations in the lower bound not containing any atoms not occurring in any upper bound. 
Formally, this is done by defining the set ${\frak F}(Y)$ that contains all atoms not occurring in any upper bound $y\in Y$, and the set ${\frak T}_Y(X)$ that contains 
all the sets in $X$ without any atoms in ${\frak F}(Y)$.

\begin{definition}%
Given $X,Y\subseteq {\cal A}_{\cal P}$, we define:
\begin{itemize}
\item ${\frak F}(Y)=\{ \alpha\in {\cal A}_{\cal P}\mid \alpha\not\in \bigcup Y\}$.
\item ${\frak T}_Y(X)=\{x\in X\mid x\cap {\frak F}(Y)=\emptyset\}$.
\end{itemize}
\end{definition}

We are now ready to define the $\Phi$-operator which is the basis of the well-founded semantics for disjunction~\cite{alcantara2005well}, obtained by taking as set of new lower bounds 
all models of reducts of some upper bound, and as new upper bounds all models of reducts of some lower bound in ${\frak T}_Y(X)$.

\begin{definition}
Let some disjunctive logic program ${\cal P}$ and some $X,Y\subseteq \wp({\cal A}_{\cal P})$ be given. Then:
\[ \Phi_{\cal P}(X,Y)=\Gamma_{\cal P}(Y)\uparrow\times \Gamma_{\cal P}({\frak T}_Y(X))\downarrow.\]
The \emph{well-founded semantics for disjunction of a program ${\cal P}$}, denoted ${\sf WFS}_d({\cal P})$ is defined as the least fixpoint of $\Phi_{\cal P}$, obtained by applying 
$\Phi_{\cal P}$ iteratively to $\{\emptyset\}\times \{{\cal A}_{\cal P}\}$.
\end{definition}

We can show that our well-founded semantics is an approximation of the ${\sf WFS}_d$. It really is an approximation, since we do not apply closed world reasoning in constructing the lower 
bound (but one could). 

\begin{theorem}
\label{prop:well:founded:state:represents:alcantara:almost}
For any disjunctive logic program ${\cal P}$, ${\sf WF}({\cal IC}_{\cal P})\preceq^A_i {\sf WFS}_d({\cal P})$.
\end{theorem}

\begin{proof}
By Proposition  \ref{lemma:stable:ic:is:reduct}, $\Gamma_{\cal P}(X)=\bigcup_{x\in X}C({\cal IC}^l_{\cal P})(x)$ for any $X\subseteq 2^{{\cal A}_{\cal P}}$. Furthermore, as  for any 
$X,Y\subseteq {\cal A}_{\cal P}$, ${\frak T}_Y(X)\subseteq X$, it holds that $\bigcup_{x\in{\frak T}_Y(X)} S({\cal IC}_{\cal P})(x)\preceq^H_L \bigcup_{x\in X} S({\cal IC}_{\cal P})(x)$. 
The proposition now immediately follows from these two observations.
\end{proof}

The following example (taken from~\cite{alcantara2005well}) shows that ${\sf WFS}_d({\cal P})$ can give rise to a strictly more precise approximation than that of ${\sf WF}({\cal IC}_{\cal P})$:

\begin{example}
Consider the logic program ${\cal P}=\{p\lor q\leftarrow; q\leftarrow \lnot r\}$. We first calculate ${\sf WF}({\cal IC}_{\cal P})$ as follows:
\begin{eqnarray*}
S({\cal IC}_{\cal P})(\upclosure{\{\emptyset\}},\downclosure{\{p,q,r\}})&=&
\upclosure{\min_\subseteq\left(\bigcup_{x\subseteq \{p,q,r\}}\model_2(\frac{{\cal P}}{x})\right)}\times \downclosure{\min_\subseteq\left(\bigcup_{x\subseteq \{p,q,r\}}\model_2(\frac{{\cal P}}{x})\right)}\\ 
&=& \upclosure{\{\{p\},\{q\}\}}\times \downclosure{\{\{p\},\{q\}\}}\\
S^2({\cal IC}_{\cal P})(\upclosure{\{\emptyset\}},\downclosure{\{p,q,r\}})&=&\upclosure{\min_\subseteq\left(\model_2(\frac{{\cal P}}{\{p\}})\cup \model_2(\frac{{\cal P}}{\{q\}})\right)}\times \downclosure{ \min_\subseteq\left(\bigcup_{x\subseteq \{p,q,r\}}\model_2(\frac{{\cal P}}{x})\right)}\\
&=&\upclosure{\{\{q\}\}}\times \downclosure{\{\{p\},\{q\}\}}\\
S^3({\cal IC}_{\cal P})(\upclosure{\{\emptyset\}},\downclosure{\{p,q,r\}})&=&
\upclosure{\min_\subseteq\left(\model_2(\frac{{\cal P}}{\{q\}})\cup \model_2(\frac{{\cal P}}{\{q\}})\right)}\times \downclosure{ \min_\subseteq\left(\bigcup_{x\subseteq \{p,q,r\}}\model_2(\frac{{\cal P}}{x})\right)}\\
&=&\upclosure{\{\{q\}\}}\times \downclosure{\{\{p\},\{q\}\}}\\
\end{eqnarray*}
and so a fixpoint is reached after the second iteration. The well-founded state for this program thus corresponds to the convex set $\{\{q\},\{p\},\{q\}\}$.

It can be observed that $\frac{{\cal P}}{x}=\{p\lor q\leftarrow\}$ for any $\{r\}\subseteq x\subseteq {\cal A}_{\cal P}$. In particular, as there are some $x\in \upclosure{\{p\},\{q\}}$ s.t.\ $r\in x$, 
this explains why $\{p\}$ is part of the upper bound of ${\sf WF}({\cal IC}_{\cal P})$. It is exactly this kind of behaviour that the filtering out lower bounds that have elements not occuring in any upper 
bound in ${\frak T}_Y(X)$ tries to avoid. Indeed, ${\sf WFS}_d({\cal P})$ is built up as follows:
\begin{eqnarray*}
\Phi_{\cal P}(\{\emptyset\},\{\{p,q,r\}\})&=&S({\cal IC}_{\cal P})(\upclosure{\{\emptyset\}},\downclosure{\{p,q,r\}})\\
&=&\upclosure{\{\{p\},\{q\}\}}\times \downclosure{\{\{p\},\{q\}\}}\\
\Phi^2_{\cal P}(\{\emptyset\},\{\{p,q,r\}\})&=&
\upclosure{\min_\subseteq\left(\model_2(\frac{{\cal P}}{\{p\}})\cup \model_2(\frac{{\cal P}}{\{q\}})\right)}\times
 \downclosure{ \min_\subseteq\left(\model_2(\frac{{\cal P}}{\{p\}})\cup \model_2(\frac{{\cal P}}{\{q\}})\cup  \model_2(\frac{{\cal P}}{\{q,r\}})\right)}\\
 &=&\upclosure{\{\{p\}\}}\times \downclosure{\{\{p\}\}}\\
 \Phi^3_{\cal P}(\{\emptyset\},\{\{p,q,r\}\})&=&
\upclosure{\min_\subseteq\left(\model_2(\frac{{\cal P}}{\{p\}})\cup \model_2(\frac{{\cal P}}{\{q\}})\right)}\times
 \downclosure{ \min_\subseteq\left(\model_2(\frac{{\cal P}}{\{p\}})\cup \model_2(\frac{{\cal P}}{\{q\}})\cup  \model_2(\frac{{\cal P}}{\{q,r\}})\right)}\\
 &=&\upclosure{\{\{p\}\}}\times \downclosure{\{\{p\}\}}
\end{eqnarray*}
The upper bound of $\Phi^2_{\cal P}(\{\emptyset\},\{\{p,q,r\}\})$ can be seen to hold in view of  ${\frak T}_{\downclosure{\{\{p\},\{q\}\}}}( \upclosure{\{\{p\},\{q\}\}})=\{\{p\},\{q\},\{p,q\}\}$. After two iterations, a fixed point is reached, and thus we see that ${\sf WFS}_d({\cal P})$ corresponds to the convex set $\{\{p\}\}$, which is a more precise approximation than  ${\sf WF}({\cal IC}_{\cal P})$. Formally:
\[ {\sf WF}({\cal IC}_{\cal P})=\upclosure{\{\{q\}\}}\times \downclosure{\{\{p\},\{q\}\}}\prec_i^A  {\sf WFS}_d({\cal P}) =\upclosure{\{\{p\}\}}\times \downclosure{\{\{p\}\}}.\]
\end{example}
We leave the investigation of techniques to make the well-founded fixpoint more precise, such as the technique using ${\frak T}_Y(X)$, for further work.}

\paragraph{Summary}
Let's summarize the results in Section~\ref{sec:stable:semantics}. We first defined and studied stable operators (Section \ref{subsec:stable:interpretation:sem}), establishing which properties carry over 
(under certain conditions) from the deterministic to the non-deterministic setting, and showing the usefulness of stable fixpoints in disjunctive logic programming.  Then (Section~\ref{subsec:well-foundedstate}), 
we introduced and studied the well-founded state, proving its existence, uniqueness, and showing that it is more precise than the Kripke-Kleene state and that it approximates any fixpoint of the 
approximated operator. Finally (Section~\ref{sec:well-founded:alcantara}), we have shown that the well-founded state is useful for knowledge representation, as it is closely related to the well-founded semantics for disjunction~\cite{alcantara2005well}.

\section{Related Work}
\label{sec:related:work}

The starting point of this work is the approximation fixpoint theory (for deterministic operator), as introduced by  Denecker, Marek and Truszczy{\'n}ski~\cite{denecker2000approximations},
followed by a series of papers \cite{antic2013hex, bogaerts2019weighted, bogaerts2015grounded, charalambidis2018approximation, denecker2012approximation, denecker2002ultimate,
denecker2003uniform,  HA2020argumentative, strass2015analyzing}. As indicated previously (see, e.g., Remark~\ref{remark:deterministic}), this work generalizes AFT in 
the sense that all the operators and fixpoints defined in this paper coincide with the respective counterparts for deterministic operators.

This paper (extending and improving our paper in~\cite{DBLP:conf/kr/HeyninckA21}) is also inspired by the work of Pelov and Truszczy\'nski~\cite{pelov2004semantics},  
 which extends approximation fixpoint theory to dealing with non-deterministic operators. Their work provides a representation theorem, in terms of non-deterministic AFT, 
 of specific two-valued semantics for disjunctive logic programs (namely, the two-valued stable semantics and two-valued weakly supported and supported models). 
 We have compared our work to that of Pelov and Truszczy\'nski~\cite{pelov2004semantics} in Section~\ref{sec:intro}.

To the best of our knowledge, the only setting with a similar unifying potential that has been applied to non-deterministic or disjunctive reasoning is
\emph{equilibrium logic\/}~\cite{pearce2006equilibrium}. The similarities between equilibrium logic and AFT have been noted in~\cite{denecker2012approximation}, 
where it was indicated that equilibrium semantics are defined for a larger class of logic programs than those that are represented by AFT, a limitation of AFT which
we have overcome in this paper. Furthermore, defining three-valued stable and well-founded semantics is not possible in standard equilibrium logic, but requires an extension
known as \emph{partial equilibrium logic\/}~\cite{cabalar2006analysing,cabalar2007partial}, which can be seen as a six-valued semantics.  In contrast, the well-founded semantics 
{\em is defined\/} in AFT using the same operator used to define the stable semantics. That being said, in future work we plan to compare in more detail the well-founded semantics 
for DLP obtained on the basis of partial equilibrium logic and the well-founded semantics obtained in this work.

\section{Summary, Conclusion, and Future Work}
\label{sec:conc}

This paper contains a full  approximation fixpoint theory to non-deterministic operators. We introduced deterministic operators,
their non-deterministic approximation operators, and the various fixpoint semantics, namely, the Kripke-Kleene interpretation and state semantics, the stable interpretation semantics 
and the well-founded state semantics. The properties of these semantics and their representation of disjunctive logic programming semantics is summarized in Table~\ref{tab:summary}. 
The relation between these fixpoint semantics is summarized in Figure~\ref{fig:fixpoint:relations}.

\begin{table}[h]
{\centering
\renewcommand{\arraystretch}{1.3}
\setlength{\tabcolsep}{4.5pt}
\begin{tabular}{llcccll}
Name & Definition & Exist & Unique & $\leq_t$-min.& Result & DLP-representation \\ \hline
KK interp.\  & $\lfp({\cal O})$ & $\times$ & $\times$ & $\times$ &   & weakly supported (Th.\ref{theo:correspondence:supported})  \\
KK state & $\lfp({\cal O}')$ & $\checkmark$ & $\checkmark$ & $\times$& Th.\ref{theorem:ndso:fixpoint} & \\
Stable interp.\  & ${\rm fp}(S({\cal O}))$ &  $\times$ & $\times$ & $\checkmark$& Prop.\ref{prop:stable:is:minimal:fp}& stable models (Th.\ref{prop:stable:fixpoints:represent:stable:models}) \\
WF state & $\lfp(S({\cal O})')$ & $\checkmark$ & $\checkmark$ & $\times$ & Th.\ref{proposition:properties:of:well-founded:state}& WF sem.\ with disjunct.\ (Th.\ref{prop:well:founded:state:represents:alcantara:almost})  \\ 
& &&&&& Min.\ mod.\ of positive dlps (Prop.\ref{prop:positive:programs:state})
\end{tabular}}
\caption{Approximating operators and their properties.}
\label{tab:summary}
\end{table}

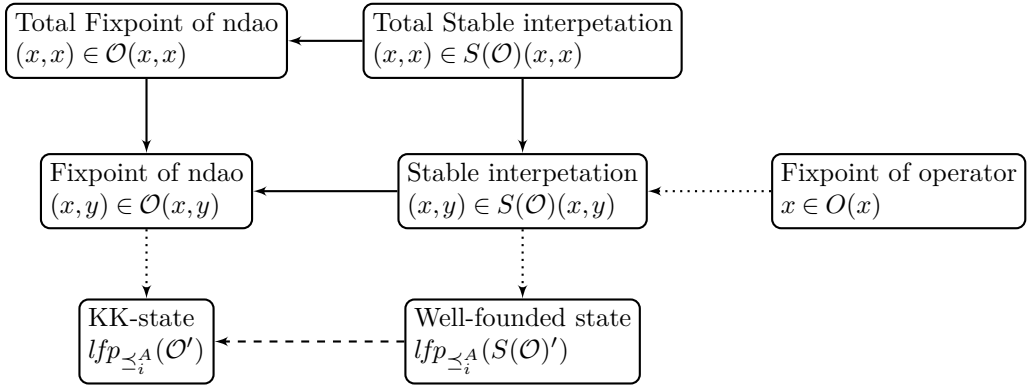
\begin{figure}[h]\centering
 \begin{tikzpicture}[thick,scale=1, every vertex/.style={scale=1}]
\tikzset{vertex/.style = {shape=rectangle,draw,minimum size=1.5em,rounded corners=.8ex,scale=1}}
\tikzset{edge/.style = {->,> = latex'}}

\node[vertex] (KKS)  at (0,0) [draw, align=left] {KK-state\\ $l\!f\!p_{\preceq^A_i}({\cal O}')$};
\node[vertex] (WFS)  at (5,0) [draw, align=left] {Well-founded state\\ $l\!f\!p_{\preceq^A_i}(S({\cal O})')$};
\node[vertex] (KKI)  at (0,2) [draw, align=left] {Fixpoint of ndao \\$(x,y)\in {\cal O}(x,y)$};
\node[vertex] (2KKI)  at (0,4) [draw, align=left] {Total Fixpoint  of ndao \\$(x,x)\in{\cal O}(x,x)$};
\node[vertex] (SI)  at (5,2) [draw, align=left] {Stable interpetation\ \\$(x,y)\in S({\cal O})(x,y)$};
\node[vertex] (2SI)  at (5,4) [draw, align=left] {Total Stable interpetation\ \\$(x,x)\in S({\cal O})(x,x)$};
\node[vertex] (O)  at (10,2) [draw, align=left] {Fixpoint of operator \\ $x\in O(x)$};

\draw[edge, dotted] (KKI) to (KKS);
\draw[edge, dashed] (WFS) to (KKS);
\draw[edge, dotted] (SI) to (WFS);
\draw[edge] (SI) to (KKI);
\draw[edge] (2SI) to (SI);
\draw[edge, dotted] (O) to (SI);
\draw[edge] (2KKI) to (KKI);
\draw[edge] (2SI) to (2KKI);

\end{tikzpicture}
\caption{Relations between various fixpoints introduced in this paper. The arrows have the following meaning: full arrows mean that every instance of the first block 
(i.e., at the outgoing end of the arrow) is an instance of the second block (at the ingoing end of the arrow). Dotted arrows mean that all instances of 
the first block are approximated by the second block. Dashed arrows mean that every instance of the first block is more precise than every instance of the second block. 
The relations are shown in~Proposition \ref{prop:KK-states-properties} and Theorem~\ref{proposition:properties:of:well-founded:state}.}
\label{fig:fixpoint:relations}
\end{figure}

This work also
allows to generalize the results in~\cite{antic2013hex,pelov2004semantics}, which provide further approximation operators for disjunctive logic programs with aggregates or external atoms, to further semantics of disjunctive logic programs, thus answering an open question in 
these works. 

The advantage of studying non-deterministic operators
is thus at least twofold:
\begin{enumerate}
   \item allowing to define a family of semantics 
           for non-monotonic reasoning with disjunctive information,
  \item clarifying similarities and differences between semantics stemming from the use of different operators.
\end{enumerate}

The introduction of disjunctive information in AFT points to a wealth of further research, such as defining three-valued and well-founded semantics for various disjunctive 
nonmonotonic formalisms and studying on the basis of which operators various well-founded semantics for DLP  can be  represented in our framework. For example, our framework 
can potentially be used for defining three-valued and well-founded semantics for propositional theories~\cite{truszczynski2010reducts}, disjunctive logic programs with 
recursive aggregates~\cite{faber2004recursive}, logic programs with aggregates in the head~\cite{gelfond2014vicious}, logic programs with forks~\cite{aguado2019forgetting}, 
and disjunctive default logics~\cite{Gelfond_et-al_PKRR_1991,bonevac2018defaulting}.

Non-deterministic approximation fixpoint theory has already been applied in order to obtain a generalization of abstract dialectical frameworks (ADFs) to \emph{conditional abstract dialectical frameworks}. In a nutshell, abstract dialectical frameworks consist of sets of arguments or atoms, in which every atom $a$ is assigned a Boolean acceptance condition $C_a$, which 
codifies under which conditions an atom can be accepted. Heyninck and co-authors~\cite{DBLP:conf/aaai/HeyninckTKRS22} generalized ADFs as to allow for acceptance conditions to be 
assigned to any, i.e.\ possibly non-atomic, formulas. This necessitated the generalization of the so-called $\Gamma$-operator to a non-deterministic operator, which was done in the 
framework of non-deterministic approximation fixpoint theory as presented here. 

Our
framework lays the ground for the generalization 
of various interesting concepts introduced (or adapted) to AFT, such as ultimate approximations \cite{denecker2002ultimate},  grounded fixpoints 
\cite{bogaerts2015grounded}, strong equivalence \cite{truszczynski2006strong}, stratification \cite{vennekens2006splitting} and argumentative representations 
\cite{HA2020argumentative} to a non-deterministic setting. Extensions to DLP with negations in the rules' heads and corresponding 4-valued semantics~\cite{SI95} 
can also be considered.

Another issue that is worth some consideration is a study of the complexity of computing different types of fixpoints of non-deterministic 
approximation operators, in the style of Strass and Wallner \cite{strass2015analyzing}.

\section*{Acknowledgements}
We thank the reviewers of \cite{DBLP:conf/kr/HeyninckA21} for their helpful remarks. 
The first author is supported by the German National Science Foundation (DFG-project KE-1413/11-1). He worked on this paper while affiliated with the Vrije Universiteit Brussel, the 
Technische Universität Dortmund, the University of Cape Town and CAIR South-Africa, in addition to his current affiliation.
The second author is supported by the Israel Science Foundation (Grant~550/19).
The first and last author are supported by the Flemish Government in the ``Onderzoeksprogramma Artificiële
Intelligentie (AI) Vlaanderen'' programme, and by the FWO Flanders project G0B2221N.

\bibliographystyle{plain}  %
\bibliography{nondetAFT}

\end{document}